\documentclass[11pt]{article}
\usepackage{amsmath, amssymb, amsthm, fullpage}
\usepackage{color, graphicx}
\usepackage{algorithm, algorithmic}

\def\real{\mathbb{R}}
\def\R{\mathbb{R}}

\def\pdim{p}

\def\parset{\Omega}

\def\norm#1{\|#1\|}

\def\loss{\mathcal{L}}

\def\numobs{n}
\def\pdim{d}
\def\spindex{s}

\theoremstyle{plain}

\newtheorem{proposition}{Proposition}
\newtheorem{theo}{Theorem}[section]

\newtheorem{lem}{Lemma}[section]
\newtheorem{prop}{Proposition}[section]
\newtheorem{cor}{Corollary}[section]
\newtheorem{corollary}{Corollary}

\theoremstyle{definition} 

\newtheorem{nota}{Notation}[section]
\newtheorem{de}{Definition}[section]
\newtheorem{exa}{Example}[section]
\newtheorem{as}{Assumption}[section]
\newtheorem{alg}{Algorithm}[section]

\newcommand{\btheo}{\begin{theo}}
\newcommand{\bde}{\begin{de}}
\newcommand{\ble}{\begin{lem}}
\newcommand{\bpr}{\begin{prop}}
\newcommand{\bno}{\begin{nota}}
\newcommand{\bex}{\begin{exa}}
\newcommand{\bcor}{\begin{cor}}
\newcommand{\spro}{\begin{proof}}
\newcommand{\bas}{\begin{as}}
\newcommand{\balg}{\begin{alg}}

\newcommand{\etheo}{\end{theo}}
\newcommand{\ede}{\end{de}}
\newcommand{\ele}{\end{lem}}
\newcommand{\epr}{\end{prop}}
\newcommand{\eno}{\end{nota}}
\newcommand{\eex}{\end{exa}}
\newcommand{\ecor}{\end{cor}}
\newcommand{\fpro}{\end{proof}}
\newcommand{\eas}{\end{as}}
\newcommand{\ealg}{\end{alg}}

\theoremstyle{plain}

\newtheorem{theos}{Theorem}
\newtheorem{props}{Proposition}
\newtheorem{lems}{Lemma}
\newtheorem{cors}{Corollary}

\theoremstyle{definition}
\newtheorem{exas}{Example}
\newtheorem{algs}{Algorithm}
\newtheorem{asss}{Assumption}
\newtheorem{defns}{Definition}

\newcommand{\btheos}{\begin{theos}}
\newcommand{\etheos}{\end{theos}}
\newcommand{\bprops}{\begin{props}}
\newcommand{\eprops}{\end{props}}
\newcommand{\bdes}{\begin{defns}}
\newcommand{\edes}{\end{defns}}
\newcommand{\blems}{\begin{lems}}
\newcommand{\elems}{\end{lems}}
\newcommand{\bcors}{\begin{cors}}
\newcommand{\ecors}{\end{cors}}
\newcommand{\bexs}{\begin{exas}}
\newcommand{\eexs}{\end{exas}}
\newcommand{\balgs}{\begin{algs}}
\newcommand{\ealgs}{\end{algs}}
\newcommand{\bass}{\begin{asss}}
\newcommand{\eass}{\end{asss}}

\renewcommand{\P}{\ensuremath{\mathbb{P}}}
\newcommand{\ip}[2]{\langle #1, #2\rangle}

\newtheorem{theorem}{Theorem}
\newtheorem{lemma}{Lemma}

\newtheorem{assumption}{Assumption}

\newcommand{\order}{\mathcal{O}}

\newcommand{\defn}{\ensuremath{:  =}}

\newcommand{\inprod}[2]{\ensuremath{\langle #1 , \, #2 \rangle}}

\newcommand{\normal}[1]{\ensuremath{\mathcal{N}(0,#1)}}

\newcommand{\sample}{\ensuremath{z}}

\makeatletter
\long\def\@makecaption#1#2{
        \vskip 0.8ex
        \setbox\@tempboxa\hbox{\small {\bf #1:} #2}
        \parindent 1.5em  
        \dimen0=\hsize
        \advance\dimen0 by -3em
        \ifdim \wd\@tempboxa >\dimen0
                \hbox to \hsize{
                        \parindent 0em
                        \hfil 
                        \parbox{\dimen0}{\def\baselinestretch{0.96}\small
                                {\bf #1.} #2
                                } 
                        \hfil}
        \else \hbox to \hsize{\hfil \box\@tempboxa \hfil}
        \fi
        }
\makeatother

\newenvironment{carlist}
 {\begin{list}{$\bullet$}
 {\setlength{\topsep}{0in} \setlength{\partopsep}{0in}
  \setlength{\parsep}{0in} \setlength{\itemsep}{\parskip} \addtolength{\itemsep}{-0.15cm}
  \setlength{\leftmargin}{0.07in} \setlength{\rightmargin}{0.08in}
  \setlength{\listparindent}{0in} \setlength{\labelwidth}{0.08in}
  \setlength{\labelsep}{0.1in} \setlength{\itemindent}{0.05in}}}
 {\end{list}}

\newcommand{\bcar}{\begin{carlist}}
\newcommand{\ecar}{\end{carlist}}

\newenvironment{proof-of-theorem}[1][{}]{\noindent{\bf Proof of
    Theorem~{#1}} \hspace*{1em}}{\qed\smallskip\\}

\newlength{\widebarargwidth}
\newlength{\widebarargheight}
\newlength{\widebarargdepth}
\DeclareRobustCommand{\widebar}[1]{%
  \settowidth{\widebarargwidth}{\ensuremath{#1}}%
  \settoheight{\widebarargheight}{\ensuremath{#1}}%
  \settodepth{\widebarargdepth}{\ensuremath{#1}}%
  \addtolength{\widebarargwidth}{-0.3\widebarargheight}%
  \addtolength{\widebarargwidth}{-0.3\widebarargdepth}%
  \makebox[0pt][l]{\hspace{0.3\widebarargheight}%
    \hspace{0.3\widebarargdepth}%
    \addtolength{\widebarargheight}{0.3ex}%
    \rule[\widebarargheight]{0.95\widebarargwidth}{0.1ex}}%
  {#1}}

\long\def\comment#1{}

\newcommand{\qpar}{\ensuremath{q}}
\newcommand{\radq}{\ensuremath{R_\qpar}}

\newcommand{\Parset}{\ensuremath{\parset}}

\newcommand{\iter}[1]{\ensuremath{\itercon(#1)}}

\newcommand{\paramtil}{\ensuremath{\widetilde{\param}}}

\newcommand{\Exs}{\ensuremath{\mathbb{E}}}

\newcommand{\MESS}{\xi_\totiters}

\newcommand{\Sset}{\ensuremath{S}}
\newcommand{\Tset}{\ensuremath{A}}
\newcommand{\SsetComp}{\ensuremath{\Sset^c}}
\newcommand{\TsetComp}{\ensuremath{\Tset^c}}

\newcommand{\CovMat}{\ensuremath{\Sigma}}

\newcommand{\kdim}{\ensuremath{k}}

\newcommand{\sigmin}{\ensuremath{\sigma_{\operatorname{\scriptsize{min}}}}}

\newcommand{\Ball}{\ensuremath{\mathbb{B}}}

\newcommand{\regpar}{\ensuremath{\lambda}}

\newcommand{\lossrsc}{\ensuremath{\gamma}}

\newcommand{\thetahat}{\ensuremath{\widehat{\theta}}}

\newcommand{\PREFACT}{\ensuremath{c_0}}

\newcommand{\widgraph}[2]{\includegraphics[keepaspectratio,width=#1]{#2}}

\newcommand{\Loss}{\ensuremath{\loss}}
\newcommand{\LossN}{\ensuremath{\loss_{\numobs}}}

\newcommand{\plaincon}{\ensuremath{c}}

\newcommand{\stochgrad}[1]{\ensuremath{g^{#1}}}
\newcommand{\plstochgrad}{\ensuremath{g}}

\newcommand{\pval}{\ensuremath{p}}
\newcommand{\TOTEPOCHPLAIN}{\ensuremath{K}}
\newcommand{\TOTEPOCH}{\ensuremath{\TOTEPOCHPLAIN_{\totiters}}}

\newcommand{\dualpar}{\ensuremath{\mu}}
\newcommand{\diter}[1]{\ensuremath{\dualpar^{#1}}}
\newcommand{\E}{\ensuremath{\mathbb{E}}}
\newcommand{\prox}{\ensuremath{\psi}}
\newcommand{\proxbound}{\ensuremath{A_{\prox}}}
\newcommand{\stepsize}{\ensuremath{\alpha}}
\newcommand{\mystep}[1]{\ensuremath{\alpha^{#1}}}

\newcommand{\sign}{\ensuremath{\operatorname{sign}}}

\newcommand{\phihinge}{\ensuremath{\phi_{\mbox{\tiny{hin}}}}}
\newcommand{\philog}{\ensuremath{\phi_{\mbox{\tiny{log}}}}}

\newcommand{\xavg}{\ensuremath{\bar{\param}}}
\newcommand{\param}{\ensuremath{\theta}}
\newcommand{\paramother}{\ensuremath{\tilde{\param}}}

\newcommand{\epochopt}[1]{\ensuremath{\widehat{\param}_{#1}}}

\newcommand{\opt}{\ensuremath{\param^*}}
\renewcommand{\iter}[1]{\ensuremath{\param^{#1}}}

\newcommand{\radius}{\ensuremath{R}}
\newcommand{\proxcenter}[1]{\ensuremath{y_{#1}}}
\newcommand{\reggrad}[1]{\ensuremath{\nu^{#1}}}

\newcommand{\samplex}{\ensuremath{x}}
\newcommand{\sampley}{\ensuremath{y}}

\newcommand{\samplew}{\ensuremath{w}}
\newcommand{\xbound}{\ensuremath{B}}

\newcommand{\F}{\ensuremath{\mathcal{F}}}

\newcommand{\kstar}{\ensuremath{{k^*}}}

\newcommand{\DelAvg}{\ensuremath{\Delta^*}}
\newcommand{\DelBarAvg}{\ensuremath{\widehat{\Delta}}}
\newcommand{\DelIt}[1]{\ensuremath{\Delta_{#1}}}

\newcommand{\lips}{\ensuremath{G}}
\newcommand{\noise}{\ensuremath{\sigma}}
\newcommand{\gaussnoise}{\ensuremath{\eta}}

\newcommand{\graderr}{\ensuremath{e}}
\newcommand{\mygraderr}[1]{\ensuremath{\graderr^{#1}}}

\newcommand{\devcon}{\ensuremath{\omega}}
\newcommand{\newdevcon}{\ensuremath{\delta}}

\newcommand{\sampleY}{\ensuremath{Y}}
\newcommand{\sampleX}{\ensuremath{X}}
\newcommand{\coneslack}{\ensuremath{\norm{\opt_{\Sset^c}}_1}}

\newcommand{\totalslacksq}{\ensuremath{\varepsilon^2(\opt;\Sset, \slopfac)}}
\newcommand{\totalslacksqsimple}{\ensuremath{\varepsilon^2(\opt; \Sset)}}

\newcommand{\ellqtol}{\ensuremath{\varphi}}
\newcommand{\slopfac}{\ensuremath{\tau}}
\newcommand{\const}{\ensuremath{c}}
\newcommand{\realrsc}{\ensuremath{\overline{\lossrsc}}}

\newenvironment{proof-of-lemma}[1][{}]{\noindent{\bf Proof of Lemma {#1}}
  \hspace*{1em}}{\qed\smallskip\\}
\newenvironment{proof-of-proposition}[1][{}]{\noindent{\bf
    Proof of Proposition {#1}}
  \hspace*{1em}}{\qed\smallskip\\}
\newenvironment{proof-of-corollary}[1][{}]{\noindent{\bf
    Proof of Corollary {#1}}
  \hspace*{1em}}{\qed\smallskip\\}
\newcommand{\myalg}{Regularization Annealed epoch Dual AveRaging}
\newcommand{\myalgshort}{RADAR}
\newcommand{\CovMax}{\ensuremath{\rho(\CovMat)}}

\newcommand{\liptil}{\ensuremath{\widetilde{\lips}}}
\newcommand{\noisetil}{\ensuremath{\widetilde{\noise}}}
\newcommand{\compositegrad}[1]{\ensuremath{{\widehat{\plstochgrad}}^{#1}}}
\newcommand{\finalparam}[1]{\ensuremath{\thetahat_{#1}}}
\newcommand{\totiters}{\ensuremath{T}}
\newcommand{\slackiters}{\ensuremath{\kappa}}

\newcommand{\xnoise}{\ensuremath{\gaussnoise_{\samplex}}}
\newcommand{\SparseSet}{\ensuremath{\mathbb{K}}}
\newcommand{\Kstar}{\ensuremath{\TOTEPOCH^*}}

\newcommand{\Sbar}{\ensuremath{{S^c}}}

\newcommand{\snndel}[1]{}

\newcommand{\usedim}{\ensuremath{d}}
\newcommand{\Sample}{\ensuremath{Z}}
\newcommand{\SamSpace}{\ensuremath{\mathcal{Z}}}

\newcommand{\Lossbar}{\ensuremath{\widebar{\Loss}}}

\newcommand{\FIGSIZE}{0.45\textwidth}

\newcommand{\TERMONE}{\Phi_1} \newcommand{\TERMTWO}{\Phi_2} 

\newcommand{\HACKLRSC}{$2'$}
\newcommand{\BHACKLRSC}{$\mathbf{2'}$}

\begin{document}

\begin{center}
{\LARGE{{\bf{Stochastic optimization and sparse statistical recovery:
        \\ An optimal algorithm for high dimensions}}}}

\vspace*{.2in}

\begin{tabular}{ccc}
  Alekh Agarwal$^\dagger$
&
  Sahand Negahban$^\ddag$
&
  Martin J. Wainwright$^{\star, \dagger}$ \\
  \texttt{alekh@eecs.berkeley.edu} & 
  \texttt{sahandn@mit.edu} &
  \texttt{wainwrig@stat.berkeley.edu}
\end{tabular}

\vspace*{.2in}

\begin{tabular}{cc}
  Department of Statistics$^\star$, and & Department of EECS$^\ddag$\\
  Department of EECS$^\dagger$, & Massachusetts Institute of Technology\\
  University of California, Berkeley, CA &   Cambridge, MA  
\end{tabular}

\vspace*{.2in}

\today

\vspace*{.2in}

\begin{abstract}
  We develop and analyze stochastic optimization algorithms for
  problems in which the expected loss is strongly convex, and the
  optimum is (approximately) sparse. Previous approaches are able to
  exploit only one of these two structures, yielding an
  $\order(\pdim/T)$ convergence rate for strongly convex objectives in
  $\pdim$ dimensions, and an $\order(\sqrt{(\spindex \log \pdim)/T})$
  convergence rate when the optimum is $\spindex$-sparse. Our
  algorithm is based on successively solving a series of
  $\ell_1$-regularized optimization problems using Nesterov's dual
  averaging algorithm. We establish that the error of our solution
  after $T$ iterations is at most $\order((\spindex \log\pdim)/T)$,
  with natural extensions to approximate sparsity. Our results apply
  to locally Lipschitz losses including the logistic, exponential,
  hinge and least-squares losses. By recourse to statistical minimax
  results, we show that our convergence rates are optimal up to
  multiplicative constant factors. The effectiveness of our approach
  is also confirmed in numerical simulations, in which we compare to
  several baselines on a least-squares regression problem.
\end{abstract}

\end{center}

\section{Introduction}

Stochastic optimization algorithms have many desirable features for
large-scale machine learning, and accordingly have been the focus of
renewed and intensive study in the last several years (e.g., see the
papers~\cite{ShalevSiSr07,BottouBo07,DuchiSi09c,Xiao10} and references
therein). The empirical efficiency of these methods is backed with
strong theoretical guarantees, providing sharp bounds on their
convergence rates.  These convergence rates are known to depend on the
structure of the underlying objective function, with faster rates
being possible for objective functions that are smooth and/or
(strongly) convex, or optima that have desirable features such as
sparsity.  More precisely, for an objective function that is
\emph{strongly convex}, stochastic gradient descent enjoys a
convergence rate ranging from $\order(1/T)$, when features vectors are
extremely sparse, to $\order(\pdim/T)$ when feature vectors are
dense~\cite{HazanKaKaAg06,NemirovskiJuLaSh09,HazanKa11a}. Such results
are of significant interest, because the strong convexity condition is
satisfied for many common machine learning problems, including
boosting, least squares regression, support vector machines and
generalized linear models, among other examples.

A complementary type of condition is that of \emph{sparsity}, either
exact or approximate, in the optimal solution.  Sparse models have
proven useful in many application areas (see the overview
papers~\cite{Donoho00, NegRavWaiYu09, BuhlmannGe2011} and references
therein for further background), and many optimization-based
statistical procedures seek to exploit such sparsity via
$\ell_1$-regularization.  A significant feature of optimization
algorithms for sparse problems is their mild logarithmic scaling with
the problem
dimension~\cite{NemirovskiYu83,Shalev-ShwartzTe11,DuchiSi09c,Xiao10}.
More precisely, it is known~\cite{NemirovskiYu83,Shalev-ShwartzTe11}
that when the optimal solution $\opt$ has at most $\spindex$ non-zero
entries, appropriate versions of the stochastic mirror descent
algorithm converge at a rate $\order(\spindex\sqrt{(\log\pdim)/T})$.
Srebro et al.~\cite{SrebroSrTe10} exploit the smoothness of many
common loss functions; in application to sparse linear regression,
their analysis yields improved rates of the form
$\order(\gaussnoise\sqrt{(\spindex\log\pdim)/T})$, where $\gaussnoise$
is the noise variance.  While the $\sqrt{\log \pdim}$ scaling of the
error makes these methods attractive in high dimensions, observe that
the scaling with respect to the number of iterations is relatively
slow---namely, $\order(1/\sqrt{T})$ as opposed to the $\order(1/T)$
rate possible for strongly convex problems.

Many optimization problems encountered in practice exhibit \emph{both}
features: the objective function is strongly convex, and the optimum
is sparse, or more generally, well-approximated by a sparse vector.
This fact leads to the natural question: is it possible to design an
algorithm for stochastic optimization that enjoys the best features of
both types of structure?  More specifically, the algorithm should have
a $\order(1/T)$ convergence rate, but simultaneously enjoy the mild
logarithmic dependence on dimension.  The main contribution of this
paper is to answer this question in the affirmative, and in
particular, to analyze a new algorithm that has convergence rate
$\order( (\spindex \log\pdim )/T )$ for a strongly convex problem with
an $\spindex$-sparse optimum in $\pdim$ dimensions.  Moreover, using
information-theoretic techniques, we prove that this rate is
\emph{unimprovable} up to constant factors, meaning that no algorithm
can converge at a substantially faster rate. \\

The algorithm proposed in this paper builds off recent work on
multi-step methods for strongly convex
problems~\cite{JuditskyNes10,HazanKa11a,LanGh2010}, but involves some
new ingredients that are essential to obtain optimal rates for
statistical problems with sparse optima. In particular, instead of
performing updates on the same objective, we form a sequence of
objective functions by decreasing the amount of regularization as the
optimization algorithm proceeds.  From a statistical viewpoint, this
reduction is quite natural: at the initial stages, the algorithm has
seen only a few samples, and so should be regularized more heavily,
whereas at later stages when the effective sample size is much larger,
the regularization should be appropriately attenuated. Each step of
our algorithm can be computed efficiently with a closed form update
rule in many common examples. In summary, the outcome of our
development is an \emph{optimal one-pass} algorithm for many
structured statistical problems in high dimensions, and with
computational complexity linear in the sample size. Numerical
simulations confirm our theoretical predictions regarding the
convergence rate of the algorithm, and also establish its superiority
compared to regularized dual averaging~\cite{Xiao10} and stochastic
gradient descent algorithms. They also confirm that a direct
application of the multi-step method of Juditsky and
Nesterov~\cite{JuditskyNes10} is inferior to our algorithm, meaning
that our gradual decrease of regularization is quite critical. In
order to keep our presentation focused, we restrict our attention to
multi-step variants of the dual averaging algorithm; however, it is
worth noting that similar results can also be achieved for mirror
descent as well as Nesterov's accelerated gradient
methods~\cite{Nesterov07} by combining our results with recent work in
the optimization literature~\cite{LanGh2010}.  Although this paper
focuses exclusively on problems involving the recovery of a sparse
vector, similar ideas can be extended extend to other low-dimensional
structures such as group-sparse vectors and low-rank matrices, as
discussed in the statistical context~\cite{NegRavWaiYu09}.


\section{Problem set-up and algorithm description}
\label{SecBackground}

Given a subset $\Parset \subseteq \real^\usedim$ and a random variable
$Z$ taking values in a space $\SamSpace$, we consider an optimization
problem of the form
\begin{align}
\label{eqn:objective}
\opt & \in \arg \min_{\param \in \Parset} \E [\Loss(\param;\Sample)],
\end{align}
where $\Loss: \Parset \times \SamSpace \rightarrow \real$ is a given
loss function.  As is standard in stochastic optimization, we do not
have direct access to the \emph{expected loss function}
$\Lossbar(\param) \defn \Exs[\Loss(\param; \Sample)]$, nor to its
subgradients.  Rather, for a given a query point $\param \in \Parset$,
we observe a \emph{stochastic subgradient}, meaning a random vector
$\plstochgrad(\param) \in \real^\usedim$ such that $\E
[\plstochgrad(\param)] \in \partial \Lossbar(\param)$.  We then seek
to approach the optimum of the population objective $\Lossbar$ using a
sequence of these stochastic subgradients.

The goal of this paper is to design algorithms that are suitable for
solving the problem~\eqref{eqn:objective} when the optimum $\opt$ is
sparse.  In the simplest case, the optimum $\opt$ might be
\emph{exactly $\spindex$-sparse}, meaning that it has at most
$\spindex$ non-zero entries.  Our analysis is actually much more
general than this exact sparsity setting, in that we provide oracle
inequalities that apply to arbitrary vectors, and also guarantee fast
rates for vectors that are approximately sparse.  More precisely, for
any given subset $\Sset \subseteq \{1, \ldots, \pdim \}$ of
cardinality $|\Sset| = \spindex$, we provide upper bounds on the
optimization error that scale linearly with $\spindex$, and also
involve the residual term $\|\opt_{S^c}\|_1 \defn \sum_{j \in S^c}
|\opt_j|$.  For a general optimum $\opt$, the best bound is obtained
by choosing the subset $\Sset$ appropriately so as to balance these
two contributions to the error.

\subsection{Algorithm description}

In order to solve a sparse version of the
problem~\eqref{eqn:objective}, our strategy is to consider a sequence
of regularized problems of the form
\begin{equation}
  \label{eqn:regobj}
  \epochopt{\regpar} = \arg \min_{\param \in \Parset'}
  f_\regpar(\param) \quad \mbox{where} \quad f_\regpar(\param) \defn
  \Lossbar(\param) + \regpar \norm{\param}_1.
\end{equation}
Given a total number of iterations $\totiters$, our algorithm involves
a sequence of $\TOTEPOCH$ different epochs, where the regularization
parameter $\lambda > 0$ and the constraint set $\Parset' \subset
\Parset$ change from epoch to epoch.  More precisely, the epochs are
specified by:
\begin{itemize}
\item a sequence of natural numbers $\{T_i\}_{i=1}^{\TOTEPOCH}$, where
  $T_i$ specifies the length of the $i^{th}$ epoch and
  $\sum_{i=1}^{\TOTEPOCH} T_i = \totiters$,
\item a sequence of positive regularization weights
  $\{\regpar_i\}_{i=1}^{\TOTEPOCH}$, and
\item a sequence of positive radii $\{R_i\}_{i=1}^{\TOTEPOCH}$ and
  $\usedim$-dimensional vectors $\{\proxcenter{i}\}_{i=1}^{\TOTEPOCH}$,
  which specify the constraint set 
\begin{align}
\label{EqnDefnConSet}
\Parset(R_i) & \defn \big \{ \param \in \Parset \, \mid \, \|\param -
\proxcenter{i}\|_\pval \leq R_i \big \}
\end{align}
that is used throughout the $i^{th}$ epoch.
\end{itemize}
We initialize the algorithm in the first epoch with $\proxcenter{1} =
0$, and with any radius $R_1$ that is an upper bound on $\|\opt\|_1$.
The norm $\|\cdot\|_\pval$ used in defining the constraint set
$\Omega(R_i)$ is specified by \mbox{$\pval = \frac{2 \log
    \pdim}{2\log\pdim - 1}$,} where the rationale for this particular
choice will be provided momentarily. \\

At a high level, the goal of the $i^{th}$ epoch is to perform the
update $\proxcenter{i} \mapsto \proxcenter{i+1}$, in such a way that
we are guaranteed that $\|\proxcenter{i+1} - \opt\|^2_1 \leq
R^2_{i+1}$ for each $i = 1, 2, \ldots$.  We choose the radii so as to
decay geometrically as $R^2_{i+1} = R^2_i/2$, so that upon termination
of the $\TOTEPOCH^{th}$ epoch, we have $\|\proxcenter{\TOTEPOCH} -
\opt\|^2_1 \leq R^2_i/2^{\TOTEPOCH}$.  In order to perform the update
$\proxcenter{i} \mapsto \proxcenter{i+1}$, we run $T_i$ rounds of the
stochastic dual averaging algorithm~\cite{Nesterov09} on the
regularized objective function
\begin{align}
\label{eqn:epochregobj}
\min_{\param \in \Parset(R_i)} \big \{ \Lossbar(\param) + \regpar_i
\norm{\param}_1 \big \}.
\end{align}
When applied to this objective function in the $i^{th}$ epoch, the
dual averaging algorithm operates on stochastic subgradients of the
cost function $\Lossbar(\param) + \regpar_i \norm{\param}_1$, and
using a sequence of step sizes $\{\mystep{t}\}_{t=0}^{T_i}$, it
generates two sequences of vectors $\{\diter{t}\}_{t=0}^{T_i}$ and
$\{\iter{t}\}_{t=0}^{T_i}$, initialized as $\diter{0} = 0$ and
$\iter{0} = \proxcenter{i}$.  At iteration $t = 0, 1, \ldots, T_i$, we
let $\stochgrad{t}$ be a stochastic subgradient of $\Lossbar$ at
$\iter{t}$, and we let $\reggrad{t}$ be any element of the
subdifferential of the $\ell_1$-norm $\|\cdot\|_1$ at $\iter{t}$.
Consequently, the vector $\E[\stochgrad{t}] + \regpar_i \reggrad{t}$
is an element of the sub-differential of $\Lossbar(\param) + \regpar_i
\|\param\|_1$ at $\iter{t}$.  The stochastic dual average update at
time $t$ performs the mapping $(\diter{t}, \iter{t}) \mapsto
(\diter{t+1}, \iter{t+1})$ via the recursions
\begin{subequations}
\label{EqnOverallDual}
\begin{align}
\diter{t+1} & = \diter{t} + \stochgrad{t} + \regpar_i \reggrad{t},
\quad \mbox{and} \\
\label{eqn:regdualavg}
\iter{t+1} & = \arg \min_{\param \in \Parset(R_i)} \big\{ \mystep{t+1}
\inprod{\diter{t+1}}{\param} + \prox_{\proxcenter{i},
  \radius_i}(\param) \big\},
\end{align}
\end{subequations}
where the prox function $\prox$ is specified
below~\eqref{EqnDefnProx}. The pseudocode describing the overall
procedure is given in Algorithm~\ref{alg:epochdualavg}. 

In the stochastic dual averaging updates~\eqref{EqnOverallDual}, we
use the prox function
\begin{equation}
\label{EqnDefnProx}
  \prox_{\proxcenter{i}, \radius_i} (\param) =
  \frac{1}{2(p-1)\radius_i^2}\,\norm{\param - \proxcenter{i}}_p^2,
  \quad \mbox{where} \quad \pval = \frac{2\log \pdim}{2\log\pdim - 1}.
\end{equation}
This particular choice of the prox-function and the specific value of
$p$ ensure that the function $\prox$ is strongly convex with respect
to the $\ell_1$-norm; it has been used previously for sparse
stochastic optimization (see
e.g.~\cite{NemirovskiYu83,Shalev-ShwartzTe11,DuchiShSiTe10}). In most
of our examples, we consider the parameter space $\Parset = \R^\pdim$
and owing to our choice of the prox-function and the feasible set in
the update~\eqref{eqn:regdualavg}, we can compute $\iter{t+1}$ from
$\diter{t+1}$ in closed form.  See Appendix~\ref{AppClosed} for
further details.
\begin{algorithm}[h]
\caption{\myalg \quad (\myalgshort)}
  \begin{algorithmic}
    \REQUIRE Epoch length schedule $\{T_i\}_{i=1}^{\TOTEPOCH}$, initial
    radius $\radius_1$, step-size multiplier $\stepsize$,
    prox-function $\prox$, initial prox-center $\proxcenter{1}$,
    regularization parameters $\lambda_i$.
    \FOR{Epoch $i=1,2, \ldots, \TOTEPOCH$}
   \STATE Initialize $\diter{0} = 0$ and $\iter{0} = \proxcenter{i}$.
    \FOR{Iteration $t=0,1,\ldots, T_i-1$}
   \STATE Update $(\diter{t}, \iter{t}) \mapsto (\diter{t+1},
   \iter{t+1})$ according to rule~\eqref{EqnOverallDual} with step
   size $\mystep{t} = \stepsize/\sqrt{t}$.
   \ENDFOR
   \STATE Set $\proxcenter{i+1} = \frac{\sum_{t=1}^{T_i}
     \iter{t}}{T_i}$.
   \STATE Update $\radius^2_{i+1} = \radius^2_i/2$.
   \ENDFOR
   \STATE \textbf{Return} $\proxcenter{\TOTEPOCH+1}$
  \end{algorithmic}
  \label{alg:epochdualavg}
\end{algorithm}

It is worth noting that our update rule, in taking the subgradient of the
$\ell_1$-norm, is different from previous approaches inspired by
Nesterov's composite minimization strategy~\cite{Nesterov07}, which
compute a prox-mapping involving the
$\ell_1$-norm~\cite{DuchiSi09c,Xiao10,DuchiShSiTe10}. Our results do
extend in an obvious way to computing such an exact composite
prox-mapping. However, even when $\Parset = \R^\pdim$, computing this
exact prox-mapping with the $\ell_p$ norm constraint in our update
rule~\eqref{eqn:regdualavg} has a complexity $\order(\pdim^2)$, which
is prohibitive in high dimensions.  In contrast, our update enjoys an
$\order(\pdim)$ complexity.

\subsection{Conditions}
\label{sec:conditions}

Having defined our algorithm, we now discuss the conditions on the
objective function \mbox{$\Lossbar(\param) = \Exs[\Loss(\param; Z)]$}
and stochastic gradients that underlie our analysis.  We begin with
two conditions on the objective function.

\begin{assumption}[Local Lipschitz condition]
\label{ass:explips}
For each $\radius > 0$, there is a constant $\lips = \lips(\radius)$
such that
\begin{align}
\label{EqnLocalLipschitz}
|\Lossbar(\param) - \Lossbar(\paramother)| & \leq \lips \, \|\param -
\paramother\|_1
\end{align}
for all $\param, \paramother \in \Omega$ such that $\norm{\param -
  \opt}_1 \leq \radius$ and $\norm{\paramother - \opt}_1 \leq
\radius$. 
\end{assumption}
\noindent For instance, this condition holds whenever $\norm{\nabla
  \Lossbar(\param)}_\infty \leq \lips$ for all $\param$ such that
$\norm{\param}_1 \leq \radius$.  In the sequel, we provide several
examples of loss functions whose gradients are bounded in this
$\ell_\infty$-sense. \\

As mentioned, our goal is to obtain fast rates for objective functions
that are (locally) strongly convex.  Accordingly, our next step is to
provide a formal definition of this condition:
\begin{assumption}[Local strong convexity (LSC)]
  \label{ass:rsc}
The function $\Lossbar: \Parset \rightarrow \real$ satisfies an
$\radius$-local form of strong convexity (LSC) if there is a
non-negative constant $\lossrsc = \lossrsc(\radius)$ such that
\begin{align}
\label{EqnRSC}
 \Lossbar(\paramother) & \geq \Lossbar(\param) + \inprod{\nabla
   \Lossbar(\param)}{\paramother - \param} + \frac{\lossrsc}{2}
 \norm{\param - \paramother}_2^2.
\end{align}
for any $\param, \paramother \in \Parset$ with $\norm{\param}_1 \leq
\radius$ and $\norm{\paramother}_1 \leq \radius$.
\end{assumption}

\noindent Some of our results concerning stochastic optimization for
finite pools of examples are based upon a further weakening of the
local strong convexity condition, which will be referred to as local
restricted strong convexity (see equation~\eqref{EqnLRSC}).  Such
conditions have been used before in both statistical and
optimization-theoretic analyses of sparse high-dimensional
problems~\cite{AgarwalNeWa10,BickelRiTs09,BuhlmannGe2011,NegRavWaiYu09}. \\

Our final condition concerns the mechanism that produces the
stochastic subgradient $\plstochgrad(\param)$ of the cost function
$\Lossbar$ at $\param \in \Parset$.  It is a tail condition on the
error $\graderr(\param) \defn \plstochgrad(\param) -
\Exs[\plstochgrad(\param)]$.

\begin{assumption}[Sub-Gaussian stochastic gradients]
  \label{ass:subgauss}
For all $\param$ such that $\norm{\param - \opt}_1 \leq \radius$,
there is a constant $\noise = \noise(\radius)$ such that
\begin{align}
\label{EqnSubGauss}
\E \big[ \exp(\norm{\graderr(\param)}_\infty^2/\noise^2) \big ] \leq
\exp(1).
\end{align}
\end{assumption}
\noindent Clearly, this condition holds whenever the error vector
$\graderr(\param)$ has bounded components.  More generally, the
bound~\eqref{EqnSubGauss} holds with $\noise^2 = \order(\log \pdim)$
whenever each component of the error vector has sub-Gaussian
tails.\footnote{A zero mean random variable $Z$ is sub-Gaussian with
  parameter $\gamma$ if $\Exs[e^{t Z}] \leq \exp(\gamma^2 t^2/2)$ for
  all $t \in \real$.}

\subsection{Some illustrative examples}
\label{sec:examples}

In this section, we describe some examples that satisfy the previously
stated conditions.  These examples also help to clarify how the
parameters of interest can be computed or bounded in different applied
scenarios.

\begin{exas}[Classification under Lipschitz losses] 
\label{example:lipschitz}
In binary classification, the samples consist of pairs $\sample =
(\samplex, \sampley) \in \R^\pdim \times \{-1,1\}$.  The vector
$\samplex$ represents a set of $\pdim$ features or covariates, used to
predict the class label $\sampley$.  One way in which to predict the
label is via a linear classifier, which makes classification decisions
according to the rule $\samplex \mapsto
\sign(\inprod{\param}{\samplex})$.  The vector of weights $\param \in
\real^\pdim$ is estimated by minimizing an appropriately chosen loss
function, of which many take the form $\Loss(\param; \sample) =
\phi(\sampley \, \inprod{\theta}{\samplex})$ for a function $\phi:
\real \rightarrow \real_+$.  Common choices include the hinge loss
function
\begin{align}
\label{EqnDefnHinge}
\phihinge(\alpha) & \defn \underbrace{\max \{ 1 -\alpha, 0
  \}}_{(1-\alpha)_+},
\end{align}
which underlies the support vector machine, or the logistic loss
function \mbox{$\philog( \alpha) = \log(1 + \exp(-\alpha))$.}

Given a distribution $\P$ over $\SamSpace$, which can be either the
population distribution or the empirical distribution over a finite
sample, a common strategy is to draw $(\samplex_t, \sampley_t) \sim
\P$ at iteration $t$, and use the (stochastic) subgradient
$\stochgrad{t} = \nabla \Loss(\param;(\samplex_t, \sampley_t)) =
\samplex_t \sampley_t \phi'(\sampley_t
\inprod{\param}{\samplex_t})$. We now illustrate how the assumptions
of Section~\ref{sec:conditions} are satisfied in this setting.

\begin{itemize}
  \item \textbf{Locally Lipschitz:} In both of the above examples, the
    underlying function $\phi$ is actually globally Lipschitz with
    parameter one. Thus, we have the bound
 \begin{equation*}
      \lips \leq \E[|\phi'\big( \sampley
        \inprod{\param}{\samplex} \big)| \; \norm{\samplex}_\infty]
      \leq \E\norm{\samplex}_\infty.
    \end{equation*}
    Often, the data satisfies the normalization
    $\norm{\samplex}_\infty \leq \xbound$, in which case we get $\lips
    \leq \xbound$. More generally, we often have sub-Gaussian or
    sub-exponential tail conditions~\cite{BulKoz} on the marginal
    distribution of each coordinate of $\samplex$, in which case the
    same condition holds with $\lips = \order(\sqrt{\log\pdim}))$ or
    $\lips = \order(\log \pdim)$ respectively.

\item \textbf{LSC:} When the expectation~\eqref{eqn:objective} is
  under the population distribution, the above examples satisfy the
  local strong convexity condition. Here we focus on the example of
  the logistic loss, where $\psi(\alpha) \defn \philog''(\alpha) \; =
  \; \exp(\alpha)/ (1+\exp(\alpha))^2$ is its second derivative.

Considering the case of zero-mean covariates, and letting $\sigmin
(\CovMat)$ denote the minimum eigenvalue of the covariance matrix
$\CovMat = \Exs[\samplex \samplex^T]$, a second-order Taylor series
expansion yields
 \begin{equation*}
 \Lossbar(\paramother) - \Lossbar(\param) - \ip{\nabla
   \Lossbar(\param)}{\paramother - \param} = \frac{\psi
   (\ip{\paramtil}{\samplex})}{2}\norm{\CovMat^{1/2}(\param -
   \paramother)}_2^2 \; \stackrel{(i)}{\geq} \frac{\psi (\xbound
   \radius) \sigmin(\CovMat)}{2} \norm{\param - \paramother}_2^2,
 \end{equation*}
 where $\paramtil = a\param + (1-a)\paramother$ for some $a \in
 (0,1)$.  Note that the lower bound (i) follows from H\"older's
 inequality---that is, $| \ip{\paramtil}{\samplex} | \leq
 \norm{\paramtil}_1 \norm{\samplex}_{\infty} \leq
 \radius\xbound$. Putting together the pieces, we conclude that
 $\lossrsc \geq \psi(\xbound\radius)\sigmin(\CovMat)$ in this example.
 When sampling from a finite pool, we require an analogous condition,
 known as restricted strong convexity, to hold for the sample version;
 see Section~\ref{sec:lipschitz-sample} for further details.
\item \textbf{Sub-Gaussian gradients:} For covariates bounded in
  expectation \mbox{$\E\norm{\samplex}_\infty \leq \xbound$,} this
  condition is also relatively straightforward to verify.  A simple
  calculation identical to the verification of the Lipschitz condition
  above yields that
\begin{equation*}
\norm{\graderr(\param)}_\infty = \norm{\nabla \Loss(\param;(\samplex,
  \sampley)) - \nabla \Lossbar(\param)}_\infty \leq \norm{\nabla
  \Loss(\param;(\samplex, \sampley))}_\infty + \norm{\nabla
  \Lossbar(\param)}_\infty \leq 2\xbound.
    \end{equation*}
Thus, by setting $\noise^2 = (2 \xbound)^2$, we find that $\E \exp
\left( \frac{\norm{\graderr(\param)}^2_\infty}{4 \xbound^2} \right)
\leq \exp(1)$.  If instead of a boundedness condition, we assume that
elements of the vector $\noise$ have sub-Gaussian tails, then the
condition will hold with $\noise^2 = \order(\log \pdim)$, using
standard results on sub-Gaussian maxima (e.g.,~\cite{BulKoz}).

\end{itemize}
\end{exas}


\noindent We now turn to the problem of least-squares regression.
\bexs[Least-squares regression]
\label{example:leastsquares}
In the regression set-up, the samples are of the form \mbox{$\sample =
  (\samplex, \sampley) \in \R^{\pdim} \times \R$,} and the
least-squares estimator is obtained by minimizing the quadratic loss
$\Loss(\param; (\samplex, \sampley)) = (\sampley -
\ip{\param}{\samplex})^2/2$. To illustrate the conditions more
clearly, let us suppose to start, relaxing this condition momentarily,
that our samples are generated according to a linear model
\begin{equation}
\label{eqn:linearmodel}
\sampley = \ip{\samplex}{\opt} + \samplew,
\end{equation} 
where $\samplew \sim \normal{\gaussnoise^2}$ is observation noise, and
the covariate vectors $\samplex$ are zero-mean with covariance matrix
$\CovMat = \Exs[\samplex \samplex^T]$.  Under this condition, we have
\begin{align*}
\Lossbar(\param) \; = \; \E \big[ \Loss(\param; (\samplex, \sampley) )
  \big] \; = \; \frac{1}{2} (\param - \opt)^T \CovMat (\param - \opt)
\; = \; \frac{1}{2} \| \CovMat^{1/2} (\param - \opt) \|_2^2.
\end{align*}
Consequently, the minimizer of $\Lossbar$ is given by $\opt$, the true
parameter in the linear model~\eqref{eqn:linearmodel}.  We now proceed
to verify that our conditions from Section~\ref{sec:conditions} are
satisfied for this model.

\begin{itemize}
\item \textbf{Locally Lipschitz:} For the quadratic loss, we no longer
  have a global Lipschitz condition; instead, the local Lipschitz
  parameter $\lips(\radius)$ depends on the radius $\radius$, and the
  covariance matrix $\CovMat$ via the quantity \mbox{$\CovMax \defn
    \max_{j} \CovMat_{jj}$}.  More specifically, we have
\begin{equation*}
\lips(\radius) \stackrel{(a)}{\leq} \norm{\CovMat (\param -
  \opt)}_\infty \; \leq \; \max_{j,k} |\CovMat_{jk}| \; \norm{\param -
  \opt}_1 \stackrel{(b)}{\leq} \CovMax \radius,
\end{equation*}
where step (a) exploits H\"{o}lder's inequality, and inequality (b)
follows since $\CovMat$ is a positive semidefinite matrix.
\item \textbf{LSC:} Again, let us consider the case when $\Lossbar$ is
  defined via an expectation taken under the population
  distribution. We then have
\begin{equation*}
    \Lossbar(\paramother) - \Lossbar(\param) - \ip{\nabla
      \Lossbar(\param)}{\paramother - \param} =
    \frac{\norm{\CovMat^{1/2}(\param - \paramother)}_2^2}{2} \geq
    \frac{\sigmin(\CovMat)}{2}\norm{\param - \paramother}_2^2,
  \end{equation*}
so that $\lossrsc = \sigmin(\CovMat)$ in this example.
\item \textbf{Sub-Gaussian gradients:} Once again we assume that the
  design is bounded in $\ell_\infty$-norm, that is
  $\norm{\samplex}_\infty \leq \xbound$. It can be shown that
  Assumption~\ref{ass:subgauss} is satisfied with
  \begin{equation*}
    \noise^2(\radius) = 24 \xbound^4 \radius^2 + 36 \xbound^2
    \gaussnoise^2.
  \end{equation*}
  See Section~\ref{sec:proof-leastsquares} for details of this
  calculation.
\end{itemize}

In practice, the linear model assumption~\eqref{eqn:linearmodel} is
not likely to hold exactly, but the validity of our three conditions
can still be established under reasonable tail conditions on the
covariate-response pair.  In particular, it can be shown that same
local Lipschitz condition continues to hold with $\lips(\radius) =
2\CovMax\radius$, and the RSC condition also remains unchanged.  In
order to establish the sub-Gaussian condition
(Assumption~\ref{ass:subgauss}), we need to make assumptions about the
tail behavior of our samples $(\samplex_t, \sampley_t)$. It suffices
to assume that the distribution of the vectors $\samplex$ and the
conditional distribution $\sampleY|\sampleX$ is also
sub-Gaussian. Under these conditions, obtaining explicit bounds on
$\noise_i$ in terms of these sub-Gaussian parameters is analogous to
our calculations above, and is omitted here.  \eexs


\section{Main results and their consequences}
\label{sec:results}

We are now in a position to state the main results of this work,
regarding the convergence properties of
Algorithm~\ref{alg:epochdualavg}. Below we present two sets of
results.  Our first result (Theorem~\ref{thm:lipschitz}) applies to
problems for which the Lipschitz and sub-Gaussianity assumptions hold
over the entire parameter set $\Parset$, and the RSC assumption holds
uniformly for all $\param$ satisfying $\norm{\param}_1 \leq
\radius_1$, where $\radius_1$ is the initial radius. Examples include
classification with globally Lipschitz losses, such as the hinge and
logistic losses discussed in Example~\ref{example:lipschitz}. Our
second result (Theorem~\ref{thm:leastsquares}) applies to
least-squares loss, which is not Lipschitz on $\real^\pdim$, and
requires a somewhat more delicate treatment. Both our results follow
from a common machinery, and build off of standard convergence results
for the dual averaging algorithm~\cite{Nesterov09,Xiao10}. For stating
our results, we will assume that Assumptions~\ref{ass:explips}
and~\ref{ass:subgauss} hold with constants $\lips_i \defn
\lips(\radius_i)$ and $\noise_i \defn \noise(\radius_i)$ at epoch
$i$. Given a constant $\devcon > 0$ governing the probability of error
in our results, we also define $\devcon^2_i = \devcon^2 + 24\log i$ at
epoch $i$. Both the theorems below are based on the choice of epoch
lengths
\begin{align}
  \label{eqn:epochlength-simple}
  T_i & \defn \const_1\,\left
  [\frac{\spindex^2}{\lossrsc^2\radius_i^2} \left( (\lips_i^2 +
    \noise_i^2) \log \pdim + \devcon_i^2\noise_i^2\right) + \log \pdim
    \right],
\end{align}
where $\const_1$ is a suitably chosen universal constant.

\subsection{Optimal rates for Lipschitz losses}

We begin with a setting quite standard in the optimization literature,
in which the loss function is globally Lipschitz and the noise in our
stochastic gradients is uniformly sub-Gaussian. More formally, we
assume that there are constants $(\lips, \noise)$ such that,
independently of the choice of radius $\radius$,
Assumptions~\ref{ass:explips} and~\ref{ass:subgauss} hold with
$\lips(\radius) \equiv \lips$ and $\noise(\radius) \equiv \noise$.
There are many common examples in machine learning where these
assumptions are met, some of which are outlined in
Example~\ref{example:lipschitz}. We also use $\lossrsc$ to denote the
strong convexity constant $\lossrsc(\radius_1)$ in
Assumption~\ref{ass:rsc}.

For a total of $\totiters$ iterations in
Algorithm~\ref{alg:epochdualavg}, we state our results for the
parameter $\finalparam{\totiters} = \proxcenter{(\TOTEPOCH+1)}$ where
we recall that $\TOTEPOCH$ is the total number of epochs completed in
$\totiters$ iterations.

\subsubsection{Main theorem and some remarks}

Our main theorem provides an upper bound on the convergence rate on
our algorithm as a function of the iteration number $\totiters$,
dimension $\pdim$, strong convexity constant $\lossrsc$, Lipschitz
constant $\lips$, and some additional terms involving the sparsity of
the optimal solution $\opt$.  More precisely, for each subset $\Sset
\subseteq \{1, 2, \ldots, \pdim\}$ of cardinality $\spindex$, we
define the quantity
\begin{align}
\label{eqn:totalslackdefn-simple}
  \totalslacksqsimple & \defn  \frac{\coneslack^2}{\spindex},
\end{align}
where $\|\opt_{\Sset^c}\|_1 = \sum_{j \in \Sset^c} |\opt_j| = \sum_{j
  \notin \Sset} |\opt_j|$ is the $\ell_1$-norm of terms outside
$\Sset$. The behavior of this quantity can be used to measure the
degree of sparsity in the optimum $\opt$. For instance, we have
$\totalslacksqsimple = 0$ if and only if $\opt$ is supported on
$\Sset$. Given a constant $\devcon > 0$, we also define the shorthand
\begin{equation}
  \slackiters_{\totiters} = \log_2 \left[\frac{\lossrsc^2 \radius_1^2
      \totiters}{\spindex^2((\lips^2 + \noise^2) \log \pdim +
      \devcon^2 \noise^2)}\right] \log \pdim.
  \label{eqn:slackiters}
\end{equation}
With this notation, we have the following result:
\begin{theorem}
 \label{thm:lipschitz}
 Suppose the expected loss $\Lossbar$ satisfies
 Assumptions~\ref{ass:explips}---~\ref{ass:subgauss} with parameters
 \mbox{$\lips(\radius) \equiv \lips$,} $\lossrsc$ and
 \mbox{$\noise(\radius) \equiv \noise$} respectively, and we perform
 updates~\eqref{EqnOverallDual} with epoch
 lengths~\eqref{eqn:epochlength-simple}, and regularization/steplength
 parameters
 \begin{equation}
    \regpar_i^2 = \frac{\radius_i \lossrsc}{\spindex
      \sqrt{T_i}}\sqrt{(\lips^2 + \noise^2)\log\pdim + \devcon_i^2
      \noise^2}, \quad \mbox{and} \quad \mystep{t} = 5 \radius_i
    \sqrt{\frac{\log \pdim}{(\lips^2 + \regpar_i^2 + \noise^2)t}}.
    \label{eqn:paramsettings-simple}
  \end{equation}
Then for any subset $\Sset \subseteq \{1,\dots,\pdim\}$ of cardinality
$\spindex$ and any $\totiters \geq 2\slackiters_{\totiters}$, there is
a universal constant $\PREFACT$ such that
\begin{align}
\label{EqnEpochBound}
\norm{\finalparam{\totiters} - \opt}_2^2 & \leq \PREFACT \, \left[
  \frac{\spindex}{\lossrsc^2\totiters} \left((\lips^2 + \noise^2) \log
  \pdim + \noise^2(\devcon^2 + \log
  \frac{\slackiters_{\totiters}}{\log \pdim}) \right) + \totalslacksqsimple
  \right]
\end{align}
with probability at least \mbox{$1 - 6\exp(-\devcon^2/12)$.}
\end{theorem}

As with earlier work on multi-step methods for strongly convex
objectives~\cite{JuditskyNes10,LanGh2010}, the theorem predicts an
overall convergence rate of $O(\frac{1}{\lossrsc^2 \totiters})$; under
our assumptions, this rate is known to be the best
possible~\cite{NemirovskiYu83}.  Apart from this scaling, the other
interesting factors in the bound are the logarithmic scaling in the
dimension $\pdim$, and the trade-off between the two terms: the first
of which scales linearly in a chosen sparsity level $\spindex$, and
the second term $\totalslacksqsimple$ represents a form of
approximation error.  As a concrete instance, if the optimum $\opt$ is
actually supported on a subset $A \subset \{1, 2, \ldots, \pdim \}$ of
cardinality $\spindex= |A|$, then choosing $S = A$ in the
bound~\eqref{EqnEpochBound} yields an overall convergence rate of
$\order(\frac{\spindex \log \pdim}{\lossrsc^2 \, \totiters})$.

It is worthwhile comparing the convergence rate in
Theorem~\ref{thm:lipschitz} to alternative methods. A standard
approach to minimizing the objective~\eqref{eqn:objective} would be to
perform stochastic gradient descent directly on the objective, instead
of considering our sequence of regularized
problems~\eqref{eqn:regobj}.  Under the assumptions of
Theorem~\ref{thm:lipschitz}, the expected loss is strongly convex with
respect to the $\ell_2$-norm, so that stochastic gradient descent
(SGD) would converge at a rate $\order((\liptil^2 +
\noisetil^2)/(\lossrsc^2 \totiters))$, for constants $\liptil$ and
$\noisetil$ chosen to satisfy the bounds
\begin{align}
\norm{\nabla \Lossbar(\param)}_2 \leq \liptil, \quad \mbox{and} \quad
\E \Big[\exp \big(\norm{\graderr(\param)}_2^2/\noisetil^2 \big) \Big]
\leq \exp(1).
\end{align}
Under the assumptions of Example~\ref{example:lipschitz}, we find that
it suffices to choose $\liptil^2 = \xbound \pdim$ and similarly for
$\noisetil$, so that SGD would converge at rate
$\order(\frac{\pdim}{\lossrsc^2 \totiters})$.  This generic guarantee
scales linearly in the problem dimension $\pdim$, and fails to exploit
any sparsity inherent to the problem.  The key difference between this
naive application of stochastic gradient descent and our approach is
that since we minimize a regularized objective~\eqref{eqn:regobj}, our
iterates tend to be (approximately) sparse. As a result, we have a
form of local strong convexity not only in $\ell_2$-norm but also with
respect to $\ell_1$-norm; this is a key observation in exploiting
sparsity and strong convexity at the same time.  Another standard
approach is to perform mirror descent or dual averaging, using the
same prox-function as Algorithm~\ref{alg:epochdualavg} but without
breaking it up into epochs. As mentioned in the introduction, this
vanilla single-step method fails to exploit the strong convexity of
our problem and obtains the inferior convergence rate $\order(\spindex
\, \sqrt{\log \pdim / \totiters})$.

A different proposal, closer in spirit to our approach, is to minimize
a similar regularized objective~\eqref{eqn:regobj}, but with a fixed
value of $\regpar$ instead of the decreasing schedule of $\regpar_i$
used in Theorem~\ref{thm:lipschitz}. In fact, it can be obtained as a
simple consequence of our proofs that setting $\regpar = \noise
\sqrt{\log \pdim / \totiters}$ leads to an overall convergence rate of
$\order \left(\frac{1}{\lossrsc^2} \frac{\spindex \log
  \pdim}{\totiters} \right)$, a result analogous to the guarantee of
Theorem~\ref{thm:lipschitz}.  Indeed, this procedure can be understood
as applying the algorithm of Juditsky and
Nesterov~\cite{JuditskyNes10} to the problem~\eqref{eqn:objective},
but where the bounds are obtained using the additional technical
machinery introduced in this paper. However, with this fixed setting
of $\regpar$, the initial epochs tend to be much longer than required
for reducing the error by a factor of one half.  Indeed, our setting
of $\regpar$ is based on minimizing the upper bound at each epoch, and
leads to substantially improved performance in our numerical
simulations as well.  The benefits of slowly decreasing the
regularization in the context of deterministic optimization were also
noted in the recent work of Xiao and Zhang~\cite{XiaoZh12}.

It is instructive to further simplify the the bounds by making further
assumptions, allowing us to quantify these terms concretely.  We do in
the following sections.

\subsubsection{Some illustrative corollaries}

We start with a corollary for the setting where the optimum $\opt$ is
supported on a subset $\Sset$ of cardinality $\spindex$, where
$\spindex$ is a sparsity index between $1$ and $\pdim$.  For these
corollaries, so as to facilitate comparison with minimax lower bounds,
we use $\newdevcon = \devcon^2/(\log \pdim)$ as the parameter in
specifying the high-probability guarantees. Under these conditions, we
recall our earlier notation
$\slackiters_{\totiters}$~\eqref{eqn:slackiters} further simplifies to
\begin{equation*}
  \slackiters_{\totiters} = \log_2 \left[\frac{\lossrsc^2 \radius_1^2
      \totiters}{\spindex^2\log \pdim (\lips^2 + \noise^2 (1 +
      \newdevcon))}\right] \log \pdim.
\end{equation*}
Within this setup, we have the following corollary of
Theorem~\ref{thm:lipschitz}.

\begin{corollary}
\label{cor:lipschitz-pop-exact}
Under the conditions of Theorem~\ref{thm:lipschitz}, assume further
that $\opt$ takes non-zero values only on a subset $\Sset \subseteq
\{1,\dots,d\}$ of size $\spindex$. Then for all $\totiters \geq
2\slackiters_{\totiters}$, there is a universal constant $\PREFACT$
such that
\begin{align}
\norm{\finalparam{\totiters} - \opt}_2^2 & \leq \PREFACT \, \left[
  \frac{\{\lips^2 + \noise^2(1 + \newdevcon) \}}{\lossrsc^2} \; \frac{
    \spindex \log \pdim}{\totiters} + \frac{\spindex
    \noise^2}{\lossrsc^2 \totiters} \log
  \frac{\slackiters_{\totiters}}{\log\pdim} \right]
\end{align}
with probability at least \mbox{$1 - 6\exp(-\newdevcon \log
  \pdim/12)$.}
\end{corollary}

The corollary follows directly from Theorem~\ref{thm:lipschitz} by
noting that $\totalslacksqsimple = 0$ under our assumptions. It is
useful to note that the results on recovery for generalized linear
models presented here match (up to $\order(\log\log \totiters$ factors)
those that have been developed in the statistics
literature~\cite{NegRavWaiYu09,Geer08}, which are optimal under the
assumptions on the design vectors.

Theorem~\ref{thm:lipschitz} also applies to the case when the optimum
$\opt$ is not exactly sparse, but only approximately so. Such notions
of approximate sparsity can be formalized by enforcing a certain decay
rate on the magnitudes, when ordered from smallest to largest.  Here
we consider the notion of $\ell_\qpar$-``ball'' sparsity: given a
parameter $\qpar \in (0,1]$ and a radius $\radq$, consider the set of
  all vectors such that
\begin{align}
\Ball_\qpar(\radq) & \defn \Big \{\param \in \R^\pdim \; \mid \;
\sum_{j=1}^{\pdim} |\param_j|^\qpar \leq \radq \Big \}.
\end{align}
For $\qpar = 1$, this set reduces to an $\ell_1$-ball, whereas for
$\qpar \in (0,1)$, it is a non-convex but star-shaped set contained
within the $\ell_1$-ball. With these assumptions, our earlier notation
$\slackiters_{\totiters}$ further simplifies to
\begin{equation*}
  \slackiters_{\totiters} = \log_2 \left[\frac{\radius_1^2}{\radq^2}
    \left(\frac{\lossrsc^2 \totiters}{\log \pdim (\lips^2 + \noise^2
      (1 + \newdevcon) )} \right)^{-\qpar} \right] \log \pdim.
\end{equation*}
The following corollary captures the convergence of our updates for
such problems.
\begin{corollary}
\label{cor:lipschitz-pop-weak}
 Under the conditions of Theorem~\ref{thm:lipschitz}, suppose moreover
 that $\opt \in \Ball_\qpar(\radq)$ for some $\qpar \in (0,1]$. Then
   there is a universal constant $\PREFACT$ such that for all
   $\totiters \geq 2\slackiters_{\totiters}$, we have
\begin{align}
 \norm{\finalparam{\totiters} - \opt}_2^2 & \leq \PREFACT \; \radq \;
 \left[ \left \{ \frac{\{ \lips^2 + \noise^2 (1 + \newdevcon) \} \log
     \pdim }{\lossrsc^2 \totiters} \right \}^{1-\qpar/2} + \left(
   \frac{\noise^2}{\lossrsc^2 \totiters} \right)^{1-\qpar/2}\,
   \frac{\radq}{((1 + \newdevcon)\log \pdim)^{\qpar/2}}\, \log
   \frac{\slackiters_{\totiters}}{\log \pdim} \right].
\end{align}
with probability at least \mbox{$1 - 6 \exp(-\newdevcon \log \pdim/12)$.}
\end{corollary}
Note that as $\qpar$ ranges over the interval $[0,1]$, reflecting the
degree of sparsity, the convergence rate ranges from
$\order(1/\totiters)$ for $\qpar = 0$ corresponding to exact sparsity,
to $\order(1 / \sqrt{\totiters})$ for $\qpar = 1$.  This is a rather
interesting trade-off, showing in a precise sense how convergence
rates vary quantitatively as a function of the underlying sparsity.
While it might seem like the worsening of rates as $\qpar$ increases
towards one defeats our original goal of obtaining fast rates by
leveraging strong convexity of our problem, this phenomenon is
unavoidable due to existing lower bounds.  More specifically, the
results on recovery for generalized linear models presented here
exactly match those that have been developed in the statistics
literature~\cite{NegRavWaiYu09,Geer08}, which are optimal under our
assumptions on the design vectors. The reason for this phenomenon is
that our goal of obtaining logarithmic dependence with the dimension
$\pdim$ requires strong convexity of the objective with respect to
$\ell_1$-norm, while our LSC assumption only guarantees strong
convexity with respect to the $\ell_2$-norm. For a sparse optimum
$\opt$, the local strong convexity assumption also translates into the
desired $\ell_1$-strong convexity, but the constant deteriorates as
$\qpar$ increases from zero to one.

\subsubsection{Stochastic optimization over finite pools}
\label{sec:lipschitz-sample}

A common setting for the application of stochastic optimization
methods in machine learning is when one has a finite pool of examples,
say $\{\sample_1, \ldots, \sample_\numobs \}$, and the goal is to
compute
\begin{align}
  \label{eqn:sampleobjective}
  \opt & \in \arg \min_{\param \in \Parset} \Big \{ \frac{1}{\numobs}
  \sum_{i=1}^\numobs \Loss(\param; \sample_i) \Big \}.
\end{align}
In this setting, a stochastic gradient $\plstochgrad(\param)$ can be
obtained by drawing a sample $\sample_j$ at random \emph{with
  replacement} from the pool $\{\sample_1, \ldots, \sample_\numobs
\}$, and returning the gradient $\nabla \Loss(\param; \sample_j)$,
which is unbiased as an estimate of the gradient of the sample
average~\eqref{eqn:sampleobjective}.

In many applications, the dimension $\pdim$ is substantially larger
than the sample size $\numobs$, in which case the sample loss defined
above can never be strongly convex. However, it can be
shown~\cite{RasWaiYu10,NegRavWaiYu09} that under suitable a condition,
the sample objective~\eqref{eqn:sampleobjective} does satisfy a
suitably restricted form of the LSC condition formalized in
Assumption~\ref{ass:rsc}, one that is valid even when $\pdim \gg
\numobs$. As a result, the generalized form of
Theorem~\ref{thm:lipschitz} we provide in
Section~\ref{sec:proof-lipschitz} does apply to this setting as well
and we can obtain the following corollary. We will present this result
only for settings where $\opt$ is exactly sparse, the extension to
approximate sparsity is identical to the above discussion for
obtaining Corollary~\ref{cor:lipschitz-pop-weak} from
Corollary~\ref{cor:lipschitz-pop-exact}. Moreover, we also specialize
to the logistic loss
\begin{align}
\label{EqnDefnLogisticLoss}
\Loss(\param;(\samplex,\sampley)) & \defn \log(1 + \exp(-\sampley
\ip{\param}{\samplex} ) ),
\end{align}
which suffices to illustrate the main aspects of the result.  We also
introduce the shorthand \mbox{$\psi(\alpha) = \exp(\alpha)/(1 +
  \exp(\alpha))^2$}, corresponding to the second derivative of the
logistic function. Before stating the corollary, we state a condition
on the design that is needed to ensure the RSC condition. The
condition is stated on the design matrix $\sampleX \in \R^{\numobs
  \times \pdim}$ with $x_i^T$ as its $i^{th}$ row.

\begin{assumption}[Sub-Gaussian design]
  The design matrix $\sampleX$ is sub-Gaussian with parameters
  $(\CovMat, \xnoise^2)$ if
  \begin{itemize}
  \item[(a)] Each row $\samplex_i \in \R^{\pdim}$ is sampled
    independently from a zero-mean distribution with covariance
    $\CovMat$, and
  \item[(b)] For any unit-norm vector $u \in \R^{\pdim}$, the random
    variable $\inprod{u}{\samplex_i}$ is sub-Gaussian with parameter
    $\xnoise$, meaning that $\Exs[\exp(t \inprod{u}{\samplex_i})] \leq
    \exp(t^2 \xnoise^2/2)$ for all $t \in \real$.
  \end{itemize}
\label{ass:subgauss-design}
\end{assumption}
In this setup, our definition of
$\slackiters_{\totiters}$~\eqref{eqn:slackiters} is modified to
\begin{equation*}
  \slackiters_{\totiters} = \log_2 \left[\frac{\sigmin^2(\CovMat)
      \psi^2(2\xbound \radius_1) \radius_1^2 \totiters}{\spindex^2
      \xbound^2(5 + 4\newdevcon)\log \pdim } \right] \log \pdim.
\end{equation*}
We can now state a convergence result for this setup.
\begin{corollary}[Logistic regression for finite pools]
\label{cor:lipschitz-samp}
Consider the finite-pool loss~\eqref{eqn:sampleobjective}, based on
$\numobs$ i.i.d. samples from a sub-Gaussian design with parameters
$(\CovMat, \xnoise^2)$. Suppose further that
Assumptions~\ref{ass:explips} and \ref{ass:subgauss} are satisfied and
the optimum $\opt$ of the problem~\eqref{eqn:sampleobjective} is
$\spindex$-sparse. Then there are universal constants $(\PREFACT,
\const_1, \const_2, \const_3)$ such that for all $\totiters \geq 2
\slackiters_{\totiters}$ and $\numobs \geq \const_3\,\frac{\log\pdim}
           {\sigmin^2(\CovMat)} \max(\sigmin^2(\CovMat), \xnoise^4)$,
           we have
\begin{align}
 \norm{\finalparam{\totiters} - \opt}_2^2 & \leq
 \frac{\PREFACT}{\sigmin^2(\CovMat)} \; \frac{\spindex \log
   \pdim}{\totiters} \Big \{ \frac{1}{\psi^2(2 \xbound \radius_1)}
 \big \{ \xbound^2 (1 + \newdevcon) \big \}\Big \} + \PREFACT
 \frac{\spindex \noise^2}{\sigmin^2(\CovMat) \psi^2(2\xbound
   \radius_1) \totiters} \log\frac{\slackiters_{\totiters}}{\log
   \pdim}. 
\end{align}
with probability at least \mbox{$1 - 2\exp(-\const_1\numobs
  \min(\sigmin^2(\CovMat)/\xnoise^4,1)) - 6 \exp(-\newdevcon \log
  \pdim/12)$.}
\end{corollary}

Once again we observe optimal dependence on the quantities $\spindex$,
$\log\pdim$, $T$ and $\sigmin^2(\CovMat)$ in our convergence rate. For
the purposes of optimization, a dependence on the strong convexity of
the loss through $1/\psi^2(2 \xbound \radius_1)$ also seems
unavoidable. Indeed, the lower bound of Agarwal et
al.~\cite{AgarwalBaRaWa12} for the complexity of stochastic convex
optimization with strongly convex objectives implies that such a
scaling is necessary for any stochastic first-order method. While the
result in their Theorem 2 is stated in terms of
$\Lossbar(\finalparam{\totiters}) - \Lossbar(\opt)$ and posits a
$1/\lossrsc$ scaling, it can be easily extended to also imply a
$1/\lossrsc^2$ scaling for the error $\norm{\finalparam(\totiters) -
  \opt}_2^2$. Finally, we observe that the bound only holds once the
number of samples $\numobs$ in the
objective~\eqref{eqn:sampleobjective} is large enough. This arises
since the sample objective is not strongly convex by itself, but does
satisfy a restricted version of the LSC condition once the sample size
is large enough. These ideas are further clarified in the proof of the
corollary that we present in Section~\ref{sec:proof-lipschitz-sample}.


\subsection{Optimal rates for least squares regression}

In this section, we specialize to the case of least squares regression
described previously in Example~\ref{example:leastsquares}. For ease
of presentation, we further assume that the linear model
assumption~\eqref{eqn:linearmodel} holds. Since the least-squares cost
function is not Lipschitz over the entire set $\Parset$, we need the
general local setting of our assumptions. For brevity, we introduce
the shorthand notation $\lips_i = \lips(2\radius_i)$ and $\noise_i =
\noise(2\radius_i)$, and note that all of these parameters now depend
on the epoch $i$.

The following theorem characterizes the convergence rate of
Algorithm~\ref{alg:epochdualavg} for least-squares regression, when
applied to independently and identically distributed (i.i.d.)  samples
generated from the linear model~\eqref{eqn:linearmodel} with
$\xbound$-bounded covariates (i.e., $\norm{\samplex}_\infty \leq
\xbound$ with probability one), and additive Gaussian noise with
variance $\gaussnoise^2$. For this example, we (re)define
\begin{align}
  \label{eqn:slackiters-leastsq}
  \slackiters_{\totiters} & \defn \frac{\spindex^2\xbound^4 +
    \lossrsc^2}{\lossrsc^2} (\devcon^2 + \log\pdim) \log_2 \left[
    \frac{\lossrsc^2\radius_1^2
      \totiters}{\spindex^2\gaussnoise^2\xbound^2(\devcon^2 + \log
      \pdim) } \right].
\end{align}
In stating the result, we make use of the shorthand
\begin{align*}
\MESS & \defn \log \frac{\slackiters_{\totiters} \;
  \lossrsc^2}{(\spindex^2 \xbound^4 + \lossrsc^2)(\devcon^2 + \log
  \pdim)}
\end{align*}


\begin{theorem}
\label{thm:leastsquares}
Consider the updates~\eqref{EqnOverallDual} with epoch
lengths~\eqref{eqn:epochlength-simple} and regularization/stepsize
parameters
\begin{align}
\label{eqn:paramsettings-leastsq}
\regpar_i^2 = \frac{\radius_i \lossrsc}{\spindex \sqrt{T_i} }
\sqrt{(\lips_i^2 + \noise_i^2) \log \pdim + \devcon_i^2 \noise_i^2}
\quad \mbox{and} \quad \mystep{t} = 5 \radius_i \sqrt{ \frac{\log
    \pdim}{(\lips_i^2 + \regpar_i^2 +\noise_i^2) t}}.
\end{align}
Then there is a universal constant $\PREFACT$ such that for any
$\totiters \geq 2 \slackiters_{\totiters}$ and for any subset $\Sset$
of $\{1, \dots, \pdim\}$ of cardinality $\spindex$, we have
\begin{align*}
\norm{\finalparam{\totiters} - \opt}_2^2 & \leq
\frac{\PREFACT}{\sigmin^2(\CovMat)} \left[ \frac{\spindex
    \gaussnoise^2 \xbound^2}{\totiters} \Big(\devcon^2 + \log \pdim +
  \MESS \Big) + \totalslacksqsimple \right]
\end{align*}
with probability at least \mbox{$1 - 6 \exp(-\devcon^2/12)$.}
\end{theorem}

Once again, if we focus only on the scaling with iteration number
$\totiters$, the above theorem gives an overall convergence rate of
$\order(1/\totiters)$. The dominant term in the above bound scales as
$\order \left( \frac{\gaussnoise^2 \xbound^2}{\sigmin^2(\CovMat)} \;
\frac{\spindex \log \pdim}{\totiters} \right)$. In a stochastic
optimization setting where each stochastic gradient is based on
drawing one fresh sample from the underlying distribution, the number
of iterations $\totiters$ also corresponds to the number of samples
seen. In such a scenario, the above iteration complexity bound is
unimprovable in general due to matching sample complexity lower bounds
for the sparse linear regression problem~\cite{RasWaiYu11}. This
optimality is further clarified in the corollaries that we present
below for the exact and approximately sparse cases. The corollaries
are analogous to our earlier result in
Corollaries~\ref{cor:lipschitz-pop-exact}
and~\ref{cor:lipschitz-pop-weak} for the case of Lipschitz losses.

\begin{corollary}
  Under the conditions of Theorem~\ref{thm:leastsquares}, we have the
  following guarantees.
  \begin{itemize}
  \item[(a)] Exact sparsity: Suppose that $\opt$ is supported on a
    subset of cardinality $\spindex$.  Then there is a universal
    constant $\PREFACT$ such that for all $\totiters \geq
    2\slackiters_{\totiters}$, we have
 \begin{align}
 \norm{\finalparam{\totiters} - \opt}_2^2 & \leq \PREFACT \;
 \frac{\spindex \log \pdim}{\totiters} \; \frac{\gaussnoise^2
   \xbound^2}{ \sigmin^2(\CovMat)}(1 + \newdevcon) + \PREFACT \;
 \frac{\spindex \gaussnoise^2 \xbound^2}{\totiters \sigmin^2(\CovMat)}
 \; \MESS
 \end{align}
 with probability at least \mbox{$1 - 6 \exp(-\newdevcon \log
   \pdim/12)$.}
\item[(b)] Weak sparsity: Suppose $\opt \in \Ball_\qpar(\radq)$ for
  some $\qpar \in (0,1]$.  Then there is a universal constant
    $\PREFACT$ such that for all $\totiters \geq 2
    \slackiters_{\totiters}$, we have
\begin{align}
\norm{\finalparam{\totiters} - \opt}_2^2 & \leq \PREFACT \; \radq
\left\{ \frac{\gaussnoise^2 \xbound^2}{\sigmin^2(\CovMat)} \frac{\log
  \pdim( 1 + \newdevcon) + \MESS}{\totiters} \right \}^{1 - \qpar/2}
\end{align}
 with probability at least \mbox{$1 - 6 \exp(-\newdevcon \log
   \pdim/12)$.}
  \end{itemize}
  \label{cor:leastsquares-pop}
\end{corollary}

Part (a) of the corollary follows from observing that
$\totalslacksqsimple = 0$ in the result of
Theorem~\ref{thm:leastsquares}, under our assumptions here. Part (b)
involves setting $\Sset$ based on the assumption $\opt \in
\Ball_{\qpar}(\radq)$, analogous to the proof of
Corollary~\ref{cor:lipschitz-pop-weak}.

A corollary analogous to Corollary~\ref{cor:lipschitz-samp} can also
be obtained from Theorem~\ref{thm:leastsquares}. This involves
replacing the use of the RSC assumption for the sample-averaged
objective as before, and we leave such a development to the reader.


\subsection{A modified method with constant epoch lengths}
\label{sec:adaptive}

Algorithm~\ref{alg:epochdualavg} as described is efficient and simple
to implement. However, the convergence results are based on epoch
lengths $T_i$ set in an appropriate ``doubling'' manner.  In practice,
this setting might be difficult to achieve, since it is not
immediately clear how to set the epoch lengths $T_i$ unless all of the
problem parameters are provided. Juditsky and
Nesterov~\cite{JuditskyNes10} address this issue by proposing an
algorithm that uses fixed epoch lengths, and is also additionally
robust to the knowledge of problem parameters such as the strong
convexity and Lipschitz constant. In this section, we discuss how a
similar approach with fixed epoch lengths also works in our set-up.
At a coarse level, if we have a total budget of $\totiters$
iterations, then this version of our algorithm allows us to set the
epoch lengths to $\order(\log \totiters)$, and guarantees convergence
rates that are $\order((\log \totiters)/\totiters)$, so at most a log
factor worse than our earlier results. We note that unlike past work,
our objective function changes at each epoch, which leads to certain
new technical difficulties.

For ease of presentation in stating a fixed-epoch length result, we
assume $\slopfac = 0$ and $\realrsc = \lossrsc$ throughout this
section. We further restrict ourselves to the setting of
Theorem~\ref{thm:lipschitz} with $\lips_i \equiv \lips$ and $\noise_i
\equiv \noise$, with the extension to least-squares case analogous to
that for obtain Theorem~\ref{thm:leastsquares}.

\begin{theorem}
 \label{thm:lipschitz-fixedepoch}
Suppose the expected loss satisfies
Assumptions~\ref{ass:explips}-~\ref{ass:subgauss} with parameters
$\lips_i \equiv \lips$, $\lossrsc$ and $\noise_i \equiv \noise$
respectively.  Recalling the setting of
$\slackiters_{\totiters}$~\ref{eqn:slackiters}, suppose we run
Algorithm~\ref{alg:epochdualavg} for a total of $\totiters$ iterations
with epoch length $T_i \equiv
\totiters\log\pdim/\slackiters_{\totiters}$, and with parameter
settings~\eqref{eqn:paramsettings-simple}. Assuming that the above
setting ensures that $T_i = \order(\log\pdim)$, for any subset $\Sset
\subseteq \{1, \dots, \pdim\}$ of cardinality at most $\spindex$,
  \begin{equation*}
    \norm{\finalparam{\totiters} - \opt}_2^2 =
    \order\left(\spindex\,\frac{(\lips^2 + \noise^2)\log\pdim +
      (\devcon^2 + \log (\slackiters_{\totiters}/\log \pdim))
      \noise^2}{\totiters} \,
    \frac{\log\pdim}{\slackiters_{\totiters}} \right)  
  \end{equation*}
with probability at least $1 - 3\exp(\devcon^2/12)$.
\end{theorem}

The theorem shows that up to logarithmic factors in $\totiters$, not
setting the epoch lengths optimally is not critical which is an
important practical concern. We note that a similar result can also be
proved for the case of least-squares regression.


\section{Proofs of main results}
\label{sec:mainproofs}

We now turn to the proofs of our main results, which are all based on
a proposition that characterizes the convergence rate of the updates
updates~\eqref{EqnOverallDual} within each epoch.  Proving this
intermediate result requires combining the standard analysis of the
dual averaging algorithm with the statistical properties of the
minimizer of the epoch objective $f_i$ at each epoch. We then build on
this basic convergence result using an iterative argument in order to
establish our main Theorems~\ref{thm:lipschitz}
and~\ref{thm:leastsquares}.

\subsection{Set-up and a general result}
\label{SecSetup}

In our proofs, we use a weaker form of the local strong convexity
(LSC) condition, known as locally restricted strong convexity, or
local RSC.  This weakened condition allows us to adapt our proofs to
finite pool optimization (Corollary~\ref{cor:lipschitz-samp}) in a
seamless way, and also to establish slightly more general versions of
our main results:\\

\noindent {\bf{Assumption \BHACKLRSC $\;$}} (Locally restricted
strong convexity)
The function $\Lossbar: \Parset \rightarrow \real$ satisfies a
$\radius$-local form of restricted strong convexity (RSC) if there are
non-negative constants $(\lossrsc, \slopfac)$ such that
\begin{align}
\label{EqnLRSC}
 \Lossbar(\paramother) & \geq \Lossbar(\param) + \inprod{\nabla
   \Lossbar(\param)}{\paramother - \param} + \frac{\lossrsc}{2}
 \norm{\param - \paramother}_2^2 - \slopfac\norm{\param -
   \paramother}_1^2.
\end{align}
for any $\param, \paramother \in \Parset$ with $\norm{\param}_1 \leq
\radius$ and $\norm{\paramother}_1 \leq \radius$. \\

Note that this condition reduces to the standard form of local strong
convexity in Assumption~\ref{ass:rsc} when $\slopfac = 0$.  The key
weakening here is the presence of the additional \emph{tolerance
  term}---namely, the quantity $- \slopfac \|\param -
\paramother\|_1^2$.  Due to this term, the lower bound~\eqref{EqnLRSC}
provides a nontrivial constraint only for pairs of vectors $\param,
\paramother$ such that $\frac{\lossrsc}{2} \norm{\param -
  \paramother}_2^2 \gg \slopfac \norm{\param - \paramother}_1^2$.
Since the ratio of the $\ell_1$ and $\ell_2$ norms is a measure of
sparsity, the local RSC condition thus enforces local strong convexity
only in directions that are relatively sparse.  As a concrete example,
if the difference $\param - \paramother$ is $\spindex$-sparse, then we
have $\|\param - \paramother\|_1^2 \leq \spindex \,\|\param -
\paramother\|_2^2$, so that the condition~\eqref{EqnLRSC} guarantees
that
\begin{align}
\label{EqnMUFCSUX}
 \Lossbar(\paramother) - \Lossbar(\param) + \inprod{\nabla
   \Lossbar(\param)}{\paramother - \param} \big \} & \geq \frac{1}{2}
 \, \big \{\lossrsc - 2 \spindex \slopfac \big \} \norm{\param -
   \paramother}_2^2,
\end{align}
a non-trivial statement whenever $\lossrsc > 2 \spindex \slopfac$.

With this intuition, for applications of the condition~\eqref{EqnLRSC}
with $\slopfac \neq 0$, we introduce the \emph{effective RSC constant}
\begin{equation}
  \label{eqn:realrsc}
  \realrsc = \lossrsc - 16\spindex \slopfac,
\end{equation}
where we have introduced the factor of $16$ for later theoretical
convenience.  In addition, we use a slightly generalized definition of
the approximation-error term $\totalslacksqsimple$, namely
\begin{align}
\label{eqn:totalslackdefn}
\totalslacksq & \defn \frac{\coneslack^2}{\spindex} \left(1 +
\frac{\spindex \slopfac}{\realrsc} \right),
\end{align}
which reduces to $\totalslacksqsimple$ when $\slopfac = 0$.  So as to
simplify notation, we use $f_i(\param) \defn \Lossbar(\param) +
\regpar_i\norm{\param}_1$ to denote the objective at epoch
$i$. Following standard notation in the optimization literature, we
also require a quantity $\proxbound$ such that \mbox{$\proxbound \geq
  \prox(\param)$} for all $\|\param\|_1 \leq 1$; in our case, the
choice of prox-function~\eqref{EqnDefnProx} ensures that $\proxbound =
e \log \pdim$ suffices. We also recall our notation $\devcon_i^2 =
\devcon^2 + 24\log i$.

We now state and prove a slightly generalized form of
Theorem~\ref{thm:lipschitz} that allows for $\slopfac > 0$.  It is
based on the epoch lengths
\begin{align}
\label{eqn:epochlength}
T_i & \defn \const_1 \, \left[\frac{\spindex^2 \lossrsc^2}{ \realrsc^4
    \radius_i^2} \left( \proxbound (\lips^2 + \noise^2) + \devcon_i^2
  \noise^2 \right) + \frac{ \lossrsc \proxbound}{\realrsc} \right],
\end{align}
where $\const_1$ is a universal constant.  The more general form of
Theorem~\ref{thm:lipschitz} also involves the quantity
\begin{align}
\label{eqn:slackiters-gen}
\slackiters_{\totiters} & \defn \log_2 \left [\frac{\realrsc^4
    \radius_1^2 \totiters }{\lossrsc^2 \spindex^2 ( (\lips^2 +
    \noise^2) \proxbound + \devcon^2 \noise^2)} \right] \frac{\lossrsc
  \log \pdim}{\realrsc}.
\end{align}
It applies to the dual averaging updates~\eqref{EqnOverallDual} with
the epoch lengths~\eqref{eqn:epochlength} and regularization/stepsize
parameters
\begin{align}
\label{eqn:paramsettings-lipschitz}
\regpar_i^2 = \frac{\radius_i \realrsc}{\spindex \sqrt{T_i}}
\sqrt{\proxbound (\lips^2 + \noise^2) + \devcon_i^2 \noise^2}, \quad
\mbox{and} \quad \mystep{t} = 5 \radius_i \sqrt{ \frac{\proxbound}{
    (\lips^2 + \regpar_i^2 + \noise^2)t}}.
\end{align}

\begin{theorem}
\label{thm:lipschitz-general}
Suppose the expected loss $\Lossbar$ satisfies
Assumptions~\ref{ass:explips}, \HACKLRSC$\,$ and~\ref{ass:subgauss}
with parameters \mbox{$\lips(\radius) \equiv \lips$}, $(\lossrsc,
\slopfac)$, and \mbox{$\noise(\radius) \equiv \noise$} respectively,
and that we run Algorithm~\ref{alg:epochdualavg} with parameter
settings~\eqref{eqn:paramsettings-lipschitz} and epoch
lengths~\eqref{eqn:epochlength}. Then there is universal constant
$\PREFACT$ such that for any $\totiters \geq
2\slackiters_{\totiters}$, for any integer $\spindex \in [1, \pdim]$
such that $\lossrsc - 16 \spindex \slopfac > 0$, and for any subset
$\Sset \subset \{1, \ldots, \pdim\}$ of cardinality $\spindex$, we
have
\begin{align}
 \norm{\finalparam{\totiters} - \opt}_2^2 & \leq \PREFACT \, \left[
   \frac{\spindex \lossrsc^2}{\realrsc^4 \totiters} \left( \lips^2
   \log \pdim + \noise^2 \left(\log\pdim + \devcon^2 + \log
   \slackiters_{\totiters} \right) \right) + \frac{\lossrsc}{\realrsc}
   \totalslacksq \right].
\end{align}
with probability at least \mbox{$1 - 6\exp(-\devcon^2/12)$.}
\end{theorem}
\noindent In order to prove this theorem, we require some intermediate
results on the convergence rates within each epoch.  We state these
results here, deferring their proofs to the appendices, before
returning to prove Theorem~\ref{thm:lipschitz} and its corollaries.


\subsection{Convergence within a single epoch}

This intermediate result applies to iterates generated using the dual
averaging updates~\eqref{EqnOverallDual} for $T_i$ rounds with
parameters~\eqref{eqn:paramsettings-lipschitz}, where $\lips =
\lips_i$, and the error bound is stated in terms of the averaged
parameter at the $i_{th}$ epoch, namely the vector $\xavg(T_i) \defn
\frac{1}{T_i} \sum_{t=1}^{T_i} \iter{t}$.
\begin{proposition}
\label{prop:mainepoch}
Suppose $\Lossbar$ satisfies Assumptions~\ref{ass:explips},
\HACKLRSC$\,$ and~\ref{ass:subgauss} with parameters $\lips_i$,
$(\lossrsc, \slopfac)$ and $\noise_i$ respectively, and assume that
$\norm{\opt - \proxcenter{i}}_p \leq \radius_i$. Suppose we apply the
updates~\eqref{EqnOverallDual} with stepsizes based on
equation~\eqref{eqn:paramsettings-lipschitz}. Then there exists a
universal constant $\plaincon > 0$ such that for any radius
\mbox{$\radius_i^2 \geq 4 \lossrsc \totalslacksq/\realrsc$,} any
integer $\spindex \in [1, \pdim]$ such that $\lossrsc - 16 \spindex
\slopfac > 0$, and any subset $\Sset \subseteq \{1, \dots, \pdim \}$
of cardinality at most $\spindex$, we have
\begin{subequations}
\begin{align}
\label{eqn:fvaluebound}
f_i(\xavg(T_i)) - f_i(\epochopt{i}) & \leq 30
\radius_i\sqrt{\frac{2\proxbound(\lips_i^2 + \noise_i^2)}{T_i}} + 
\, \frac{\devcon_i\noise_i\radius_i}{\sqrt{T_i}} + 30
  \radius_i\regpar_i \sqrt{\frac{\proxbound}{T_i}} \quad \mbox{and}
  \\%
\label{eqn:ell1bound}
\norm{\xavg(T_i) - \opt}_1^2 & \leq \plaincon \,
\frac{\lossrsc}{\realrsc} \left[ \frac{\spindex \radius_i}{\realrsc
    \sqrt{T_i}} \left(\sqrt{\proxbound(\lips_i^2 + \noise_i^2) } +
  \devcon_i \noise_i \right) + \frac{\radius_i^2 \proxbound}{T_i} +
  \totalslacksq \right],
\end{align}
\end{subequations}
where both bounds are valid with probability at least \mbox{$1 - 3
  \exp(-\devcon_i^2/12)$} for any $\devcon_i \leq 9 \sqrt{\log T_i}$.
\end{proposition}

On one hand, inequality~\eqref{eqn:fvaluebound} is a relatively direct
consequence of known convergence results about stochastic dual
averaging~\cite{Nesterov09}. On the other hand, the
bound~\eqref{eqn:ell1bound}---which plays a central role in the our
proofs---requires some additional statistical properties of the
optimal solution at each epoch $i$.  See
Appendix~\ref{app:within-epoch} for further details on these
properties, and the proof of Proposition~\ref{prop:mainepoch}.

Before moving on, we note that the bounds in
Proposition~\ref{prop:mainepoch} can be simplified further based on
the parameter settings in equations~\eqref{eqn:epochlength}
and~\eqref{eqn:paramsettings-lipschitz}. Substituting these choices in
our bounds yields the inequalities
\begin{subequations}
\begin{align}
\label{eqn:fvaluebound-simple}
f_i(\xavg(T_i)) - f_i(\epochopt{i}) & \leq \const\, \frac{\radius_i^2
  \realrsc^2}{\spindex \lossrsc}\quad \mbox{and} \\%
\label{eqn:ell1bound-simple}
\norm{\xavg(T_i) - \opt}_1^2 & \leq \plaincon \,
\left[\radius_i^2 + \frac{\lossrsc}{\realrsc}\,\spindex\totalslacksq\right]. 
\end{align}
\end{subequations}

 In addition to this proposition, we need to state two more technical
 lemmas, the first of which bounds the error $\DelIt{i} \defn
 \epochopt{i} - \opt$.
\begin{lemma}
\label{lemma:opterror}
At epoch $i$, assume that $\norm{\opt - \proxcenter{i}}_p \leq
\radius_i$. Then the error $\DelIt{i} = \epochopt{i} - \opt$ satisfies
the bounds
\begin{subequations}
\begin{align}
\label{EqnOptEllTwo}
\norm{\epochopt{i} - \opt}_2 & \leq \frac{4}{\realrsc} \sqrt{\spindex}
\regpar_i + 2 \sqrt{\frac{\regpar_i\coneslack + 4 \slopfac
    \coneslack^2}{\realrsc}}, \quad \mbox{and} \\
\label{EqnOptEllOne}
\norm{\epochopt{i} - \opt}_1 & \leq \frac{8}{\realrsc} \spindex
\regpar_i + 4 \sqrt{\frac{\spindex (\regpar_i \coneslack + 4 \slopfac
    \coneslack^2)}{\realrsc}} + 2 \coneslack.
\end{align}
\end{subequations}
\end{lemma}
\noindent For future reference, it is convenient to note that the
bound~\eqref{EqnOptEllOne} implies that
\begin{equation}
\label{eqn:epocherrorbound}
\norm{\opt - \epochopt{i}}_1 \leq \frac{9}{\realrsc}\spindex \regpar_i
+ 8 \coneslack \sqrt{\frac{\slopfac \spindex}{\realrsc}} + 6
\coneslack,
\end{equation}
where we have made use of the elementary inequalities $\sqrt{a+b} \leq
\sqrt{a} + \sqrt{b}$ and $2 \sqrt{ab} \leq a + b$, valid for all
non-negative $a,b$. \\

Our next lemma provides a simplified version of the RSC condition that
holds under conditions of Proposition~\ref{prop:mainepoch}. The lemma
is stated in terms of the error
\begin{align}
\label{EqnDefnDelBarAvg}
\DelBarAvg(T_i) & \defn \xavg(T_i) - \epochopt{i}
\end{align}
between the average \mbox{$\xavg(T_i) \defn \frac{1}{T_i}
  \sum_{t=1}^{T_i} \iter{t}$} over trials $1$ through $T_i$ in epoch
$i$, and the epoch \mbox{optimum $\epochopt{i}$.}

\begin{lemma}
  \label{lemma:rsc-simple}
  Under conditions of Proposition~\ref{prop:mainepoch} and with
  parameter settings~\eqref{eqn:epochlength}
  and~\eqref{eqn:paramsettings-lipschitz}, we have
  \begin{align*}
    \frac{\realrsc}{2} \norm{\DelBarAvg(T_i)}_2^2 &\leq
    f_i(\xavg(T_i)) - f_i(\epochopt{i}) + \const\, \slopfac \left(
    \frac{\realrsc}{\lossrsc} \radius_i^2 + \spindex \totalslacksq
    \right)
  \end{align*}
with probability at least $1 - 3\exp(-\devcon_i^2/12)$.
\end{lemma}

\comment{
In particular, using the setting of $\regpar_i$ from
Proposition~\ref{prop:mainepoch}, the above bound simplifies to
\begin{align}
\label{eqn:cone-to-epochopt}
\norm{\DelBarAvg(T)}_1^2 & \leq 8 \spindex \norm{\DelBarAvg(T)}_2^2 +
c\,\left( \frac{\spindex \radius_i}{\realrsc \sqrt{T}}\sqrt{ \left(
  \proxbound(\lips_i^2 + \noise_i^2) + \devcon^2 \noise_i^2\right)} +
\frac{\radius_i^2 \proxbound}{T} + \totalslacksq \right).
\end{align}
}


\subsection{Proof of Theorem~\ref{thm:lipschitz-general}}
\label{sec:proof-lipschitz}

\newcommand{\UKRAINE}{\ensuremath{ \frac{4 \lossrsc}{\realrsc} \spindex} \;}
\newcommand{\HACKTOTEPOCH}{\TOTEPOCH_\totiters}
\newcommand{\Plaincon}{\ensuremath{C}}

We are now equipped to prove Theorem~\ref{thm:lipschitz-general}. The
proof will be broken down into cases, corresponding to whether
$\totiters$ is ``too large'' or not. We recall that $\TOTEPOCH$ is the
total number of epochs performed after $\totiters$ steps in
Algorithm~\ref{alg:epochdualavg}.

We first consider iterations for which the bound
\begin{align}
\label{EqnCaseBound}
\radius^2_{\TOTEPOCH} & \geq \UKRAINE \totalslacksq
\end{align}
is satisfied.  We then provide an additional lemma which allow us to
control the iterates for epochs $i$ after which the squared
$\radius_i^2$ violates the bound~\eqref{EqnCaseBound}.

\subsubsection{Proof assuming inequality~\eqref{EqnCaseBound} holds}

Our first step is to ensure that the bound \mbox{$\norm{\opt -
    \proxcenter{i}}_\pval \leq \radius_i$} holds at each epoch $i$, so
that Proposition~\ref{prop:mainepoch} can be applied in a recursive
manner. We prove this intermediate claim by induction on the epoch
index.  By construction, this bound holds at the first epoch. Assume
that it holds for epoch $i$. Recall that the epoch length setting in
Theorem~\ref{thm:lipschitz-general} is of the form
\begin{equation*}
  T_i = \Plaincon \, \left [ \frac{\spindex^2 \lossrsc^2}{\realrsc^4
      \radius_i^2} \left( \proxbound (\lips^2 + \noise^2) +
    \devcon_i^2 \noise^2 \right) + \frac{ \lossrsc
      \proxbound}{\realrsc} \right],
\end{equation*}
where $\Plaincon \geq 1$ is a constant that we are free to choose.
Upon substituting this setting of $T_i$ in the
inequality~\eqref{eqn:ell1bound}, the simplified
bound~\eqref{eqn:ell1bound-simple} further yields
\begin{equation*}
  \norm{\xavg(T_i) - \opt}^2_1 \leq \frac{\const}{\sqrt{\Plaincon}} \,
  \left (\radius_i^2 + \frac{\lossrsc \spindex}{\realrsc}
  \totalslacksq \right) \stackrel{(i)}{\leq}
  \frac{\const'}{\sqrt{\Plaincon}} \; \radius_i^2,
\end{equation*}
where step (i) follows due to our assumption $\radius_i^2 \geq
4\lossrsc \spindex \totalslacksq /\realrsc$.  Thus, by choosing
$\Plaincon$ sufficiently large, we may ensure that $\norm{\xavg(T_i) -
  \opt}^2_1 \leq \radius_i^2/2 \defn \radius_{i+1}^2$.  Consequently,
if $\opt$ is feasible at epoch $i$, it stays feasible at epoch $i+1$,
and so by induction, we are guaranteed the feasibility of $\opt$
throughout the run of the algorithm by induction.

As a result, Lemma~\ref{lemma:rsc-simple} applies, and we 
find that
\begin{align*}
  \frac{\realrsc}{2} \norm{\DelBarAvg(T_i)}_2^2 &\leq f_i(\xavg(T_i))
  - f_i(\epochopt{i}) + \const\, \slopfac \left(
  \frac{\realrsc}{\lossrsc} \radius_i^2 + \spindex \totalslacksq
  \right)
\end{align*}
with probability at least $1 - 3\exp(-\devcon^2_i/12)$. Appealing to
the simplified form~\eqref{eqn:fvaluebound-simple} of
Proposition~\ref{prop:mainepoch}, we can further obtain the inequality
\begin{align*}
  \frac{\realrsc}{2} \norm{\DelBarAvg(T_i)}_2^2 &\leq
  \const\,\frac{\radius_i^2 \realrsc^2}{\spindex \lossrsc} + \const\,
  \slopfac \left( \frac{\realrsc}{\lossrsc} \radius_i^2 + \spindex
  \totalslacksq \right).
\end{align*}
Recalling that $\lossrsc = \realrsc + 16\spindex \slopfac$, the above
bound further simplifies to
\begin{align}
  \frac{\realrsc}{2} \norm{\DelBarAvg(T_i)}_2^2 &\leq \const\, \left(
  \frac{\radius_i^2 \realrsc}{\spindex} + \spindex \slopfac
  \totalslacksq \right).
  \label{EqnLaundry}
\end{align}

We have now bounded $\DelBarAvg(T_i) = \xavg(T_i) - \epochopt{i}$, and
Lemma~\ref{lemma:opterror} provides a bound on $\DelIt{i} =
\epochopt{i} - \opt$, so that the error $\DelAvg(T_i) = \xavg(T_i) -
\opt$ can be controlled using triangle inequality. In particular,
combining inequality~\eqref{EqnOptEllTwo} with the
bound~\eqref{EqnLaundry}, we find that
\begin{align*}
\norm{\DelAvg(T_i)}_2^2 & \leq \const \left [ \frac{\radius_i^2}
  {\spindex} + \frac{ \slopfac \spindex \totalslacksq}{\realrsc} +
  \frac{ \spindex \regpar_i^2}{\realrsc^2} + \frac{ \regpar_i
    \norm{\opt_{\SsetComp}}_1 + \slopfac
    \norm{\opt_{\SsetComp}}_1^2}{\realrsc} \right].
\end{align*}
Since $2\regpar_i\norm{\opt_{\SsetComp}}_1 \leq \frac{\spindex
  \regpar_i^2}{\realrsc} + \frac{\realrsc
  \norm{\opt_{\SsetComp}}_1}{\spindex}$ by Cauchy-Schwartz inequality,
we can further simplify the above bound to obtain
\begin{align*}
\norm{\DelAvg(T_i)}_2^2 & \leq \const \left [ \frac{\radius_i^2}
  {\spindex} + \frac{ \slopfac \spindex \totalslacksq}{\realrsc} +
  \frac{ \spindex \regpar_i^2}{\realrsc^2} +
  \frac{\norm{\opt_{\SsetComp}}_1^2}{\spindex} \left(1 +
  \frac{\slopfac \spindex} {\realrsc}\right) \right].
\end{align*}
Substituting our choice of $\regpar_i$ and $T_i$ from
equations~\eqref{eqn:paramsettings-lipschitz}
and~\eqref{eqn:epochlength} respectively yields the final bound
\begin{equation*}
\norm{\DelAvg(T_i)}_2^2 \leq \const\left[\frac{\radius_i^2}{\spindex}
  + \totalslacksq \left(1 + \frac{\spindex\slopfac}{\realrsc} \right)
  \right],
\end{equation*}
a bound that holds probability at least \mbox{$1 - 3
  \exp(-\devcon_i^2/12)$.} Recalling that $\radius_i^2 = \radius_1^2
2^{-(i-1)}$, we see that the error after $i$ epochs is at most
\begin{equation*}
  \norm{\DelAvg(T_i)}_2^2 \leq \const\left[\frac{\radius_1^2
      2^{-(i-1)}}{\spindex} + \totalslacksq \left(1 +
    \frac{\spindex\slopfac}{\realrsc}\right) \right].
\end{equation*} 
Since $\lossrsc = \realrsc + 16\spindex \slopfac$, some algebra then
leads to
\begin{equation}
  \norm{\DelAvg(T_i)}_2^2 \leq \const\left[\frac{\radius_1^2
      2^{-(i-1)}}{\spindex} + \totalslacksq
    \frac{\lossrsc}{\realrsc} \right],
  \label{eqn:almostthere}
\end{equation} 
with probability at least $1 -
3\sum_{j=1}^i\exp(-\devcon_j^2/12)$. Recalling our setting
$\devcon_i^2 = \devcon^2 + 24\log i$, we can apply the union bound and
simplify the error probability as
\begin{equation*}
  \sum_{j=1}^i \exp(-\devcon_j^2/12) = \sum_{j=1}^i \exp(-(\devcon^2 +
  24\log j)/12) = \exp(-\devcon^2/12)\sum_{j=1}^i \frac{1}{j^2}.
\end{equation*}
As a result, we can upper bound the net probability of our bounds
failing after the $\TOTEPOCH$ epochs performed by
Algorithm~\ref{alg:epochdualavg} as 
\begin{equation*}
  \sum_{j=1}^{\TOTEPOCH} 3\exp(-\devcon_j^2/12) \leq
  \sum_{j=1}^{\infty} 3\exp(-\devcon_j^2/12) \leq 3\exp(-\devcon^2/12)
  \sum_{j=1}^{\infty} \frac{1}{j^2} \leq \frac{3\pi^2}{6}
  \exp(-\devcon^2/12),
\end{equation*}
where the last step follows summing the infinite series. Finally
noting that $\pi^2/6 \leq 2$ gives us the stated probability of
$1-6\exp(-\devcon^2/12)$ with which our bounds hold. In order to
complete the proof of the theorem, we need to convert the error
bound~\eqref{eqn:almostthere} from its dependence on the number of
epochs $\TOTEPOCH$ to the number of iterations $\totiters$. This
requires us to obtain $\TOTEPOCH$ in terms of $\totiters$, which we do
next. Letting $T(K)$ be the number of iterations needed to complete
$K$ epochs, we start by computing an upper bound $g(K)$ on $T(K)$
based on our epoch length setting~\eqref{eqn:epochlength}. Then
inverting the bound allows us to deduce the lower bound $\TOTEPOCH
\geq g^{-1}(\totiters)$, which allows us to obtain error bounds in
terms of $\totiters$.
\begin{align*}
T(K) &= \sum_{i=1}^{K} T_i = \sum_{i=1}^{K} c \,
\left[\frac{\spindex^2 \lossrsc^2}{\realrsc^4\radius_i^2}
  \left(\proxbound (\lips^2 + \noise^2) + (\devcon^2 + 24 \log i)
  \noise^2 \right) + \frac{\lossrsc \proxbound}{\realrsc} \right] \\
& = c \, \left[\frac{\spindex^2 \lossrsc^2}{\realrsc^4 \radius_1^2}
  \left (\proxbound ( \lips^2 + \noise^2) + (\devcon^2 + 24\log K)
  \noise^2\right) \sum_{i=1}^{K} 2^{i-1} + \frac{\lossrsc
    \proxbound}{\realrsc} K \right] \\
& \leq \underbrace{c \, \left[\frac{\spindex^2 \lossrsc^2}{ \realrsc^4
      \radius_1^2} \left( \proxbound(\lips^2 + \noise^2) + (\devcon^2
    + 24 \log K) \noise^2 \right) 2^{K} + \frac{\lossrsc
      \proxbound}{\realrsc} K \right]}_{\defn g(K)}
\end{align*}
where the last inequality sums the geometric progression. Inverting
the above inequality to obtain $g^{-1}(\totiters)$, along with some
straightforward algebra completes the proof.


\subsubsection{Case 2:  Extension to arbitrary $\TOTEPOCH$}

As we observed in the previous section, when the
bound~\eqref{EqnCaseBound} holds, we can ensure that $\opt$ stays
feasible at each epoch, thereby allowing us to use the error bounds
from Proposition~\ref{prop:mainepoch}.  However, once $\totiters$
becomes large enough, the bound~\eqref{EqnCaseBound} will no longer
hold, so that the the feasibility of $\opt$ for subsequent epochs can
no longer be guaranteed.  In this section, we deal with this remaining
set of iterations.  In particular, let us define the critical epoch
number
\begin{align}
\label{eqn:Kstar}
\Kstar & \defn \arg\max \big \{ k \geq 1~ \mid ~\radius_k^2 \geq
\UKRAINE \totalslacksq \big\},
\end{align}
beyond which the bound~\eqref{EqnCaseBound} no longer holds. By the
definition of $\Kstar$, we are guaranteed that
\begin{align*}
\radius_{\Kstar}^2 \; \geq \; \UKRAINE \totalslacksq \;
\stackrel{(i)}{\geq} \radius^2_{\Kstar+1} \; \stackrel{(ii)}{=}
\radius_{\Kstar}^2/2, 
\end{align*}
where inequality (i) follows since $\Kstar$ is the largest epoch for
which the bound~\eqref{EqnCaseBound} holds, and step (ii) follows
from our setting
$\radius_{i+1}^2 = \radius_i^2/2$.

Now our earlier argument applies to all epochs $k \leq \Kstar$, and it
guarantees that after $\Kstar$ epochs, we have
\begin{subequations}
  \begin{align}
    \label{eqn:Kstarboundell2}
    \norm{\proxcenter{\Kstar} - \opt}_2^2 &\leq \const\,\left(
    \frac{\radius_{\Kstar}^2}{\spindex} + \totalslacksq
    \frac{\lossrsc}{\realrsc} \right) \leq \const\,
    \totalslacksq \frac{\lossrsc}{\realrsc},
    \\ \norm{\proxcenter{\Kstar} - \opt}_1^2 &\leq \const\,\left(
    \radius_{\Kstar}^2 + \frac{\lossrsc \spindex} {\realrsc}
    \totalslacksq \right) \leq \const\,\frac{\lossrsc \spindex}
                  {\realrsc}.
    \label{eqn:Kstarboundell1}
  \end{align}
  \label{eqn:Kstarbounds}
\end{subequations}
with probability at least $1 - 3 \, \big( \sum \limits
_{i=1}^{\Kstar}\frac{1}{i^2} \big) \: \exp(-\devcon^2/12)$.

Our approach for the remaining epochs $k > \Kstar$ is to show that
even though $\opt$ may no longer be feasible, the error of the
algorithm cannot get significantly worse than that at epoch
$\Kstar$. In order to do so, we need an additional lemma.
\begin{lemma}
\label{lemma:badepochell2}
Suppose that Assumptions~\ref{ass:explips}, \HACKLRSC$\,$
and~\ref{ass:subgauss} are satisfied with parameters $\lips_i$,
$(\lossrsc, \slopfac)$ and $\noise_i$ respectively at epochs $i = 0,
1, 2, \ldots$. Assume that at some epoch $k$, the prox center
$\proxcenter{k}$ satisfies the bound \mbox{$\norm{\proxcenter{k} -
    \opt}_2 \leq \plaincon_1 \, \radius_k / \sqrt{\spindex}$}, and
that for all epochs $j \geq k$, the epoch lengths satisfy the bounds
\begin{equation*}
  \frac{\lossrsc \spindex}{\realrsc^2} \sqrt{\frac{
      \proxbound(\lips_j^2 + \noise_j^2) + \devcon_j^2\noise_j^2}
    {T_j}} \leq \frac{\radius_k}{2}, \quad \mbox{and} \quad
  \frac{\proxbound}{T_j} \leq \const_2.
\end{equation*}
Then for all epochs $j \geq k$, we have the error bound
$\norm{\proxcenter{j} - \opt}_2^2 \leq \const_3
\,\frac{\radius_{k}^2}{\spindex}$ with probability at least $1 - 3
\sum_{i=k+1}^j \exp(-\devcon_i^2/12)$.
\end{lemma}

\vspace*{.1in}

\noindent See Appendix~\ref{AppLemBad} for the proof of this lemma. \\

Equipped with this lemma, the remainder of the proof is
straightforward. Specifically, inequality~\eqref{eqn:Kstarboundell2}
ensures that for all epochs $j \geq \Kstar$, we have
\begin{equation*}
\norm{\proxcenter{\Kstar} - \opt}_2^2 \leq \const \,\left(
\frac{\radius_{\Kstar}^2}{\spindex} + \totalslacksq \frac{
  \lossrsc}{\realrsc} \right) \; \stackrel{(i)}{\leq} \; \const
\,\frac{\radius_{\Kstar}^2}{\spindex},
\end{equation*}
with probability at least $1 - 3 \,
\big(\sum_{i=1}^{\Kstar}\frac{1}{i^2} \big) \; \exp(-\devcon^2/12)$.
Here step (i) follows from our definition of $\Kstar$.  Now we apply
Lemma~\ref{lemma:badepochell2} to conclude that if
$\TOTEPOCH \geq \Kstar$, then
\begin{equation*}
\norm{\proxcenter{\Kstar} - \opt}_2^2 \leq \const\,\left(
\frac{\radius_{\Kstar}^2}{\spindex}\right) \leq
\const\,\frac{\lossrsc}{\realrsc} \totalslacksq
\end{equation*}
with probability at least $1 - 3 \, \big
(\sum_{i=1}^{\TOTEPOCH}\frac{1}{i^2} \big)
\exp(-\devcon^2/12)$.  Finally, observing that the overall probability
of our bounds failing is at most $6\exp(-\devcon^2/12)$ as before, we
see that the statement of Theorem~\ref{thm:lipschitz-general} holds in
this case as well, thereby completing the proof.


\subsection{Proofs of Corollaries~\ref{cor:lipschitz-pop-weak} and~\ref{cor:lipschitz-samp}}

In this section, we will establish the corollaries of
Theorem~\ref{thm:lipschitz}. We start with
Corollary~\ref{cor:lipschitz-pop-weak} before moving on to
Corollary~\ref{cor:lipschitz-samp}, the latter needing our more
general statement of Theorem~\ref{thm:lipschitz-general}.

\subsubsection{Proof of Corollary~\ref{cor:lipschitz-pop-weak}}

The corollary follows from Theorem~\ref{thm:lipschitz} by making a
particular choice of $\Sset$ based on our assumption $\opt \in
\Ball_\qpar(\radq)$.  Specifically, given a parameter $\ellqtol > 0$,
define
\begin{equation*}
|\Sset_{\ellqtol}| = \quad \mbox{and} \quad \Sset_{\ellqtol} \defn \{j
\in \{1,\ldots,\pdim\} \mid |\opt_j| \geq |\opt_i| ~\mbox{for all} ~ i
\notin \Sset_{\ellqtol}\},
\end{equation*}
to be the set of $\radq\ellqtol^{-\qpar}$ indices corresponding to the
largest coefficients of $\opt$ in absolute value. Given this
definition, some straightforward algebra yields that $|\opt_i| \leq
\ellqtol$ for all $i \notin \Sset_{\ellqtol}$, which further yields
$\norm{\opt_{\Sset^c_{\ellqtol}}}_1 \leq \ellqtol^{1-\qpar} \radq$.
(For instance, see Negahban et al.~\cite{NegRavWaiYu09} for more
detail on these calculations.)  With these choices, the error bound of
Theorem~\ref{thm:lipschitz} simplifies to
\begin{align*}
  \norm{\finalparam{\totiters} - \opt}_2^2 & \leq \const \Big \{
  \frac{\ellqtol^{-\qpar}\radq}{\lossrsc^2\totiters} ((\lips^2 +
  \noise^2)\log\pdim + \devcon^2\noise^2) +
  \frac{\radq^2\ellqtol^{2-2\qpar}}{\radq\ellqtol^{-\qpar}} \Big \}.
\end{align*}
This upper bound is minimized by setting $\ellqtol^* \defn
\sqrt{\frac{(\lips^2 + \noise^2)\log\pdim +
    \devcon^2\noise^2}{\lossrsc^2\totiters}}$; substituting this
choice and performing some algebra yields the claim of the corollary.


\subsubsection{Proof of Corollary~\ref{cor:lipschitz-samp}}
\label{sec:proof-lipschitz-sample}

In order to prove this result, we must first demonstrate that the RSC
condition holds.  For notational simplicity, we introduce the
shorthand
\begin{align*}
\LossN(\param) & \defn \sum_{i=1}^\numobs \Loss(\param; (\samplex_i,
\sampley_i)) \; = \; \sum_{i=1}^\numobs \log \big[1 + \exp(\sampley_i
  \inprod{\param}{\samplex_i} \big) \big].
\end{align*}
Performing a Taylor series expansion of $\paramother$ around $\param$
yields
\begin{equation*}
\LossN(\paramother) - \LossN(\param) - \ip{\nabla
  \LossN(\param)}{\paramother - \param} =
\frac{1}{\numobs}\sum_{i=1}^{\numobs} \psi(a_i)\ip{\samplex_i}{\param
  - \paramother}^2,
\end{equation*}
where $\psi(t) = \frac{\exp(t)}{[1 + \exp(t)]^2}$ is the second
derivative of the logistic function, and $a_i \defn \inprod{\alpha
  \param + (1-\alpha) \paramother}{\samplex_i \sampley_i}$ for some
$\alpha \in [0,1]$.

Under the assumptions of Corollary~\ref{cor:lipschitz-samp}, we
further know that \mbox{$|a_i| \leq 2 \xbound \radius_1$,} and hence
that \mbox{$\psi(a_i) \geq \psi(2 \xbound \radius_1)$.}  Consequently,
in order to establish the local RSC condition~\eqref{EqnLRSC}, it
suffices to lower bound the quantity
\begin{equation*}
  \frac{1}{\numobs} \sum_{i=1}^\numobs \inprod{\samplex_i}{\param -
    \paramother}^2 = \frac{1}{\numobs}\norm{\sampleX(\param -
    \paramother)}_2^2,
\end{equation*}
where $\sampleX \in \real^{\numobs \times \pdim}$ is the design
matrix, with the vector $\samplex_i^T$ as its $i^{th}$ row.
Quantities of this form have been studied in random matrix theory and
sparse statistical recovery. We state a specific result that holds
under our conditions, and provide a proof in
Appendix~\ref{proof:lemma-rsc}. 

\begin{lemma}
Under the conditions of Corollary~\ref{cor:lipschitz-samp}, there are
universal constants $\const, \const_1$ such that we have
  \begin{equation*}
    \frac{\norm{\sampleX v}_2^2} {\numobs} \geq
    \frac{\sigmin(\CovMat)} {2} \norm{v}_2^2 - \const\,
    \frac{\log\pdim} {\numobs} \max\left\{\sigmin(\CovMat),
    \frac{\xnoise^4} {\sigmin(\CovMat)}\right\} \norm{v}_1^2 \qquad
    \mbox{ for all $v \in \R^\pdim$}
  \end{equation*}
with probability at least $1 - 2 \exp( -\const_1 \numobs \min
(\sigmin^2(\CovMat)/\xnoise^4,1) )$,
  \label{lemma:rsc}
\end{lemma}

\noindent Consequently, the local RSC~\eqref{EqnLRSC} condition holds
with
\begin{equation*}
\lossrsc = \frac{\sigmin(\CovMat) \psi(2 \xbound \radius_1)}{2} \quad
\mbox{and} \quad \slopfac = \const\, \frac{\log\pdim} {\numobs} \max
\left \{ \sigmin(\CovMat), \frac{\xnoise^4}
      {\sigmin(\CovMat)}\right\}.
\end{equation*}
Substituting these values in the general statement in
Theorem~\ref{thm:lipschitz-general} completes the proof of the
corollary.


\subsection{Proof of Theorem~\ref{thm:leastsquares}}
\label{sec:proof-leastsquares}

The main difference from the proof of Theorem~\ref{thm:lipschitz} is
that here we obtain improving bounds on the Lipschitz and sub-Gaussian
constants at each epoch. Recalling that $\norm{\iter{t} - \opt}_1 \leq
2\radius_i$ at epoch $i$, a little calculation shows that $\lips_i
\leq 2 \norm{\CovMat}_\infty \radius_i$, where $\norm{\CovMat}_\infty$
is the elementwise $\ell_\infty$ norm of $\CovMat$. Since $\CovMat$ is
positive-semidefinite, we can further conclude that
$\norm{\CovMat}_{\infty} \leq \CovMax$. Assuming further that
$\norm{\samplex}_\infty \leq \xbound$, we see that $\CovMax \leq
\xbound$. We further have the bound
\begin{align*}
  \norm{\stochgrad{t}}_\infty^2 & = \norm{(\ip{\samplex_t}{\iter{t}} -
    \sampley_t)\samplex_t}_\infty^2 \\ 
& \leq 2 \norm{\ip{\samplex_t}{\iter{t} - \opt}\samplex_t}_\infty^2 +
  2\norm{\samplex_t \samplew_t}_\infty^2 \\ 
& \leq 8 \xbound^4 \radius_i^2 + 2 \xbound^2 \samplew_t^2.
\end{align*} 
Since $\samplew_t \sim \normal{\gaussnoise^2}$, it is easy to check
that
\begin{equation*}
\E \Big[\exp(\samplew_t^2/\noise_0^2) \Big] \leq \exp(1), \quad
\mbox{where} \quad \noise_0 = e \, \gaussnoise \sqrt{\frac{2}{e^2 -
    1}} \leq 3 \gaussnoise,
\end{equation*}
so that Assumption~\ref{ass:subgauss} is satisfied with
\mbox{$\noise^2_i = c \left[\xbound^4 \radius_i^2 + \xbound^2
    \gaussnoise^2 \right]$,} for a universal constant $c$. Plugging
these quantities into our earlier bound from
Proposition~\ref{prop:mainepoch} on the epoch length, we observe that
with probability at least $1-3\exp(-\devcon_i^2/12)$ the number of
iterations needed at epoch $i$ is at most
\begin{align*}
  T_i & \leq \plaincon \,\left [ \frac{\spindex^2 \lossrsc^2}
    {\realrsc^4 \radius_i^2} \left(\xbound^4 \radius_i^2 +
    \gaussnoise^2 \xbound^2 \right) (\devcon_i^2 + \log \pdim) +
    \frac{\lossrsc \log \pdim}{\realrsc} \right] \\
& \leq c \, \left[ \frac{\spindex^2 \gaussnoise^2 \xbound^2
      \lossrsc^2}{\realrsc^2 \radius_i^2} (\devcon_i^2 + \log \pdim) +
    \frac{\spindex^2 \lossrsc^2}{\realrsc^4} (\xbound^4 + 1)
    (\devcon^2 + \log \pdim) \right]. 
\end{align*}
We can now mimic our earlier argument to obtain the total number of
iterations across all epochs.


\subsection{Proof of Theorem~\ref{thm:lipschitz-fixedepoch}}

The proof relies on an additional technical lemma in addition to our
earlier development. In order to prove the theorem, we observe that
the key argument in the convergence analysis of
Section~\ref{sec:mainproofs} was the ability to reduce the error to
the optimum $\opt$ by a multiplicative factor after every
epoch. However, with a fixed epoch length $T_0$, it may not be
possible to continue reducing the error once the number of epochs
becomes large enough. This is analogous to the difficulty we
encountered in the proof of Theorem~\ref{thm:lipschitz-general}, and
will again be addressed using Lemma~\ref{lemma:badepochell2}. We start
by deducing the epoch number $\kstar$ such that we successfully halve
the error at each epoch up to $\kstar$. We will then use other
arguments to demonstrate that the error does not increase much for
epochs $k > \kstar$, and this requires some delicate treatment of our
changing objective functions.  Specifically, given a fixed epoch
length $T_0 = \order(\log\pdim)$, we define
\begin{equation}
    \kstar \defn \sup\left\{i~:~ 2^{j/2+1} \leq
    \frac{\const\radius_1\lossrsc}{\spindex}
    \sqrt{\frac{T_0}{(\lips_j^2 + \noise_j^2)\log\pdim +
        \devcon^2\noise_j^2}} \quad \mbox{for all epochs}~j \leq
    i\right\}.
    \label{eqn:kstar}
\end{equation}
  
We start with a simple lemma, showing that
Algorithm~\ref{alg:epochdualavg} run with fixed epoch lengths $T_0$
has the desired behavior for the first $\kstar$ epochs.
\begin{lemma}
\label{lemma:goodepoch}
Suppose that $T_0 = \order(\log\pdim)$ and define $\kstar$ based on
equation~\eqref{eqn:kstar}. Then we have
\begin{equation*}
\norm{\proxcenter{k} - \opt}_1 \leq \radius_k \quad \mbox{and} \quad
\norm{\proxcenter{k} - \epochopt{k}}_1 \leq \radius_k \qquad \mbox{for
  all $1 \leq k \leq \kstar+1$}
\end{equation*}
with probabilty at least $1 - 3k\exp(-\devcon^2/12)$.  Under the same
conditions, there is a universal constant $c$ such that
\begin{equation*}
 \norm{\proxcenter{k} - \opt}_2 \leq
 \const\frac{\radius_k}{\sqrt{\spindex}} \quad \mbox{and} \quad
 \norm{\proxcenter{k} - \epochopt{k}}_2 \leq
 \const\frac{\radius_k}{\sqrt{\spindex}} \quad \mbox{ for all $2 \leq
   k \leq \kstar + 1$.}
\end{equation*}
\end{lemma}

The key challenge in proving the theorem is understanding the behavior
of the method after the first $\kstar$ epochs. Since the algorithm
cannot guarantee that the error to $\opt$ will be halved for epochs
beyond $\kstar$, we can no longer guarantee that $\opt$ will even be
feasible at the later epochs. However, this is exactly the same
problem that arose in the proof of
Theorem~\ref{thm:lipschitz-general}. Specifically, we can use
Lemma~\ref{lemma:badepochell2} in order to control the error after the
first $\kstar$ epochs.

In order to check the condition on epoch lengths in
Lemma~\ref{lemma:badepochell2}, we begin by observing that by the
definition~\eqref{eqn:kstar} of $\kstar$, we know that
\begin{equation}
\label{eqn:kstarbound}
\const \frac{\spindex}{\lossrsc}
\sqrt{\frac{\proxbound(\lips_{\kstar}^2 + \noise_{\kstar}^2) +
    \devcon^2 \noise_{\kstar}^2}{T_0}} \leq \radius_1 2^{-\kstar/2-1}
= \frac{\radius_{\kstar+1}}{2}.
  \end{equation}
Since we assume that the constants $\lips_k, \noise_k$ are decreasing
in $k$, the inequality also holds for all $k \geq \kstar+1$, so that
Lemma~\ref{lemma:badepochell2} applies in our setup here.  We further
observe that the setting of the epoch lengths in
Theorem~\ref{thm:lipschitz-fixedepoch} ensures that the total number
of epochs we perform is
\begin{equation*}
  k_0 = \log\left(\frac{\radius_1\lossrsc}{\spindex}
  \sqrt{\frac{\totiters}{(\lips^2 + \noise^2)\log\pdim +
      \devcon^2\noise^2}}\right).
\end{equation*}

Now we have one of two possibilities. Either $k_0 \leq \kstar$ or $k_0
\geq \kstar$. In the first case, Lemma~\ref{lemma:goodepoch} ensures
the error bound $\norm{\proxcenter{k_0} - \opt}_2^2 \leq \const\,
\radius_{k_0}^2/\spindex$ applies. In the second case, we appeal to
Lemma~\ref{lemma:badepochell2} and obtain an error bound of $\const\,
\radius_{\kstar}^2/\spindex$. The proof is completed by substituting
our choices of $k_0$ and $\kstar$ in these bounds.


\section{Simulations}
\label{sec:sims}

In this section, we present the results of various numerical
simulations that illustrate different aspects of our theoretical
convergence results. We focus on least-squares regression, described
in more detail in Example~\ref{example:leastsquares}. Specifically, we
generate samples $(\samplex_t, \sampley_t)$ with each coordinate of
$\samplex_t$ distributed as $\mbox{Unif}[-\xbound, \xbound]$ and
$\sampley_t = \ip{\opt}{\samplex_t} + \samplew_t$. We pick $\opt$ to
be sparse vector with $\spindex = \lceil \log \pdim \rceil$ non-zero
co-ordinates, and $\samplew_t \sim \normal{\gaussnoise^2}$ where
$\gaussnoise^2 = 0.5$. Given an iterate $\iter{t}$, we generate a
stochastic gradient of the expected loss~\eqref{eqn:objective} at
$(\samplex_t, \sampley_t)$. For the $\ell_1$-norm, we pick the sign
vector of $\iter{t}$, with $0$ for any component that is zero, a
member of the $\ell_1$-sub-differential.

Our first set of results evaluate Algorithm~\ref{alg:epochdualavg}
against other stochastic optimization baselines, where all algorithms
are given complete knowledge of problem parameters.  In this first set
of simulations, we terminate epoch $i$ once $\norm{\proxcenter{i+1} -
  \opt}_p^2 \leq \norm{\proxcenter{i} - \opt}_p^2 / 2$, which ensures
that $\opt$ remains feasible throughout, and tests the performance of
Algorithm~\ref{alg:epochdualavg} in the most favorable scenario. We
compare the algorithm against two baselines. The first baseline is the
regularized dual averaging (RDA) algorithm~\cite{Xiao10}, applied to
the regularized objective~\eqref{eqn:epochregobj} with $\regpar = 4
\gaussnoise \sqrt{\log\pdim/\totiters}$, which is the statistically
optimal regularization parameter with $\totiters$ samples. We use the
same prox-function $\prox(\param) = \frac{\norm{\param}_p^2}{2(p-1)}$,
so that the theory~\cite{Xiao10} for RDA predicts a convergence rate
of $\order(s\sqrt{\log\pdim/\totiters})$.  Our second baseline is the
stochastic gradient (SGD) algorithm, a method that exploits the strong
convexity but not the sparsity of the
problem~\eqref{eqn:objective}. Since the squared loss is not uniformly
Lipschitz, we impose an additional constraint $\norm{\param}_1 \leq
\radius_1$, without which the algorithm does not converge. The results
of this comparison are shown in Figure~\ref{fig:compare-simple}, where
we present the error $\norm{\iter{t} - \opt}_2^2$ averaged over 5
random trials. We observe that RADAR comprehensively outperforms both
the baselines, confirming the predictions of our theory.

\begin{figure}[ht]
  \centering
  \begin{tabular}{ccc}
  \widgraph{\FIGSIZE}{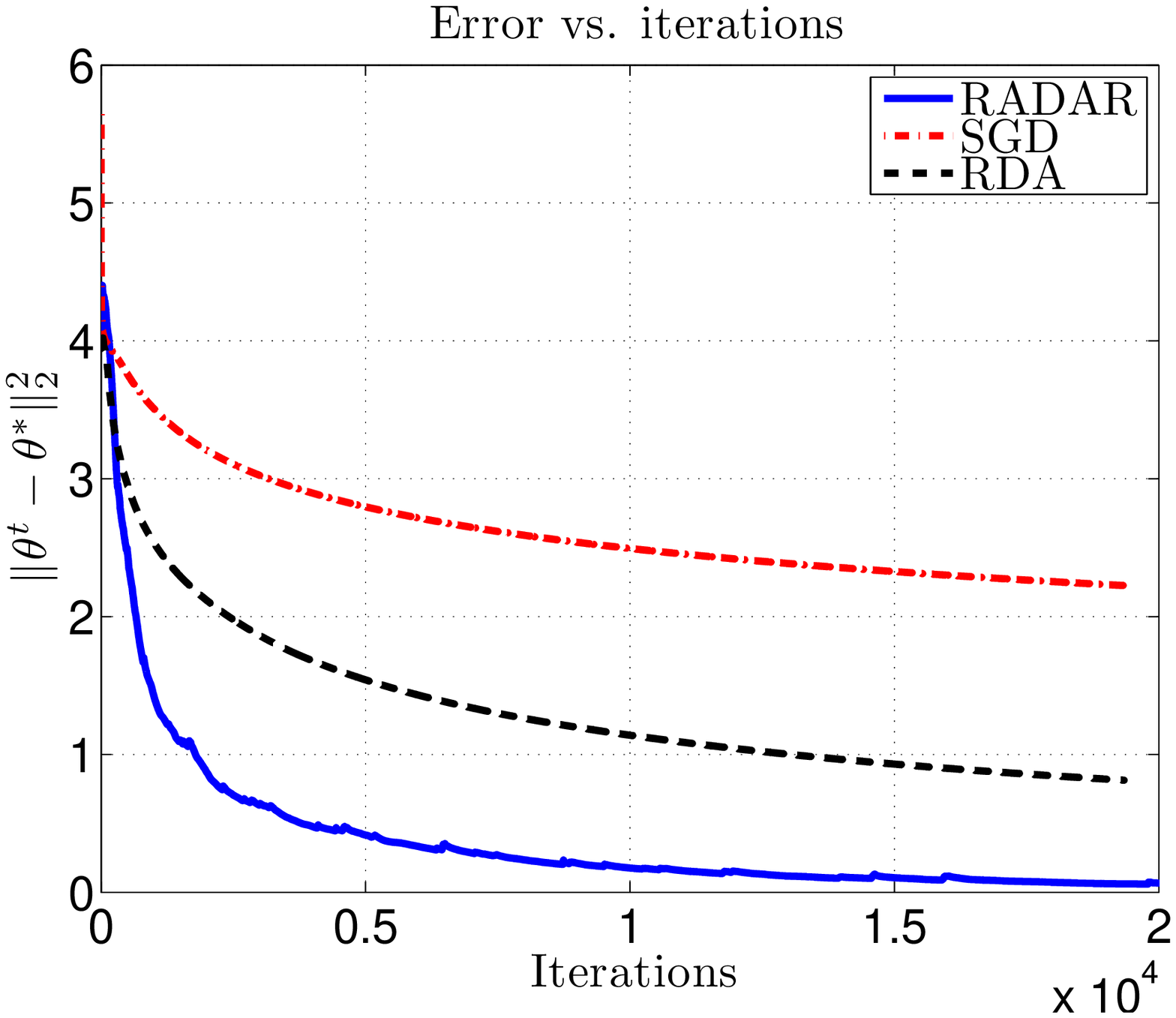}
  & \hspace*{.2in} &
  \widgraph{\FIGSIZE}{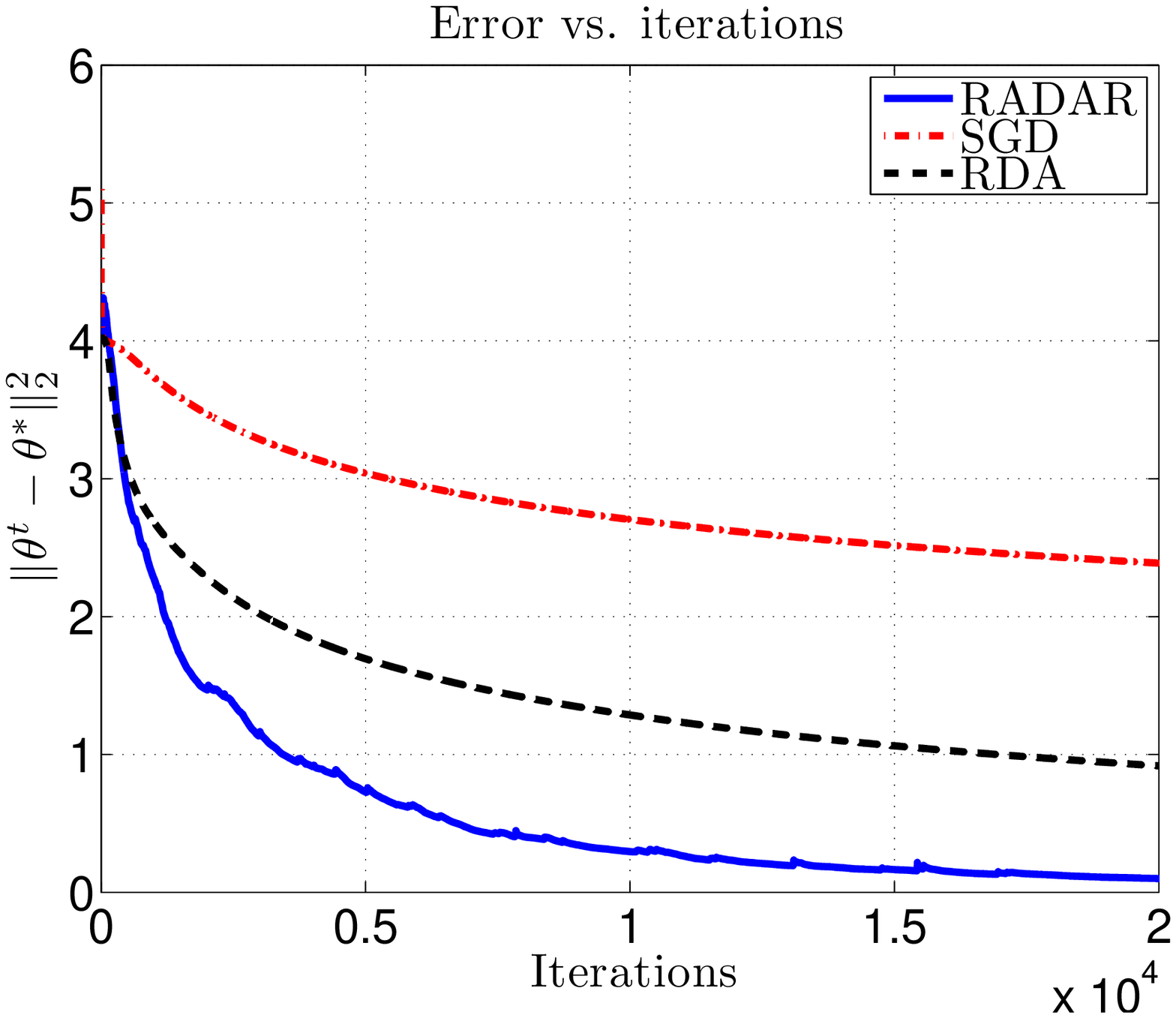}\\ $\pdim
  = 20000$ & & $\pdim = 40000$
  \end{tabular}
  \caption{A comparison of RADAR with the RDA and SGD algorithms for
    $\pdim = 20000$ (left) and $\pdim = 40000$ (right). We plot
    $\norm{\iter{t} - \opt}_2^2$ averaged over 5 random trials versus
    the number of iterations.}
  \label{fig:compare-simple}
\end{figure}

Our second set of results provides comparisons to algorithms that are
tailored to exploit sparsity.  Our first baseline here is the approach
that we described in our remarks following
Theorem~\ref{thm:lipschitz}. In this approach, we use the same
multi-step strategy as Algorithm~\ref{alg:epochdualavg} but keep
$\regpar$ fixed. We refer to this as Epoch Dual Averaging (henceforth
EDA), and again employ $\regpar = 4 \gaussnoise \sqrt{(\log \pdim)/
  \totiters}$ with this strategy. To maintain a fair comparison with
the RADAR algorithm, our epochs are again terminated by halving of the
squared $\ell_\pval$-error measured relative to $\opt$.  Finally, we
also evaluate the version of our algorithm with constant epoch
lengths, as analyzed in Theorem~\ref{thm:lipschitz-fixedepoch} using
epochs of length $\log(\totiters)$, and henceforth referred to as
RADAR-CONST.  As shown in Figure~\ref{fig:compare-fancy}, the
RADAR-CONST has relatively large error during the initial epochs,
before converging quite rapidly, a
phenomenon consistent with our theory.\footnote{To clarify, the epoch
  lengths in RADAR-CONST are set large enough to guarantee that we can
  attain an overall error bound of $\order(1/\totiters)$, meaning that
  the initial epochs for RADAR-CONST are much longer than for
  RADAR. Thus, after roughly 500 iterations, RADAR-CONST has done only
  2 epochs and operates with a crude constraint set
  $\Parset(\radius_1/4)$.  During epoch $i$, the step size scales
  proportionally to $R_i/\sqrt{t}$, where $t$ is the iteration number
  within the epoch; hence, when $\radius_i$ is large, the relatively
  large initial steps in an epoch can take us to a bad solution even
  when we start with a good solution $\proxcenter{i}$. As $\radius_i$
  decreases further with more epochs, this effect is mitigated and the
  error of RADAR-CONST does rapidly decrease like our theory
  predicts.}
Even though the RADAR-CONST method does not use the knowledge of
$\opt$ to set epochs, all three methods exhibit the same eventual
convergence rates, with RADAR (set with optimal epoch lengths)
performing the best.  Although RADAR-CONST is very slow in initial
iterations, its convergence rate remains competitive with EDA (even
though EDA \emph{does} exploit knowledge of $\opt$), but is worse than
RADAR as expected.

Overall, our experiments demonstrate that RADAR and RADAR-CONST have
practical performance consistent with our theoretical predictions.
Although optimal epoch length setting is not too critical for our
approach, better data-dependent empirical rules for determining epoch
lengths remains an interesting question for future research.  The
relatively poorer performance of EDA demonstrates the importance of
our decreasing regularization schedule.

\begin{figure}[ht]
\centering
\begin{tabular}{ccc}
\widgraph{\FIGSIZE}{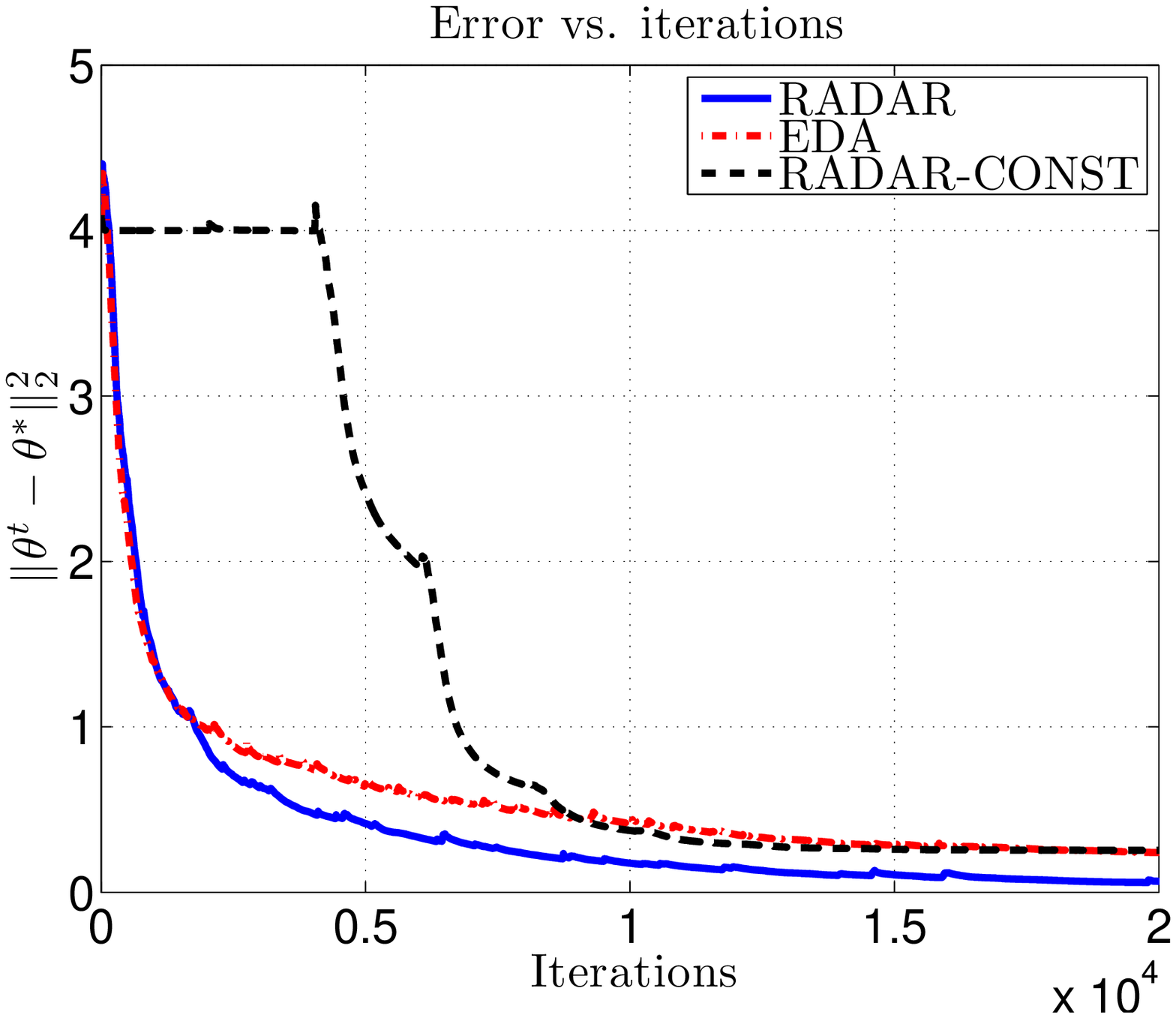}
& \hspace*{.2in} &
\widgraph{\FIGSIZE}{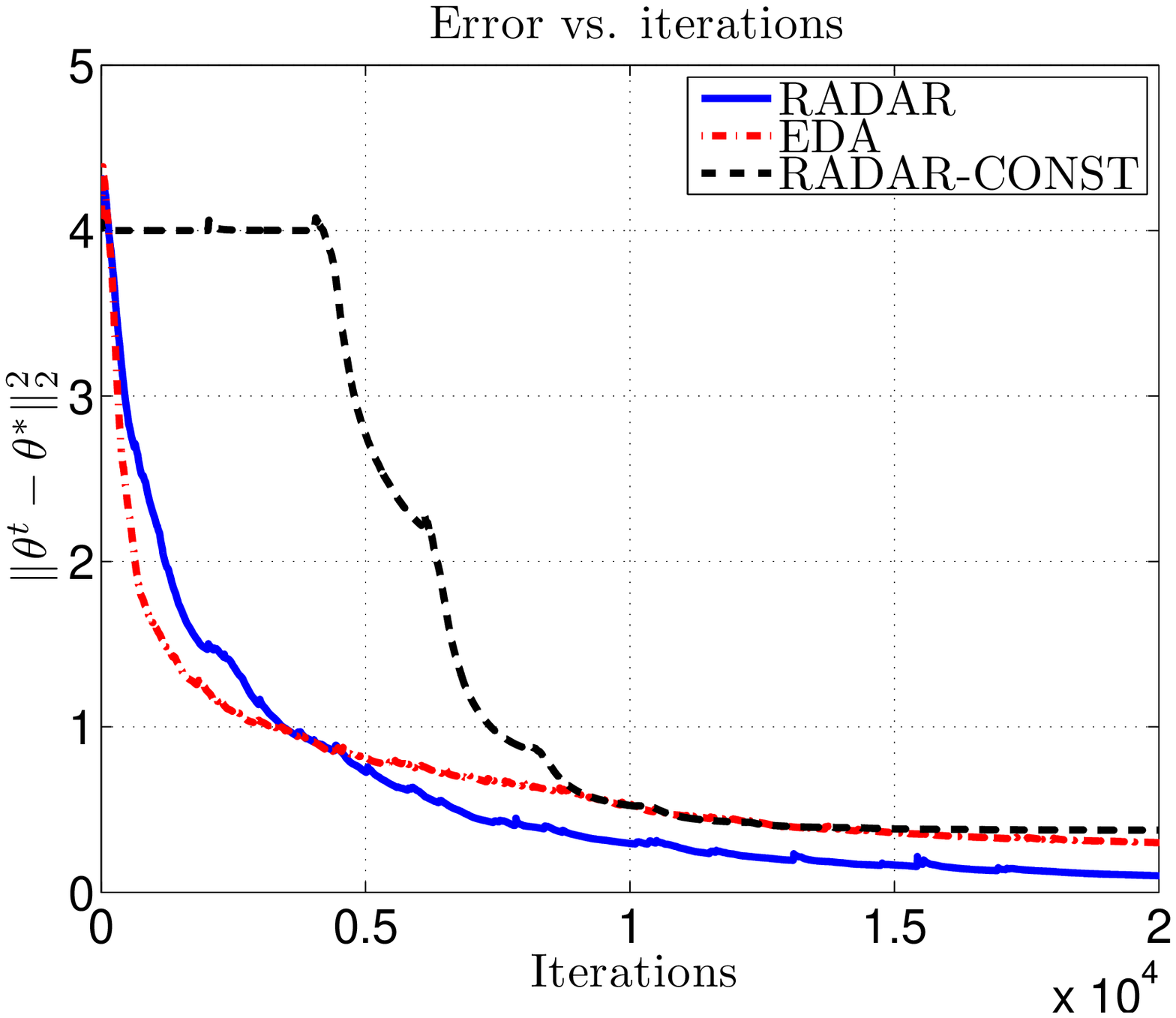}\\ $\pdim =
20000$ & & $\pdim = 40000$
\end{tabular}
\caption{A comparison of RADAR with EDA and RADAR-CONST for $\pdim =
  20000$ (left) and $\pdim = 40000$ (right). We plot $\norm{\iter{t} -
    \opt}_2^2$ averaged over 5 random trials versus the number of
  iterations.}
  \label{fig:compare-fancy}
\end{figure}

\section{Discussion}

In this paper, we presented an algorithm that is able to take
advantage of the strong convexity and sparsity conditions that are
satsified by many common problems in machine learning. Our algorithm
is simple and efficient to implement, and for a $\pdim$-dimensional
objective with an $\spindex$-sparse optima, it achieves the
minimax-optimal convergence rate $\order(\spindex\log\pdim/T)$.  We
also demonstrate optimal convergence rates for problems that have
weakly sparse optima, with implications for problems such as sparse
linear regression and sparse logistic regression. While we focus our
attention exclusively on sparse vector recovery due to space
constraints, the ideas naturally extend to other structures such as
group sparse vectors and low-rank matrices~\cite{NegRavWaiYu09}. It
would be interest to study similar developments for other algorithms
such as mirror descent or Nesterov's accelerated gradient methods,
leading to multi-step variants of those methods with optimal
convergence rates in our setting.

\subsection*{Acknowledgements}

The work of all three authors was partially supported by ONR MURI
grant N00014-11-1-0688 to MJW.  In addition, AA was partially
supported by a Google Fellowship, and SN was partially supported by the
Yahoo KSC award.


\appendix

\section{Closed-form updates}
\label{AppClosed}

In this appendix, we derive a closed form expression for the
update~\eqref{eqn:regdualavg} when $\Parset = \R^\pdim$. Recalling our
definition of the prox-function~(\ref{EqnDefnProx}), the constraint
$\iter{t+1} \in \Parset(\radius_i)$ can be rewritten as
\mbox{$\prox_{\proxcenter{i}, \radius_i}(\param) \leq 2 (\pval-1)$.}
We now form the Lagrangian at iteration $t+1$
\begin{equation*}
  \mystep{t+1} \ip{\diter{t+1}}{\param} + \prox_{\proxcenter{i},
    \radius_i} (\param) + \xi(\prox_{\proxcenter{i}, \radius_i}
  (\param) - 2 (\pval-1) ),
\end{equation*}
where $\xi \geq 0$ is the Lagrangian parameter. The first-order
optimality condition for the Lagrangian allows us to conclude that
$\mystep{t+1} \diter{t+1} + \nabla \prox_{\proxcenter{i},
  \radius_i}(\iter{t+1})(1 + \xi) = 0$, so that the iterate at time
$t+1$ is given by
\begin{equation*}
  \iter{t+1} = \nabla \prox^*_{\proxcenter{i}, \radius_i} \left(-
  \frac{\mystep{t+1} \diter{t+1}} {1 + \xi} \right),
\end{equation*}
where $\prox_{\proxcenter{i}, \radius_i}^*(\dualpar)$ denotes the
Fenchel conjugate~\cite{HiriartUrrutyLe96}. Recalling the form of our
prox-function, we have

\begin{equation*}
  \prox_{\proxcenter{i}, \radius_i}(\param) =
  \frac{1}{2(p-1)\radius_1^2}\norm{\param - \proxcenter{i}}_p^2 \quad
  \mbox{and} \quad \prox_{\proxcenter{i}, \radius_i}^* (\dualpar) =
  \ip{\proxcenter{i}}{\dualpar} + \frac{(p-1)\radius_i^2}{2}
  \norm{\dualpar}_q^2,
\end{equation*}
where $q = p/(p-1)$ is the conjugate exponent to $p$. This is a
straightforward consequence of the Fenchel duality of $\ell_p$ norms
(e.g., see Example 6.0.2~~\cite{HiriartUrrutyLe96}).

We can now take the gradient of the dual function to obtain the
following closed form expression
\begin{equation*}
  \iter{t+1} = \proxcenter{i} +
  \frac{\radius_i^2(p-1)\mystep{t+1}}{(1+\xi)} |\diter{t+1}|^{(q-1)}
  \mbox{sign}(\diter{t+1}) \norm{\diter{t+1}}_q^{(2-q)}. 
\end{equation*}
The value of $\xi$ can now be obtained by backsubstitution in the
constraint $\param \in \Parset(\radius_i)$.  Doing so and performing
some algebra yields
\begin{equation*}
  \xi \defn \max \left \{0, (\pval-1) \mystep{t+1}
  \norm{\diter{t+1}}_q\radius_i - 1\right\},
\end{equation*}
where $|\diter{t+1}|^{(q-1)}$ refers to taking the absolute values and
exponents elementwise and $q = p/(p-1)$ is the conjugate exponent to
$\pval$.


\section{Proofs for convergence within a single epoch}
\label{app:within-epoch}

In this appendix, we prove various results on convergence behavior
within a single epoch, including Lemmas~\ref{lemma:opterror}
and~\ref{lemma:rsc-simple} as well as Proposition~\ref{prop:mainepoch}.

\subsection{Proof of Lemma~\ref{lemma:opterror}}

By the optimality of $\epochopt{i}$, we have
\begin{align}
\label{EqnRaventosInitial}
\Lossbar(\epochopt{i}) + \regpar_i\norm{\epochopt{i}}_1 & \leq
\Lossbar(\opt) + \regpar_i\norm{\opt}_1.
\end{align}
From the local RSC assumption~\eqref{EqnLRSC}, for any vector $\param$
feasible during epoch $i$, we have the lower bound
\begin{align} 
\Lossbar(\param) & \geq \Lossbar(\opt) + \ip{\nabla
  \Lossbar(\opt)}{\param - \opt} + \frac{\lossrsc}{2} \norm{\param -
  \opt}_2^2 - \slopfac \norm{\param - \opt}_1^2 \nonumber \\
\label{EqnRaventos}
& \stackrel{(i)}{\geq} \Lossbar(\opt) + \frac{\lossrsc}{2}
\norm{\param - \opt}_2^2 - \slopfac \norm{\param - \opt}_1^2,
\end{align}
where step (i) follows since $\opt$ minimizes $\Lossbar(\param)$.
Applying inequality~\eqref{EqnRaventos} with $\param = \epochopt{i}$
and combining with the initial bound~\eqref{EqnRaventosInitial} yields
\begin{align*}
\Lossbar(\epochopt{i}) + \regpar_i\norm{\epochopt{i}}_1 & \leq
\Lossbar(\epochopt{i}) - \frac{\lossrsc}{2} \norm{\epochopt{i} -
  \opt}_2^2 + \slopfac \norm{\epochopt{i} - \opt}_1^2 +
\regpar_i\norm{\opt}_1.
\end{align*}
Using the definition $\opt = \epochopt{i} + \DelIt{i}$ and triangle
inequality, we can further simplify to obtain
 \begin{align*}
  \regpar_i\norm{\epochopt{i}}_1 &\leq \regpar_i\norm{\opt}_1 -
  \frac{\lossrsc}{2} \norm{\opt - \epochopt{i}}_2^2 + \slopfac
  \norm{\epochopt{i} - \opt}_1^2 \; \leq \; \regpar_i \norm{
    \epochopt{i}}_1 + \regpar_i \norm{\DelIt{i}}_1 -
  \frac{\lossrsc}{2} \norm{\DelIt{i}}_2^2 + \slopfac
  \norm{\DelIt{i}}_1^2,
\end{align*}
Rearranging the terms above yields
\begin{align*}
\frac{\lossrsc}{2}\norm{\DelIt{i}}_2^2 & \leq \regpar_i
\norm{\DelIt{i}}_1 + \slopfac \norm{\DelIt{i}}_1^2 \leq 2
\sqrt{\spindex} \regpar_i \norm{\DelIt{i}}_2 
+ 2 \regpar_i \coneslack
+ 8 \spindex \slopfac \norm{\DelIt{i}}_2^2 + 8 \slopfac \coneslack^2.
\end{align*}
Some elementary algebra then yields that $\epsilon \defn
\norm{\DelIt{i}}_2$ satisfies the quadratic inequality
\begin{align*}
\frac{1}{2} \big \{ \lossrsc - 16 \spindex \slopfac \big \} \epsilon^2
- \big \{ 2 \sqrt{\spindex} \regpar_i \big \} \epsilon - \big \{2
\regpar_i \coneslack + 8 \slopfac \coneslack^2 \big \} & \leq 0,
\end{align*}
which then implies the $\ell_2$-error bound~\eqref{EqnOptEllTwo}. \\

In order to establish the $\ell_1$-error bound~\eqref{EqnOptEllOne},
we require an auxiliary lemma that allows us to translate between
the $\ell_2$ and $\ell_1$-norms:
\begin{lemma}
\label{lemma:cone}
For any pair of vectors $\param, \paramother \in \Parset$, suppose
that $\norm{\paramother}_1 \leq \norm{\param}_1 + \epsilon$ for some
$\epsilon \geq 0$.  Then for any set $\Tset \subseteq \{1, 2, \ldots,
\pdim\}$, the vector $\Delta \defn \param - \paramother$ satisfies the
inequality
\begin{align}
\norm{\Delta_{\TsetComp}}_1 \leq & \norm{\Delta_{\Tset}}_1 + 2
\norm{\param_{\TsetComp}}_1 + \epsilon.
\end{align}
\end{lemma}
\begin{proof}
Since $\Tset$ and $\TsetComp$ are disjoint, the bound assumed in the
lemma statement can be written
\begin{equation}
\label{EqnGramona}
\norm{\paramother_\Tset}_1 + \norm{\paramother_{\TsetComp}}_1 \leq
\norm{\param_\Tset}_1 + \norm{\param_{\TsetComp}}_1 + \epsilon.
\end{equation} 
Since $\paramother = \Delta + \param$ by definition, triangle
inequality implies that 
\begin{equation*}
\|\paramother_\Tset\|_1 \geq \|\param_\Tset\|_1 - \|\Delta_\Tset\|_1
\quad \mbox{and} \quad \|\paramother_{\TsetComp}\|_1 \geq
    \|\Delta_{\TsetComp}\|_1 - \|\param_{\TsetComp}\|_1. 
\end{equation*}
Substituting into the bound~\eqref{EqnGramona}, we obtain
\begin{align*}
\norm{\Delta_{\TsetComp}}_1 - \norm{\param_{\TsetComp}} +
\norm{\param_\Tset}_1 - \norm{\Delta_\Tset}_1 \leq
\norm{\param_\Tset}_1 + \|\param_{\TsetComp}\|_1 + \epsilon,
\end{align*}
and rearranging terms completes the proof.
\end{proof}

In particular, applying Lemma~\ref{lemma:cone} to the pair $\param =
\epochopt{i}$, and $\paramother = \opt$ with the tolerance $\epsilon =
0$ and the subset $\Tset = \Sset$, we find that the error vector
$\DelIt{i} \defn \epochopt{i} - \opt$ satisfies the bounds
\begin{align}
\label{eqn:optcone}  
\norm{(\DelIt{i})_{\SsetComp}}_1 & \leq \norm{(\DelIt{i})_S}_1 + 2
\coneslack.
\end{align}
Consequently, we have
\begin{align}
\label{EqnKiva}
\norm{\DelIt{i}}_1 & = \norm{(\DelIt{i})_\Sset} +
\norm{(\DelIt{i})_{\SsetComp}}_1 \; \leq \; 2 \norm{(\DelIt{i})_S}_1 +
2 \coneslack \; \leq \; 2 \sqrt{\spindex} \norm{(\DelIt{i})_S}_2 + 2
\coneslack,
\end{align}
where the final step uses the fact that $(\DelIt{i})_\Sset$ is an
$\spindex$-vector.


\subsection{Proof of Proposition~\ref{prop:mainepoch}:  
Inequality~\eqref{eqn:fvaluebound}}

We are now equipped to prove Proposition~\ref{prop:mainepoch},
beginning with the first bound~\eqref{eqn:fvaluebound}.  Introducing
the convenient shorthand $\mygraderr{t} = \stochgrad{t} - \nabla
\Lossbar(\iter{t})$, our assumptions guarantee that there are
constants $\lips_i$ and $\noise_i$ such that $\E \exp \left
(\frac{\norm{\mygraderr{t}}_\infty^2}{\noise_i^2} \right) \leq
\exp(1)$, and
\begin{align}
\label{eqn:epochbounds}
  | \Lossbar(\param) - \Lossbar(\paramother) | & \leq
  \lips_i\norm{\param - \paramother}_1 \quad \mbox{for all $\param,
    \paramother$ satisfying $\norm{\param - \proxcenter{i}}_1 \leq
    \radius_i$.}
\end{align}
Our starting point is a known result for the convergence of the
stochastic dual averaging algorithm.  Recalling the definition
$\xavg(T) = \sum_{t=1}^T \iter{t}/T$, and letting $\compositegrad{t} =
\stochgrad{t} + \regpar_i\reggrad{t}$ be the stochastic subgradient at
iteration $t$, we have
\begin{align}
  \label{eqn:simpledadeterministic}
  f_i(\xavg(T)) - f_i(\epochopt{i}) & \leq \frac{1}{ 2 T} \sum_{t=1}^T
  \frac{ \mystep{t - 1}}{\gamma_{\prox}}
  \norm{\compositegrad{t}}_\infty^2 + \frac{1}{ T \mystep{T}}
  \prox_{\proxcenter{i}, \radius_i}(\epochopt{i}) - \frac{1}{T}
  \sum_{t=1}^T \ip{\mygraderr{t}} {\iter{t} - \epochopt{i}},
\end{align}
where $\gamma_{\prox}$ is the strong convexity coefficient of the
prox-function with respect to the $\ell_\infty$ norm, equal to
$1/(e\radius_i^2)$ in our case. This bound follows directly from the
analysis of Nesterov~\cite{Nesterov09} and Xiao~\cite{Xiao10}; the
specific form~\eqref{eqn:simpledadeterministic} given here corresponds
to Lemma 2 of Duchi et al.~\cite{DuchiAgWa11}.

Now observe that since $\compositegrad{t} = \nabla \Lossbar(\iter{t})
+ \regpar_i \reggrad{t} + \mygraderr{t}$, triangle inequality yields
the upper bound
\begin{align*}
\norm{\compositegrad{t}}_\infty^2 & \leq 2 \big( \norm{\nabla
  \Lossbar(\iter{t}) + \regpar_i \reggrad{t}}_\infty^2 +
\norm{\mygraderr{t}}_\infty^2 \big) \; \stackrel{(i)}{\leq} \; 4
\lips_i^2 + 4\regpar_i^2 + 2\norm{\mygraderr{t}}_\infty^2,
\end{align*}
where inequality (i) uses the Lipschitz condition in
Assumption~\ref{ass:explips}, and the Lipschitz property of the
$\ell_1$-norm.  From this point, further simplifying the error
bound~\eqref{eqn:epochbounds} requires controlling the random terms
\begin{equation}
\label{EqnEURODayOne}
\sum_{t=1}^T \mystep{t-1} \norm{\mygraderr{t}}_\infty^2 \quad
\mbox{and} \quad \sum_{t=1}^T \ip{\mygraderr{t}} { \iter{t} -
  \epochopt{i}}
\end{equation}
Accordingly, we state an auxiliary lemma that provide tail bounds
appropriate for this purpose. \\

\begin{lemma}
\label{lemma:errsquarebound}
Under the sub-Gaussian tail condition (Assumption~\ref{ass:subgauss}):
\begin{enumerate}
\item[(a)] With step sizes $\mystep{t} = \stepsize/\sqrt{t}$, we have
\begin{align}
\sum_{t=1}^T \mystep{t-1} \norm{\mygraderr{t}}_\infty^2 \leq 2
\noise^2_i \stepsize \sqrt{T} + \noise^2_i \devcon \stepsize \sqrt{
  2\log T}
\end{align}
with probability at least $1 - 2 \exp(-\devcon^2/12)$ for all $\devcon
\leq 9 \sqrt{\log T}$.
\item[(b)] We have $\sum_{t=1}^T \inprod{\mygraderr{t}}{\epochopt{i} -
  \iter{t}} \leq \devcon \radius_i \noise_i \sqrt{T}$ with probability
  at least \mbox{$1 - 2 \exp(-\devcon^2/12)$.}
\end{enumerate}
\end{lemma}


See Appendix~\ref{AppLemErrsquarebound} for the proof of this result.
We now use it to control the terms in equation~\eqref{EqnEURODayOne}.
Starting with the first term, we observe from
Lemma~\ref{lemma:errsquarebound} (a) that for $T \geq 1$, we are
guaranteed that
\begin{align*}
  \sum_{t=1}^T \mystep{t-1} \norm{\compositegrad{t}}_\infty^2 & \leq
  \sum_{t=1}^T \mystep{t-1} (4 \lips_i^2 + 4 \regpar_i^2 +
  2\norm{\mygraderr{t}}_\infty^2) \\
& \leq 4 (\lips_i^2 + \regpar_i^2) \sum_{t=1}^T \mystep{t-1} + 2
  \noise_i^2 \stepsize (2\sqrt{T} + \devcon_i \sqrt{\log T}) \\
& \leq 8 (\lips_i^2 + \regpar_i^2) \stepsize \sqrt{T} + 22 \noise_i^2
  \stepsize \sqrt{T},
\end{align*}
with probability at least $1 - 2 \exp(-\devcon_i^2/12)$.  Here the
last step uses the inequality $9 \log T \leq \sqrt{T}$ valid for all
$T \geq 1$, as well as the assumption $\devcon_i \leq 9 \sqrt{\log
  T}$. Thus, we have established an upper bound on the gradient terms
with an \emph{effective Lipschitz constant}
\begin{equation}
  \label{eqn:errsquarebound}
  \sum_{t=1}^T \mystep{t-1} \norm{\compositegrad{t}}_\infty^2 \leq 22
  \stepsize \sqrt{T} (\lips_i^2 + \regpar_i^2 + \noise_i^2).
\end{equation}
Part (b) of Lemma~\ref{lemma:errsquarebound} directly controls the
second random quantity.  

We now plug in the results of these lemmas into our earlier error
bound~\eqref{eqn:simpledadeterministic}, which yields, with
probability at least $1- 3 \exp(-\devcon_i^2/12)$, the upper bound
\begin{align*}
  f_i(\xavg(T)) - f_i(\epochopt{i}) &\leq \frac{22 \stepsize}{2
    \gamma_{\prox} \sqrt{T}} (\lips_i^2 + \regpar_i^2 + \noise_i^2) +
  \frac{\proxbound}{ \stepsize \sqrt{T}} + \frac{\devcon_i \noise_i
    \radius_i}{\sqrt{T}} \\
& \leq 10 \sqrt {\lips_i^2 + \noise_i^2 + \regpar_i^2}
  \sqrt{\frac{\proxbound}{\gamma_{\prox} T}} + \frac{\devcon_i
    \noise_i \radius_i}{\sqrt{T}}.
\end{align*}
Here the second inequality uses the setting $\stepsize = 5\sqrt{
  \proxbound \gamma_{\prox}/ (\lips_i^2 + \regpar_i^2 +
  \noise_i^2)}$. We also note that under our assumption that $\prox$
is $1$-strongly convex with respect to $\norm{\cdot}_1$, we have that
$\gamma_{\prox} = 1/(e \radius_i^2)$ at epoch $i$. Thus with
probability at least $1 - 3 \exp(-\devcon_i^2/12)$, we have the error
bound
\begin{align}
\label{eqn:epochmainbound}
f_i(\xavg(T)) - f_i(\epochopt{i}) & \leq 30
\radius_i\sqrt{\frac{\proxbound(\lips_i^2 + \noise_i^2 +
    \regpar_i^2)}{T}} + \frac{\devcon_i \noise_i \radius_i}{\sqrt{T}}
\nonumber \\
& \leq 30 \radius_i \sqrt{\frac{\proxbound(\lips_i^2 +
    \noise_i^2)}{T}} + 30 \radius_i \regpar_i
\sqrt{\frac{\proxbound}{T}} + \frac{\devcon_i \noise_i
  \radius_i}{\sqrt{T}},
\end{align}
thus completing the first bound in Proposition~\ref{prop:mainepoch}.


\subsection{Proof of Lemma~\ref{lemma:rsc-simple}}
\label{AppLemIterCone}

\newcommand{\TMP}{\ensuremath{T_i}}

The main idea of this proof is to first convert the error
bound~\eqref{eqn:epochmainbound} from function values into $\ell_1$
and $\ell_2$-norm bounds by exploiting the (approximate) sparsity of
$\opt$. We will then use these bounds to simplify the RSC
condition. Since the error bound~\eqref{eqn:epochmainbound} for the
minimizer $\epochopt{i}$, it also holds for any other feasible vector.
In particular, applying it to $\opt$, we obtain the bound
\begin{align}
\label{EqnSandwich}
f_i(\xavg(\TMP)) - f_i(\opt) & \leq \const\,
\frac{30\radius_i}{\regpar_i} \sqrt{\frac{\proxbound(\lips_i^2 +
    \noise_i^2)}{\TMP}} + 30 \radius_i \sqrt{\frac{\proxbound}{\TMP}}
+ \frac{\devcon_i \noise_i \radius_i}{\regpar_i \sqrt{\TMP}}.
\end{align}
Our next step is to lower bound the left-hand side of this inequality.
We have
\begin{align*}
f_i(\xavg(\TMP)) - f_i(\opt) & = \Lossbar(\xavg(\TMP)) +
\regpar_i\norm{\xavg(\TMP)}_1 - \Lossbar(\opt) -
\regpar_i\norm{\opt}_1 \\
& \stackrel{(i)}{\geq} \Lossbar(\opt) + \regpar_i \norm{\xavg(\TMP)}_1
- \Lossbar(\opt) - \regpar_i \norm{\opt}_1 \\
& = \regpar_i \big \{ \norm{\xavg(\TMP)}_1 - \norm{\opt}_1 \big \},
\end{align*}
where inequality (i) follows since $\opt$ minimizes $\Lossbar$.
Combining with the bound~\eqref{EqnSandwich} yields
\begin{equation}
\label{EqnEngland}
\norm{\xavg(\TMP)}_1 \leq \norm{\opt}_1 + \frac{30
  \radius_i}{\regpar_i} \sqrt{\frac{\proxbound (\lips_i^2 +
    \noise_i^2)}{\TMP}} + 30 \radius_i \sqrt{\frac{\proxbound}{\TMP}}
+ \frac{\devcon_i \noise_i \radius_i}{\regpar_i \sqrt{\TMP}}.
\end{equation}
At this point we recall the shorthand notations $\DelAvg(\TMP) =
\xavg(\TMP) - \opt$ and $\DelBarAvg(\TMP) = \xavg(\TMP) -
\epochopt{i}$. In order to bring the above bound on
$\norm{\xavg(\TMP)}_1$ closer to the statement of the lemma, we can
appeal to Lemma~\ref{lemma:cone}. Indeed an application of the lemma
in conjunction with the inequality~\eqref{EqnEngland} results in the
bound
\begin{equation}
  \norm{\DelAvg(\TMP)_{\SsetComp}}_1 \leq \norm{\DelAvg(\TMP)_S}_1 +
  \frac{30\radius_i}{\regpar_i} \sqrt{\frac{\proxbound(\lips_i^2 +
      \noise_i^2)}{\TMP}} + 30 \radius_i
  \sqrt{\frac{\proxbound}{\TMP}} + \frac{\devcon_i \noise_i
    \radius_i}{\regpar_i \sqrt{\TMP}} + 2 \coneslack.
  \label{eqn:conedelavg}
\end{equation}

Our next step is to convert the above cone bound on $\DelAvg(\TMP)$
into a similar result for $\DelBarAvg(\TMP)$.  In order to do so, we
observe that $\DelAvg(\TMP) - \DelBarAvg(\TMP) = \epochopt{i} - \opt$,
and hence
\begin{align*}
\|\epochopt{i} - \opt\|_1 & = \|\DelAvg_\Sset(\TMP) -
\DelBarAvg_\Sset(\TMP)\|_1 + \|\DelAvg_\Sbar(\TMP) -
\DelBarAvg_\Sbar(\TMP)\|_1 \\
& \geq \big \{\|\DelAvg_\Sset(\TMP)\|_1 - \|\DelBarAvg_\Sset(\TMP)\|_1
\big \} - \big \{ \|\DelAvg_\Sbar(\TMP)\|_1 -
\|\DelBarAvg_\Sbar(\TMP)\|_1 \big\},
\end{align*}
and hence
\begin{equation*}
  \norm{\DelBarAvg_\Sbar(\TMP)}_1 - \norm{\DelBarAvg_\Sset(\TMP)}_1
  \leq \norm{\DelAvg_\Sbar(\TMP)}_1 - \norm{\DelAvg_\Sset(\TMP)}_1 +
  \norm{\epochopt{i} - \opt}_1.
\end{equation*}
Consequently, Lemma~\ref{lemma:opterror} provides the final piece to
complete the proof. Combining inequality~\eqref{eqn:epocherrorbound}
obtained from Lemma~\ref{lemma:opterror} with our earlier
bound~\eqref{eqn:conedelavg} yields 
\begin{equation*}
\norm{\DelBarAvg(\TMP)_{\SsetComp}}_1 \leq \norm{\DelBarAvg(\TMP)_S}_1
+ \frac{9 \spindex \regpar_i}{\realrsc} + \coneslack \left ( 6 + 8
\sqrt{\frac{\spindex \slopfac}{\realrsc}} \right) +
\frac{30\radius_i}{\regpar_i} \sqrt{\frac{\proxbound (\lips_i^2 +
    \noise_i^2)}{\TMP}} + 30 \radius_i \sqrt{\frac{\proxbound}{\TMP}}
+ \frac{\devcon_i \noise_i \radius_i}{\regpar_i \sqrt{\TMP}}.
\end{equation*}
Consequently, a further use of the inequality
$\norm{\DelBarAvg(\TMP)_{\Sset}}_1 \leq \sqrt{\spindex}
\norm{\DelBarAvg(\TMP)}_2$ allows us to conclude that there is a
universal constant $\plaincon$ such that
\begin{align}
  \norm{\DelBarAvg}_1^2 & \leq 8 \spindex\norm{\DelBarAvg}_2^2 + c \,
  \left( \frac{\spindex^2 \regpar_i^2}{\realrsc^2} +
  \frac{\radius_i^2}{\regpar_i^2 \TMP} \left(\proxbound (\lips_i^2 +
  \noise_i^2) + \devcon_i^2 \noise_i^2 \right) +
  \frac{\radius_i^2\proxbound}{\TMP} + \coneslack^2 \left(1 +
  \frac{\spindex \slopfac}{\realrsc} \right) \right)
  \label{eqn:iterconebound}
\end{align}
with probability at least \mbox{$1 - 3\exp(-\devcon_i^2/12)$.}
Substituting the settings~\eqref{eqn:paramsettings-lipschitz}
and~\eqref{eqn:epochlength} of $\regpar_i$ and $\TMP$ respectively
into the above bound yields
\begin{align}
  \norm{\DelBarAvg(\TMP)}_1^2 & \leq 8 \spindex\norm{\DelBarAvg(\TMP)}_2^2 +
  c \, \left( \radius_i^2 \frac{\realrsc}{\lossrsc} + \spindex
  \totalslacksq \right),
  \label{eqn:iterconebound-simple}
\end{align}
where we recall the notation \mbox{$\totalslacksq =
  \frac{\coneslack^2}{\spindex} \big (1 + \frac{\spindex
    \slopfac}{\realrsc} \big)$}.

In order to complete the proof, we now invoke the RSC assumption
applied to the function $f_i$. Specifically, since $\epochopt{i}$
minimizes $f_i$, the RSC condition implies that
\begin{equation*}
  \frac{\lossrsc}{2} \norm{\DelBarAvg(\TMP)}_2^2  \leq f_i(\xavg(\TMP)) -
  f_i(\epochopt{i}) + \slopfac \norm{\DelBarAvg(\TMP)}_1^2. 
\end{equation*}
Combining the above inequality with the
bound~\eqref{eqn:iterconebound-simple} yields
\begin{equation*}
  \frac{\lossrsc}{2} \norm{\DelBarAvg(\TMP)}_2^2 \leq f_i(\xavg(\TMP)) -
  f_i(\epochopt{i}) + \slopfac \left[
    8\spindex\norm{\DelBarAvg(\TMP)}_2^2 + c \, \left( \radius_i^2
    \frac{\realrsc}{\lossrsc} + \spindex \totalslacksq \right)\right].
\end{equation*}
Rearranging terms and recalling the notation $\realrsc = \lossrsc -
16\spindex \slopfac$ completes the proof. 


\subsection{Proof of Proposition~\ref{prop:mainepoch}:  
Inequality~\eqref{eqn:ell1bound}}

Equipped with Lemma~\ref{lemma:rsc-simple}, we are now ready to prove
the second part of Proposition~\ref{prop:mainepoch}.  In particular,
using the inequality~\eqref{eqn:iterconebound} in the proof of
Lemma~\ref{lemma:rsc-simple}, we observe that with probability at
least \mbox{$1 - 3 \exp(-\devcon_i^2/12)$,} we have
\begin{align*}
\norm{\DelBarAvg}_1^2 & \leq 8 \spindex \norm{\DelBarAvg}_2^2 + c \,
\left( \frac{\spindex^2 \regpar_i^2}{\realrsc^2} +
\frac{\radius_i^2}{\regpar_i^2 \TMP} \left (\proxbound (\lips_i^2 +
\noise_i^2) + \devcon_i^2 \noise_i^2 \right) + \frac{\radius_i^2
  \proxbound}{\TMP} + \spindex \totalslacksq \right) \\
& \leq \frac{16 \spindex}{\lossrsc} (f_i(\xavg(\TMP)) - f_i(\epochopt{i})
+ \slopfac \norm{\DelBarAvg}_1^2) + c \, \left( \frac{\spindex^2
  \regpar_i^2}{\realrsc^2} + \frac{\radius_i^2}{\regpar_i^2 \TMP} \left
(\proxbound ( \lips_i^2 + \noise_i^2) + \devcon_i^2 \noise_i^2 \right) +
\frac{\radius_i^2 \proxbound}{\TMP} + \spindex \totalslacksq \right).
\end{align*}
Here the second inequality uses the local RSC
condition~\eqref{EqnLRSC}, and the fact that $\epochopt{i}$ minimizes
$f_i$. 

From hereonwards, all our inequalities hold with probability at least
\mbox{$1 - 3 \exp(-\devcon_i^2/12)$,} so that we no longer state it
explicitly.  Rearranging terms and recalling the
definition~\eqref{eqn:realrsc} of $\realrsc$, we obtain that
\begin{align*}
\frac{\realrsc}{\lossrsc} \norm{\DelBarAvg}_1^2 & \leq
\frac{16\spindex}{\realrsc} \, (f_i(\xavg(\TMP)) - f_i(\epochopt{i})) + c
\, \left (\frac{\spindex^2 \regpar_i^2}{\realrsc^2} +
\frac{\radius_i^2}{\regpar_i^2 \TMP} \left (\proxbound (\lips_i^2 +
\noise_i^2) + \devcon_i^2 \noise_i^2 \right ) + \frac{\radius_i^2
  \proxbound}{\TMP} + \spindex \totalslacksq \right).
\end{align*}
Combining the above bound with our earlier
inequality~\eqref{eqn:epochmainbound} yields
$\frac{\realrsc}{\lossrsc}\norm{\DelBarAvg}_1^2 \leq \TERMONE +
\TERMTWO$, where
\begin{subequations}
\begin{align}
\label{EqnTermOne}
\TERMONE & \defn \frac{16\spindex\radius_i}{\realrsc\sqrt{\TMP}} \left( 4
\sqrt{\proxbound(\lips_i^2 + \noise_i^2)} + 4 \regpar_i
\sqrt{\proxbound} + \devcon_i\noise_i\right), \quad \mbox{and} \\
\label{EqnTermTwo}
\TERMTWO & \defn c \, \left( \frac{\spindex^2\regpar_i^2}{\realrsc^2}
+ \frac{\radius_i^2}{\regpar_i^2 \TMP} \left(\proxbound(\lips_i^2 +
\noise_i^2) + \devcon_i^2 \noise_i^2 \right) +
\frac{\radius_i^2\proxbound}{\TMP} + \spindex \totalslacksq \right)
\end{align}
\end{subequations}
By the Cauchy-Schwartz inequality, we have
$\frac{\spindex\radius_i\regpar_i}{\realrsc}\sqrt{\frac{\proxbound}{\TMP}}
\leq 2\frac{\spindex^2\regpar_i^2}{\realrsc^2} +
2\frac{\radius_i^2\proxbound}{\TMP}$, and hence
\begin{align}
\label{EqnGoat}
\TERMONE & \leq \frac{16\spindex\radius_i}{\realrsc\sqrt{\TMP}} \left(
4 \sqrt{\proxbound(\lips_i^2 + \noise_i^2)} + \devcon_i \noise_i \right)
+ 128\left( \frac{\spindex^2\regpar_i^2}{\realrsc^2} +
\frac{\radius_i^2\proxbound}{\TMP} + \right).
\end{align}
Noting that $\TERMONE$ and $\TERMTWO$ involve multiple terms, some
increasing and others decreasing in $\regpar_i$, we optimize the
choice of $\regpar_i$, in particular by setting
\begin{align}
\label{eqn:regpar}
\regpar_i^2 & = \frac{\radius_i
  \realrsc}{\spindex\sqrt{\TMP}}\sqrt{\proxbound(\lips_i^2 + \noise_i^2)
  + \devcon_i^2\sigma_i^2}.
\end{align}
Using this setting and combining the upper bound
$\frac{\realrsc}{\lossrsc}\norm{\DelBarAvg}_1^2 \leq \TERMONE +
\TERMTWO$ with the form~\eqref{EqnTermTwo} of $\TERMTWO$ and the upper
bound~\eqref{EqnGoat} on $\TERMONE$, we find that
\begin{align}
  \norm{\DelBarAvg}_1^2 &\leq c\,
  \frac{\lossrsc}{\realrsc}\left[\frac{\spindex\radius_i}{\realrsc\sqrt{\TMP}}
    \left(\sqrt{\proxbound(\lips_i^2 + \noise_i^2)} +
    \devcon_i\noise_i\right) + \frac{\radius_i^2\proxbound}{\TMP} +
    \spindex \totalslacksq\right].
  \label{eqn:epochboundfinal}
\end{align}
Combining the above inequality with the error
bound~\eqref{eqn:epocherrorbound} for $\epochopt{i}$ and triangle
inequality leads to
\begin{align*}
  \norm{\DelAvg}_1^2 & \leq 2 \norm{\DelBarAvg}_1^2 + 2\norm{\opt -
    \epochopt{i}}_1^2 \leq 2\norm{\DelBarAvg}_1^2 + \frac{162
    \spindex^2 \regpar_i^2}{\realrsc^2} + \const\, \spindex \;
  \totalslacksq\\
& \leq 2 \norm{\DelBarAvg}_1^2 + \frac{\lossrsc}{\realrsc} \const
  \left( \frac{\spindex^2 \regpar_i^2}{\realrsc^2} + \,\spindex \;
  \totalslacksq \right),
\end{align*}
where the second inequality follows since $\lossrsc \geq \realrsc$.
Substituting the setting~\eqref{eqn:regpar} of $\regpar_i$ yields an
upper bound identical to our earlier bound~\eqref{eqn:epochboundfinal}
with different constants.

Finally, in order to use $\xavg(\TMP)$ as our next prox-center
$\proxcenter{i+1}$, we would also like to control the error
$\norm{\xavg(\TMP) - \epochopt{i+1}}_1^2$.  Since $\regpar_{i+1} \leq
\regpar_i$ by assumption, we obtain the same form of error
bound~\eqref{eqn:epochboundfinal}. We want to run the epoch till all
these error terms drop to $\radius_{i+1}^2 := \radius_i^2/2$.
Recalling our assumption that $\lossrsc \totalslacksq / \realrsc \leq
\radius_i^2/4$, it suffices to set the epoch length $T_i$ to ensure
that
\begin{equation*}
c \, \frac{\spindex \radius_i \lossrsc}{\realrsc^2 \sqrt{T_i}}
\sqrt{\proxbound (\lips_i^2 + \noise_i^2)} \leq
\frac{\radius_i^2}{12}, \qquad
c \, \frac{\spindex \radius_i \lossrsc}{\realrsc^2 \sqrt{T_i}} \devcon_i
\noise_i \leq \frac{\radius_i^2}{12},
\quad \mbox{and} \quad c \,\frac{\lossrsc \radius_i^2
  \proxbound}{\realrsc T_i} \leq \frac{\radius_i^2}{12}.
\end{equation*}
All the above conditions are met if we choose the epoch length 
\begin{align*}
T_i & = \Plaincon \, \left [\frac{\spindex^2 \lossrsc^2}{\realrsc^4
    \radius_i^2} \left( \proxbound (\lips_i^2 + \noise_i^2) +
  \devcon_i^2 \noise_i^2 \right) + \frac{\proxbound\lossrsc}{\realrsc}
  \right],
\end{align*}
for a suitably large universal constant $\Plaincon$, which completes
the proof of the second part of the proposition.  The stated bound in
function values follows from substituting the choice of $\regpar_i$ in
our earlier bound~\eqref{eqn:epochmainbound} and some straightforward
algebra.



\subsection{Proof of Lemma~\ref{lemma:errsquarebound}}
\label{AppLemErrsquarebound}
It remains to prove Lemma~\ref{lemma:errsquarebound}, a result used
during the proof of Proposition~\ref{prop:mainepoch}.  We do so by
exploiting some classical martingale tail bounds of the
Azuma-Hoeffding type. The particular result given here is due to Lan
et al.~\cite{LanNeSh2011}:
\begin{lemma}
\label{lemma:concentration}
Let $\sample_1,\sample_2,\ldots$ be a sequence of i.i.d.\ random
variables, let $\noise_t > 0$, $t = 1,2,\ldots$ be a sequence of
deterministic numbers, and let $\phi_t = \phi_t(\sample^t)$ be
deterministic (measurable) functions of $\sample^t = (\sample_1,
\sample_2, \ldots, \sample_t)$.  Using $\F_t$ to denote the
$\sigma$-field of $\sample^t$, we have:
\begin{itemize}
\item[(a)] Suppose $\E[\phi_t\,\mid \F_{t-1}] = 0$ with probability
  one and $\E[\exp(\phi_t^2/\noise_t^2)\,\mid \F_{t-1}] \leq \exp(1)$
  with probability one for all $t$. Then
\begin{align}
 \label{eqn:subgauss}
 \P \Big[\sum_{t=1}^T \phi_t > \delta \big(\sum_{t=1}^T
   \noise_t^2\big)^{1/2} \Big] & \leq \exp(-\delta^2/3) \qquad
 \mbox{for all $\delta \geq 0$.}
\end{align}
\item[(b)] Suppose that $\E[\exp(|\phi_t|/\noise_t)\,\mid \F_{t-1}]
  \leq \exp(1)$ w.p. 1 for $t$. Letting $\noise^T =
  (\noise_1,\ldots,\noise_T)$, we have the bound
\begin{align}
\label{eqn:subexp}
\P \left[\sum_{t=1}^T \phi_t > \norm{\noise^T}_1 + \delta
  \norm{\noise^T}_2\right] & \leq \exp(-\delta^2/12) + \exp \left(
-\frac{3\norm{\noise^T}_2}{4\norm{\noise^T}_\infty} \delta \right)
\nonumber \\
& \leq \exp(-\delta^2 /12) + \exp(-3 \delta/4).
\end{align}
\end{itemize}
\end{lemma}

\noindent We now use this lemma to prove parts (a) and (b)
of Lemma~\ref{lemma:errsquarebound}.

\subsubsection{Proof of part (a)}

We start by showing that the conditions of
Lemma~\ref{lemma:concentration} are satisfied.  Indeed, by
Assumption~\ref{ass:subgauss}, we have
\begin{align*} 
\E \exp \left( \frac{ \mystep{t-1}
  \norm{\mygraderr{t}}_\infty^2}{\noise_i^2 \mystep{t-1}} \right) & = \E
\exp \left (\frac{\norm{\mygraderr{t}}_\infty^2}{\noise_i^2} \right)
\leq \exp(1).
\end{align*}
Consequently, in order to satisfy the condition of
Lemma~\ref{lemma:concentration}(b), it suffices to set $\noise_t =
\noise^2_i \mystep{t-1}$.  Recalling our choice $\mystep{t} =
\stepsize/\sqrt{t}$, we find that
\begin{align*}
\norm{\noise^T}_1 & = \noise_i^2 \sum_{t=1}^T \mystep{t} = \noise_i^2
\stepsize \sum_{t=1}^T \frac{1}{\sqrt{t}} \leq 2 \noise_i^2 \stepsize
\sqrt{T},\\
\norm{\noise^T}_2 & = \noise_i^2 \sqrt{\sum_{t=1}^T (\mystep{t})^2} =
\noise_i^2 \stepsize \sqrt{\sum_{t=1}^T \frac{1}{t}} \leq \noise_i^2
\stepsize \sqrt{2\log T} \quad \mbox{and} \\
\norm{\noise^T}_\infty & = \noise_i^2 \stepsize.
  \end{align*}
Plugging the above quantities in the statement of
Lemma~\ref{lemma:concentration}(b) with $\delta = \devcon_i$ yields
\begin{align*}
\P \Big[ \sum_{t=1}^T \mystep{t-1} \norm{\mygraderr{t}}_\infty^2 > 2
  \noise_i^2 \stepsize \sqrt{T} + \devcon_i\noise^2_i \stepsize \sqrt{2
    \log T} \Big] & \leq \exp(-\devcon_i^2/12) + \exp\left(-\frac{3
  \noise_i^2 \stepsize \sqrt{ \sum_{t=1}^T 1/t}}{4 \noise_i^2
  \stepsize} \devcon_i \right) \\
& \leq \exp(-\devcon_i^2/12) + \exp\left(-\frac{3\sqrt{\log
    T}}{4}\devcon_i\right) \\  
& \leq 2 \exp(-\devcon_i^2/12),
  \end{align*}
where the last inequality uses our assumption $\devcon_i \leq 9
\sqrt{\log T}$, thus completing the proof.


\subsubsection{Proof of part (b)}

We now turn to part (b) of Lemma~\ref{lemma:errsquarebound}.  By
assumption, we have $\E[\mygraderr{t} \mid \F_{t-1}] = 0$; moreover,
the random variable $\iter{t}$ is measurable with respect to
$\F_{t-1}$ and $\epochopt{i}$ is deterministic.  Consequently, the
first condition of Lemma~\ref{lemma:concentration}(a) is
satisfied. For the second condition, we observe that by H\"older's
inequality and Assumption~\ref{ass:subgauss}, we have
\begin{align*}
\E \exp \left( \frac{\ip{\mygraderr{t}}{\iter{t} - \epochopt{i}}^2} {4
  \radius_i^2 \noise_i^2} \right) & \leq \E \exp \left(
\frac{\norm{\mygraderr{t}}_\infty^2 \norm{\iter{t} -
    \epochopt{i}}_1^2}{4 \radius_i^2 \noise_i^2}\right) \\
& \stackrel{(a)}{\leq} \E \exp \left(
\frac{4\norm{\mygraderr{t}}_\infty^2 \radius_i^2}{4 \radius_i^2
  \noise_i^2} \right) = \E \exp \left(
\frac{\norm{\mygraderr{t}}_\infty^2}{\noise_i^2} \right) \leq \exp(1).
  \end{align*}
Here inequality (a) uses the facts that $\norm{\iter{t} -
  \proxcenter{i}}_1 \leq \radius_i$ and $\norm{\epochopt{i} -
  \proxcenter{i}}_1 \leq \radius_i$ by the definition of our
updates~\eqref{EqnOverallDual}, so that the conditions of
Lemma~\ref{lemma:concentration}(a) are satisfied with $\noise_t = 2
\noise_i \radius_i$. Plugging this setting in the result of the lemma
and setting $\delta = \devcon/2$ completes the proof.


\section{Proof of Lemma~\ref{lemma:badepochell2}}
\label{AppLemBad}

The proof of this lemma is based on pair of auxiliary results, which
we begin by stating.
\begin{lemma}
\label{lemma:badepochell1}
Suppose at some epoch $k$, we have the bound $\norm{\proxcenter{k} -
  \opt}_1 \leq \radius_k$. Then for all epochs $j \geq k$, we have
\begin{equation*}
\norm{\proxcenter{j} - \opt}_1 \leq 8 \radius_k.
\end{equation*}
\end{lemma}

\begin{lemma}
\label{LemBart}
Under the conditions of Lemma~\ref{lemma:badepochell2}, for any epoch
$i > k$, we have
\begin{align}
\label{EqnBart}
f_i(\proxcenter{i+1}) - f_i(\proxcenter{i}) & \leq \const \,
\frac{\radius_{k}^2 \realrsc^2}{\spindex \lossrsc}\, 2^{-(i - k)/2}.
\end{align}
\end{lemma}

\noindent See Sections~\ref{AppLemAnnoy} and~\ref{AppLemBart} for the
proofs of these two lemmas, respectively.

\subsection{Main argument}

With these auxiliary results in hand, we now turn to the proof of
Lemma~\ref{lemma:badepochell2}.  By the definition of $f_i$, we have
\begin{align*}
\Lossbar(\proxcenter{i+1}) - \Lossbar(\opt) & =
(\Lossbar\proxcenter{i+1} - \Lossbar(\proxcenter{i})) +
(\Lossbar(\proxcenter{i}) - \Lossbar(\opt)) \\
& = (f_i(\proxcenter{i+1}) - f_i(\proxcenter{i})) +
\regpar_k(\norm{\proxcenter{i}}_1 - \norm{\proxcenter{i+1}}_1) +
(\Lossbar(\proxcenter{i}) - \Lossbar(\opt))\\ 
& \stackrel{(i)}{\leq} \const\,\frac{\radius_{k}^2\realrsc^2}{\spindex
  \lossrsc}\,2^{-(i - k)/2} + \regpar_i\radius_i +
(\Lossbar(\proxcenter{i}) - \Lossbar(\opt)),
\end{align*}
where step (i) follows from a combination of Lemma~\ref{LemBart}, and
triangle inequality along with the feasibility of $\proxcenter{i+1}$
at epoch $i$ since $\norm{\proxcenter{i}}_1 -
\norm{\proxcenter{i+1}}_1 \leq \norm{\proxcenter{i} -
  \proxcenter{i+1}}_1 \leq \radius_i$. By applying the Cauchy-Schwarz
inequality to the second term in the bound above, we obtain
\begin{align*}
\Lossbar(\proxcenter{i+1}) - \Lossbar(\opt) & \leq
\const\,\frac{\radius_{\kstar}^2 2^{-(i - k)/2}\realrsc^2}{\spindex
  \lossrsc} + \frac{\realrsc\radius_i^2}{2\spindex} +
\frac{\spindex\regpar_i^2}{2\realrsc} + (\Lossbar(\proxcenter{i}) -
\Lossbar(\opt))\\
& \leq \const \, \frac{\radius_{\kstar}^2 2^{-(k -\kstar)/2}
  \lossrsc}{\spindex} + ( \Lossbar(\proxcenter{k}) - \Lossbar(\opt) ),
  \end{align*}
where the last inequality uses the
setting~\eqref{eqn:paramsettings-simple} of $\regpar_i$ and the
setting~\eqref{eqn:epochlength} of $T_i$. Recursing the argument
further yields
\begin{align*}
 \Lossbar(\proxcenter{j}) - \Lossbar(\opt) & \leq \const \,
 \sum_{i=k+1}^j \frac{\radius_{k}^22^{-(i -k)/2} \realrsc^2}{\spindex
   \lossrsc} + (\Lossbar(\proxcenter{k+1}) - \Lossbar(\opt)) \\
 & \leq \const \frac{\radius_{k}^2 \realrsc^2}{(\sqrt{2}-1) \spindex
   \lossrsc} + (\Lossbar(\proxcenter{k+1}) - \Lossbar(\opt)),
\end{align*}
where the second inequality upper bounds the sum of the geometric
progression. Recalling the given conditions in the lemma, we obtain
the upper bound
\begin{align*}
\Lossbar(\proxcenter{k+1}) - \Lossbar(\opt) & \leq
f_{k}(\proxcenter{k+1}) - f_{k}(\opt) + 2 \regpar_{k} \radius_{k} \;
\leq \; \const \, \frac{\realrsc^2 \radius_{k}^2}{\spindex \lossrsc}.
\end{align*}
Combining the two previous bounds yields $\Lossbar(\proxcenter{j}) -
\Lossbar(\opt) \leq \const \, \frac{\radius_{k}^2 \realrsc^2}{\spindex
  \lossrsc}$.  

Our last step is to apply the RSC condition to this inequality.  Since
$\opt$ minimizes $\Lossbar(\param)$, we have
\begin{align*}
\frac{\lossrsc}{2}\norm{\proxcenter{j} - \opt}_2^2 \leq
\Lossbar(\proxcenter{j}) - \Lossbar(\opt) + \slopfac
\norm{\proxcenter{j} - \opt}_1^2 & \leq \const \,
\left[ \frac{\radius_{k}^2 \realrsc^2}{\spindex \lossrsc} + \slopfac
  \radius_k^2\right], 
\end{align*}
where the second inequality uses
Lemma~\ref{lemma:badepochell1}. Finally, we recall that $\realrsc =
\lossrsc - 16\spindex \slopfac$, which allows us to further simplify
the above upper bound to $\frac{\lossrsc}{2}\norm{\proxcenter{j} -
  \opt}_2^2 \leq \const\, \frac{\radius_k^2 \realrsc}{\spindex}$, and
observing that $\realrsc \leq \lossrsc$ completes the proof.

\subsection{Proof of Lemma~\ref{lemma:badepochell1}}
\label{AppLemAnnoy}

The proof of this lemma is straightforward given the definition of our
updates~\eqref{EqnOverallDual}.  At any epoch $j \geq k$, the prox
center $\proxcenter{j}$ is feasible at epoch $j-1$, so that
\begin{equation*}
\norm{\proxcenter{j} - \proxcenter{j-1}}_1 \leq e\norm{\proxcenter{j}
  - \proxcenter{j-1}}_1 \leq e\radius_{j-1},
\end{equation*}
where we have used the fact that $\|\param\|_1 \leq e
\|\param\|_\pval$, by our choice~\eqref{EqnDefnProx} of $\pval$.
Consequently, by definition of the updates~\eqref{EqnOverallDual}, we
have
\begin{align*}
\|\proxcenter{j} - \opt\|_1 & \leq \| \proxcenter{j} -
\proxcenter{j-1}\|_1 + \| \proxcenter{j-1} - \opt \|_1\\ &\leq e
\radius_{j-1} + \|\proxcenter{j-1} - \proxcenter{j-2} \|_1 +
\|\proxcenter{j-2} - \opt\|_1.
\end{align*}
By repeating this argument, we may unwind the error bound until we
reach epoch $k$, thereby obtaining the bound
\begin{align*}
\|\proxcenter{j} - \opt\|_1 & \leq 
\sum_{i=k+1}^j e \radius_i +
\|\proxcenter{k} - \opt\|_1 \; \stackrel{(i)}{\leq}
\sum_{i=k+1}^j e \radius_i + \radius_{k},
\end{align*}
where inequality (i) follows by the lemma assumption.

Finally, we observe that the last term from epoch $k$ is controlled by
assumption in the lemma. As a result, we can further obtain
\begin{align*}
\|\proxcenter{j} - \opt\|_1 & \leq e \radius_{k+1}
\sum_{j=0}^\infty \sqrt{2}^{-j} + \radius_{k},
\end{align*}
Summing the geometric progression and noting that $e/(\sqrt{2}-1) \leq
7$ completes the proof.

\subsection{Proof of Lemma~\ref{LemBart}}
\label{AppLemBart}

Note that at any epoch $i$, the prox-center $\proxcenter{i}$ is always
feasible by construction. As a result,
equation~\eqref{eqn:epochmainbound} guarantees that
\begin{align}
  f_i(\proxcenter{i+1}) - f_i(\proxcenter{i}) & \leq \const \radius_i
  \left[\frac{\sqrt{\proxbound(\lips_i^2 + \noise_i^2)} + \devcon_i
      \noise_i}{\sqrt{T_i}} + \regpar_i
    \sqrt{\frac{\proxbound}{T_i}}\right] \nonumber \\
\label{EqnChile}
& \stackrel{(i)}{\leq} \const \radius_i \left[\frac{\sqrt{\proxbound
          (\lips_i^2 + \noise_i^2)} + \devcon\noise_i}{\sqrt{T_i}} +
      \frac{\spindex \regpar_i^2}{\radius_i \realrsc} +
      \frac{\radius_i \realrsc \proxbound}{\spindex T_i}\right],
\end{align}
with probability at least $1 - 3\exp(-\devcon_i^2/12)$, where step (i)
uses the elementary inequality \mbox{$2 a b \leq a^2 + b^2$.}
Recalling our setting~\eqref{eqn:paramsettings-lipschitz} of the
regularization parameter $\regpar_i$, we find that
\begin{align*}
  \frac{\spindex\regpar_i^2}{\radius_i\realrsc} & \leq
  \sqrt{\frac{\proxbound (\lips_i^2 + \noise_i^2) + \devcon_i^2
      \noise_i^2} {T_i}} \; \leq \; \frac{\sqrt{\proxbound (\lips_i^2
      + \noise_i^2)} + \devcon_i \noise_i} {\sqrt{T_i}},
\end{align*}
Substituting this upper bound in our earlier
inequality~\eqref{EqnChile} yields
\begin{align*}
  f_i(\proxcenter{i+1}) - f_i(\proxcenter{i}) & \leq \const \radius_i
  \left[\frac{\sqrt{\proxbound(\lips_i^2 + \noise_i^2)} + \devcon
      \noise_i}{\sqrt{T_i}} + \frac{\radius_i \realrsc
      \proxbound}{\spindex T_i}\right].
\end{align*}
Under the conditions of Lemma~\ref{lemma:badepochell2}, the first term
in the above inequality is at most $\realrsc^2 R_k/ (2
\spindex\lossrsc)$ for any $i > k$. Further recalling the assumption
that $T_i = \order(\proxbound)$, we see that for any $i > k$
\begin{align*}
f_i(\proxcenter{i+1}) - f_i(\proxcenter{i}) & \leq \const \, \radius_i
\left( \realrsc^2 R_k/ (2\spindex \lossrsc) + \frac{\realrsc^2
  \radius_i}{\spindex \lossrsc} \right) \nonumber \\
& = \const \, \frac{\realrsc^2}{\spindex \lossrsc} (\radius_1^2
2^{-(k-1)/2} 2^{-(i-1)/2} + \radius_1^2 2^{-(i-1)}) \nonumber \\
& = \const \,\frac{\radius_1^2 2^{-(k-1)} \realrsc^2}{\spindex
  \lossrsc} (2^{-(i-k)/2} + 2^{-(i-k)}) \nonumber \\
& \leq \const \, \frac{\radius_{k}^2 \realrsc^2}{\spindex \lossrsc}\,
2^{-(i - k)/2},
\end{align*}
which completes the proof.


\section{Proof of Lemma~\ref{lemma:rsc}}
\label{proof:lemma-rsc}

Results of this flavor have been established in prior work
(e.g.,~\cite{RudelsonZh2012, LohWa2011}).  We provide a proof here for
completeness, building on the result of Loh and
Wainwright~\cite{LohWa2011}. Following their notation, we define the
$\ell_0$-``ball'' of radius $\kdim$, namely the set $\Ball_0(\kdim)
\defn \big \{ \param \, \mid \, \mbox{$\param_j \neq 0$ for at most
  $\kdim$ indices} \}$, as well as the set
\begin{equation*}
  \SparseSet(\kdim) = \Ball_0(\kdim) \cap \Ball_2(1). 
\end{equation*}
We establish our claim by appealing to Lemma 12 of Loh and Wainwright,
applying their result with the settings
\begin{equation*}
\Gamma =\frac{ X^TX}{\numobs} - \CovMat\quad \mbox{and the sparsity
  parameter} \quad \kdim \defn \const_0\, \frac{\numobs} {\log\pdim}
\min \left\{ \frac{\sigmin^2(\CovMat)} {\xnoise^4}, 1\right\},
\end{equation*}
where $\const_0$ is an appropriate universal constant chosen to ensure
that $\kdim > 1$. Based on this result, we see that it suffices to
establish
\begin{equation*}
  \left|v^T \left( \frac{X^TX}{\numobs} - \CovMat\right) v \right|
  \leq \frac{\sigmin(\CovMat)}{54} \quad \mbox{for all}~~ v \in
  \SparseSet(2\kdim). 
\end{equation*}

Under Assumption~\ref{ass:subgauss-design} on the sub-Gaussianity of
the design matrix $\sampleX$, we can establish the above condition by
appealing to Lemma 15 of Loh and
Wainwright~\cite{LohWa2011}. Specifically, we apply their result with 
$t = \frac{\sigmin(\CovMat)}{54}$ and with $s = \kdim$ as defined
above. Then we can mimic the argument in the proof of Lemma 1 in the
paper~\cite{LohWa2011} to conclude that with probability at least $1 -
2\exp(-\const_1\numobs\min(\sigmin^2(\CovMat)/\xnoise^4,1))$ we have
the bound
\begin{equation*}
  \left| v^T \left( \frac{X^TX}{\numobs} - \CovMat\right) v \right|
  \leq \frac{\sigmin(\CovMat)} {2} \left( \norm{v}_2^2 + \frac{
    \norm{v}_1^2} {\kdim} \right) \quad \mbox{for all}~~ v \in
  \R^{\pdim}. 
\end{equation*}
Substituting our setting of $\kdim$ and rearranging terms completes
the proof.


\section{Proof of Lemma~\ref{lemma:goodepoch}}

The base case for the $\ell_1$-error bound at $k = 1$ is true by our
assumption that $\norm{\opt}_1 \leq \radius_1$. As a result, the
convergence analysis of Proposition~\ref{prop:mainepoch} applies at
the first epoch. Assuming that $\kstar > 1$, our setting of $\kstar$
ensures that
\begin{equation*}
\frac{\radius_1 \spindex}{\lossrsc} \frac{\sqrt{\proxbound(\lips_1^2 +
    \noise_1^2)} + \devcon\noise_i} {\sqrt{T_0}} \; \leq \l \const
\frac{\radius_1^2} {2} \; = \; \const \radius_2^2.
  \end{equation*}
Since $T = \order(\proxbound)$ by assumption, the $\radius_1^2
\proxbound/T_0$ term can also be further upper bounded by $\const
\radius_2^2$. Hence, as long as $\radius_2^2 \geq 2 \totalslacksq$, we
obtain the stated $\ell_1$ error bound at the second epoch by applying
equation~\eqref{eqn:ell1bound} from
Proposition~\ref{prop:mainepoch}. A similar calculation using
equation~\ref{eqn:fvaluebound} yields
\begin{equation*}
\frac{\lossrsc}{2} \norm{\xavg(T_0) - \epochopt{2}}_2^2 \leq
f_2(\xavg(T_0)) - f_2(\epochopt{2}) \leq \const \,
\frac{\radius_2^2\lossrsc}{\spindex}.
\end{equation*}
Finally, we can obtain a similar bound up to constant factors on
$\norm{\xavg(T_0) - \opt}_2^2$ as well by combining with the
$\ell_2$-error bound of Lemma~\ref{lemma:opterror} as before. Thus, we
obtain our inductive claim for $k = 2$. Assuming the inductive
hypothesis for arbitrary $i < \kstar + 1$, the reasoning for obtaining
the inductive claim at $i+1$ is exactly identical to the above
arguments, completing the proof of the lemma.


\bibliographystyle{plain} \bibliography{bib}

\begin{thebibliography}{10}

\bibitem{AgarwalBaRaWa12}
A.~Agarwal, P.~Bartlett, P.~Ravikumar, and M.~J. Wainwright.
\newblock Information-theoretic lower bounds on the oracle complexity of
  stochastic convex optimization.
\newblock {\em IEEE Trans. Information Theory}, 58(5):3235 --3249, 2012.

\bibitem{AgarwalNeWa10}
A.~Agarwal, S.~Negahban, and M.~J. Wainwright.
\newblock Fast global convergence rates of gradient methods for
  high-dimensional statistical recovery.
\newblock {\em Annals of Statistics}, 2012.
\newblock Presented in part at NIPS 2010 conference; Full length version
  http://arxiv.org/pdf/1104.4824v2.

\bibitem{BickelRiTs09}
P.~J. Bickel, Y.~Ritov, and A.~B. Tsybakov.
\newblock {Simultaneous analysis of Lasso and Dantzig selector.}
\newblock {\em Ann. Stat.}, 37(4):1705--1732, 2009.

\bibitem{BottouBo07}
L.~Bottou and O.~Bousquet.
\newblock The tradeoffs of large scale learning.
\newblock In {\em NIPS}, 2007.

\bibitem{BuhlmannGe2011}
P.~B{\"u}hlmann and S.~Van De~Geer.
\newblock {\em Statistics for High-Dimensional Data: Methods, Theory and
  Applications}.
\newblock Springer Series in Statistics. Springer, 2011.

\bibitem{BulKoz}
V.~V. Buldygin and Y.~V. Kozachenko.
\newblock {\em Metric characterization of random variables and random
  processes}.
\newblock American Mathematical Society, Providence, RI, 2000.

\bibitem{Donoho00}
D.~L. Donoho.
\newblock High-dimensional data analysis: The curses and blessings of
  dimensionality, 2000.

\bibitem{DuchiAgWa11}
J.~C. Duchi, A.~Agarwal, and M.~J. Wainwright.
\newblock Dual averaging for distributed optimization: {C}onvergence analysis
  and network scaling.
\newblock {\em IEEE Transactions on Automatic Control}, 57(3):592 --606, 2012.

\bibitem{DuchiShSiTe10}
J.~C. Duchi, S.~Shalev-Shwartz, Y.~Singer, and A.~Tewari.
\newblock Composite objective mirror descent.
\newblock In {\em COLT}, pages 14--26. Omnipress, 2010.

\bibitem{DuchiSi09c}
J.~C. Duchi and Y.~Singer.
\newblock Efficient online and batch learning using forward-backward splitting.
\newblock {\em Journal of Machine Learning Research}, 10:2873--2898, 2009.

\bibitem{HazanKaKaAg06}
E.~Hazan, A.~Kalai, S.~Kale, and A.~Agarwal.
\newblock Logarithmic regret algorithms for online convex optimization.
\newblock In {\em COLT}, 2006.

\bibitem{HazanKa11a}
E.~Hazan and S.~Kale.
\newblock Beyond the regret minimization barrier: an optimal algorithm for
  stochastic strongly-convex optimization.
\newblock {\em Journal of Machine Learning Research - Proceedings Track},
  19:421--436, 2011.

\bibitem{HiriartUrrutyLe96}
J.~Hiriart-Urruty and C.~Lemar\'echal.
\newblock {\em Convex {A}nalysis and {M}inimization {A}lgorithms {I}}.
\newblock Springer, 1996.

\bibitem{JuditskyNes10}
A.~Juditsky and Y.~Nesterov.
\newblock Primal-dual subgradient methods for minimizing uniformly convex
  functions.
\newblock 2010.

\bibitem{LanGh2010}
G.~Lan and S.~Ghadimi.
\newblock Optimal stochastic approximation algorithms for strongly convex
  stochastic composite optimization, {P}art {II}: {S}hrinking procedures and
  optimal algorithms.
\newblock 2010.

\bibitem{LanNeSh2011}
G.~Lan, A.~Nemirovski, and A.~Shapiro.
\newblock Validation analysis of mirror descent stochastic approximation
  method.
\newblock {\em Mathematical Programming}, pages 1--34, 2011.

\bibitem{LohWa2011}
P.~Loh and M.~J. Wainwright.
\newblock High-dimensional regression with noisy and missing data: Provable
  guarantees with non-convexity.
\newblock {\em Annals of Statistics}, 2011.
\newblock Presented in part at NIPS Conference, December 2011:
  arxiv.org/pdf/1109.3714v2.pdf.

\bibitem{NegRavWaiYu09}
S.~Negahban, P.~Ravikumar, M.~J. Wainwright, and B.~Yu.
\newblock A unified framework for high-dimensional analysis of {M}-estimators
  with decomposable regularizers.
\newblock {\em Statistical Science}, 2012.
\newblock To appear; Original version arxiv:1010.2731v1.

\bibitem{NemirovskiJuLaSh09}
A.~Nemirovski, A.~Juditsky, G.~Lan, and A.~Shapiro.
\newblock Robust stochastic approximation approach to stochastic programming.
\newblock {\em SIAM Journal on Optimization}, 19(4):1574--1609, 2009.

\bibitem{NemirovskiYu83}
A.~Nemirovski and D.~Yudin.
\newblock {\em Problem Complexity and Method Efficiency in Optimization}.
\newblock Wiley, New York, 1983.

\bibitem{Nesterov07}
Y.~Nesterov.
\newblock Gradient methods for minimizing composite objective function.
\newblock Technical Report~76, Center for Operations Research and Econometrics
  (CORE), Catholic University of Louvain (UCL), 2007.

\bibitem{Nesterov09}
Y.~Nesterov.
\newblock Primal-dual subgradient methods for convex problems.
\newblock {\em Mathematical Programming A}, 120(1):261--283, 2009.

\bibitem{RasWaiYu10}
G.~Raskutti, M.~J. Wainwright, and B.~Yu.
\newblock Restricted eigenvalue conditions for correlated {G}aussian designs.
\newblock {\em Journal of Machine Learning Research}, 11:2241--2259, August
  2010.

\bibitem{RasWaiYu11}
G.~Raskutti, M.~J. Wainwright, and B.~Yu.
\newblock Minimax rates of estimation for high-dimensional linear regression
  over $\ell_q$-balls.
\newblock {\em IEEE Trans. Information Theory}, 57(10):6976---6994, October
  2011.

\bibitem{RudelsonZh2012}
M.~Rudelson and M.~Zhou.
\newblock Reconstruction from anisotropic random measurements.
\newblock In {\em COLT}, 2012.
\newblock full-length version: http://arxiv.org/pdf/1106.1151v1.

\bibitem{ShalevSiSr07}
S.~Shalev-Shwartz, Y.~Singer, and N.~Srebro.
\newblock Pegasos: Primal estimated sub-gradient solver for {SVM}.
\newblock In {\em ICML}, 2007.

\bibitem{Shalev-ShwartzTe11}
S.~Shalev-Shwartz and A.~Tewari.
\newblock Stochastic methods for $\ell_1$ regularized loss minimization.
\newblock {\em Journal of Machine Learning Research}, 12:1865--1892, June 2011.

\bibitem{SrebroSrTe10}
N.~Srebro, K.~Sridharan, and A.~Tewari.
\newblock Smoothness, low noise, and fast rates.
\newblock In {\em NIPS}, pages 2199--2207, 2010.

\bibitem{Geer08}
S.~A. van~de Geer.
\newblock High-dimensional generalized linear models and the {L}asso.
\newblock {\em The Annals of Statistics}, 36:614--645, 2008.

\bibitem{Xiao10}
L.~Xiao.
\newblock Dual averaging methods for regularized stochastic learning and online
  optimization.
\newblock {\em Journal of Machine Learning Research}, 11:2543--2596, 2010.

\bibitem{XiaoZh12}
L.~Xiao and T.~Zhang.
\newblock A proximal-gradient homotopy method for the sparse least-squares
  problem.
\newblock {\em ICML}, 2012.
\newblock URL http://arxiv.org/abs/1203.3002.

\end{thebibliography}

\typeout{-----------------------------------------------------}
\typeout{---> MAJOR HACK:  check labels of HACKRSC}
\typeout{-----------------------------------------------------}
\end{document}